\documentclass{article}

\PassOptionsToPackage{compress, round, sort}{natbib}
\usepackage[preprint]{neurips_2022}

\usepackage[utf8]{inputenc} % allow utf-8 input
\usepackage[T1]{fontenc}    % use 8-bit T1 fonts
\usepackage{pgfplots}
\usepackage{amsthm}

\usepackage{newpxtext} % T1, lining figures in math, osf in text
\usepackage{textcomp} % required for special glyphs
\usepackage[varg,bigdelims]{newpxmath}
\usepackage[scr=rsfso]{mathalfa}% \mathscr is fancier than \mathcal
\usepackage{bm} 
\usepackage[colorlinks=true,
linkcolor=purple,
citecolor=purple]{hyperref} 
\usepackage{url}            % simple URL typesetting
\usepackage{booktabs}       % professional-quality tables
\usepackage{amsfonts}       % blackboard math symbols
\usepackage{nicefrac}       % compact symbols for 1/2, etc.
\usepackage{microtype}      % microtypography
\usepackage{xcolor}         % colors

\usepackage{booktabs}       % professional-quality tables
\usepackage{amsfonts}       % blackboard math symbols
\usepackage{nicefrac}       % compact symbols for 1/2, etc.
\usepackage{microtype}      % microtypography
\usepackage{xcolor}         % colors

\usepackage[linesnumbered, ruled, vlined]{algorithm2e}

\usepackage{todonotes}

\usepackage{amsmath}
\usepackage{amssymb}
\usepackage{mathtools, mleftright}
\mleftright
\usepackage{subcaption}

\usetikzlibrary{shapes.geometric}

\newenvironment{thmbis}[1]
  {\renewcommand{\thetheorem}{\ref{#1}}%
   \addtocounter{theorem}{-1}%
   \begin{theorem}}
  {\end{theorem}}

\newtheorem{definition}{Definition}
\newtheorem{example}{Example}
\newtheorem{theorem}{Theorem}
\newtheorem{proposition}{Proposition}
\newtheorem{observation}{Observation}
\newtheorem{corollary}{Corollary}
\newtheorem{lemma}{Lemma}
\newtheorem{claim}{Claim}
\newtheorem{remark}{Remark}

\usepackage{amsmath,amsfonts,bm}

\def\eqref#1{equation~\ref{#1}}
\def\1{\bm{1}}

\def\vx{{\bm{x}}}
\def\vy{{\bm{y}}}
\def\vz{{\bm{z}}}

\DeclareMathAlphabet{\mathsfit}{\encodingdefault}{\sfdefault}{m}{sl}
\SetMathAlphabet{\mathsfit}{bold}{\encodingdefault}{\sfdefault}{bx}{n}

\newcommand\false{\mathbf{false}}
\newcommand\true{\mathbf{true}}

\newcommand{\ptime}{\text{\rm PTIME}}

\newcommand{\np}{\text{\rm NP}}

\newcommand{\M}{\mathcal{M}}

\newcommand{\es}{\mathbf{e}}

\newcommand{\cnf}{\mathsf{CNF}}

\tikzset{
    rt/.style={
		rectangle,
		fill = white,
		draw=black, 
		text centered,
		inner sep=0.5ex
		},
    rtt/.style={ %tighter version
    	rt,
    	inner sep=0.1ex
    	},
    ert/.style={ %edge box
     	rt,
     	dashed
     	}, 
    ertt/.style={ %edge box, tighter
        rtt,
        dashed
        }, 
    rect/.style={ % rounded prop graph boxes
        rectangle,
        fill = white,
        rounded corners,
        draw=black, 
        text centered,
        inner sep=0.8ex
        },
    rectw/.style={
        rect,
        draw=white
        },
    erect/.style={ %edge rect, unfortunate name
    	rect,
    	dashed
    	},
    erectw/.style={ %edge rectw
     	rectw,
     	dashed
     	},
    arrout/.style={
           ->,
           -latex,
           },
    arrin/.style={
           <-,
           latex-,
           },
    arrb/.style={
           <->,
           >=latex,
           }
}

\newcommand{\Comp}{\textsc{Comp}}
\newcommand{\feat}{\mathcal{F}}
\newcommand{\util}{\text{\sf utility}}
\newcommand{\prob}{\text{\sf Prob}}
\newcommand{\minimum}{\textsf{Compute-Minimum-SR}}
\newcommand{\minimal}{\textsf{Compute-Minimal-SR}}
\newcommand{\checksub}{\textsf{Check-Sub-SR}}
\newcommand{\checkmin}{\textsf{Check-Minimum-SR}}
\newcommand{\sbchecksub}{\textsf{SB-Check-SUB-SR}}
\newcommand{\counting}{\textsf{Count-Positive-Completions}}

\title{ On Computing Probabilistic Explanations \\
for 
Decision Trees}

\author{
	Marcelo Arenas$^{1,2,4}$, 
	Pablo Barcel\'o$^{2,4,5}$,
Miguel Romero$^{3,5}$, 
Bernardo Subercaseaux$^6$\\ \\
Authors are listed in alphabetical order\\
Correspondence at \texttt{bsuberca@cs.cmu.edu}.\\
{$^1$ Department of Computer Science, PUC Chile}\\
{$^2$ Institute for Mathematical and Computational Engineering, PUC Chile}\\
{$^3$ Faculty of Engineering and Science, UAI Chile}\\
{$^4$ Millenium Institute for Foundational Research on Data, Chile}\\
{$^5$ CENIA Chile}\\
{$^6$ Carnegie Mellon University, Pittsburgh, USA}
}

\begin{document}
\maketitle

\begin{abstract}
  Formal XAI (explainable AI) is a growing area that focuses on computing explanations with mathematical guarantees for the decisions made by ML models. Inside formal XAI, one of the most studied cases is that of explaining the
  %decisions made
  choices taken by decision trees, as they are traditionally deemed as one of the most 
interpretable classes of models. Recent work has focused on studying the computation of \emph{sufficient reasons}, a kind of explanation in which given a decision tree $T$ and an instance $\vx$, one explains the decision $T(\vx)$ by providing a subset $\vy$ of the features of $\vx$ such that for any other instance $\vz$ compatible with $\vy$, it holds that  $T(\vz) = T(\vx)$, intuitively meaning that the features in $\vy$ are already enough to fully justify the classification of $\vx$ by $T$. 
It has been argued, however, that sufficient reasons constitute a restrictive notion of explanation. For such a reason, 
 the community has started to study their probabilistic counterpart, in which one requires that the probability of $T(\vz) = T(\vx)$ must be at least some value $\delta \in (0, 1]$, where $\vz$ is a random instance that is compatible with $\vy$. Our paper settles the computational complexity of $\delta$-sufficient-reasons over decision trees, showing that both (1) finding $\delta$-sufficient-reasons  that are minimal in size, and (2) finding $\delta$-sufficient-reasons that are minimal inclusion-wise, do not admit polynomial-time algorithms (unless $\ptime = \np$).
   %NP-hard problems.
   This is in stark contrast with the deterministic case ($\delta = 1$) where inclusion-wise minimal sufficient-reasons are easy to compute. By doing this, we answer two open problems originally raised by Izza et al., and extend the hardness of explanations for Boolean circuits presented by W{\"a}ldchen et al. to the more restricted case of decision trees. On the positive side, we identify structural restrictions of decision trees that make the problem tractable, and show how SAT solvers might be able to tackle these problems in practical settings.
\end{abstract}

\newcommand{\tstar}{T^\star_\delta}

\section{Introduction}

\paragraph{Context.} 
The trust that %application of general use
AI models 
generate in people has been repetitively linked to our ability of \emph{explaining} the decision of said models~\citep{BarredoArrieta2020}, thus suggesting the area of \emph{explainable AI} (XAI) as fundamental for the deployment of trustworthy models. 
A sub-area of explainability that has received considerable attention over the last years,    
 showing quick progress in theoretical and practical terms, is that of {\em local} 
explanations, i.e., explanations for the outcome of a particular input to an ML model
 after the model has been trained. Several queries and scores 
have been proposed to specify explanations of this kind. These include, e.g., 
queries based on {\em prime implicants} \citep{shih2018symbolic} or {\em anchors} \citep{DBLP:conf/aaai/Ribeiro0G18}, which are parts of an instance that are sufficient to explain its classification, as well as scores that intend to quantify 
the impact of a single feature in the output of such a classification \citep{shaps,DBLP:conf/aaai/YanP21}. 
 
A remarkable achievement of this area of research has been the development of {\em formal} notions of explainability. 
%based on minimality conditions. 
The benefits brought about by this principled approach have been highlighted, in a very thorough and convincing way, in 
a recent survey by \citet{formal-xai}. A prime example of this kind of approach is given by {\em sufficient reasons}, which are also known 
as prime implicant explanations \citep{shih2018symbolic} or abductive explanations~\citep{Ignatiev2021}.

 Given an ML model $\M$ of dimension $n$ 
and a Boolean input instance $\vx \in \{0,1\}^n$, a sufficient reason for $\vx$ under $\M$ is a subset $\vy$ of the features of $\vx$, 
such that any instance $\vz$ compatible with $\vy$ receives the same classification result as $\vx$ on $\M$. 
In more intuitive words, $\vy$ is a sufficient reason for $\vx$ under $\M$ if the features in $\vy$ suffices to explain the output of $\M$ on $\vx$.  In the formal explainability approach, one then aims to find sufficient reasons $\vy$ that satisfy one of the following optimality criteria: (a) they are {\em minimum}, i.e., there are no sufficient reasons with 
fewer  features than $\vy$, or (b) they are {\em minimal}, i.e., there are no sufficient reasons that are strictly contained in $\vy$. 

\paragraph{Problem.} In the last few years, the XAI community has studied for which Boolean ML models 
the problem of computing (minimum or minimal) sufficient reasons is computationally tractable and for which it is computationally hard (see, e.g., 
\citep{DBLP:conf/nips/0001GCIN20,NEURIPS2020_b1adda14,DBLP:conf/icml/0001GCIN21}). It has been argued, however, that for practical applications
sufficient reasons might be too \emph{rigid}, as they are specified under worst-case conditions. That is, $\vy$ is a sufficient reason for $\vx$ under $\M$ if {\em every} ``completion''  
of $\vy$ is classified by $\M$ in the same way as $\vx$. As several authors have noted already, 
there is a natural way in which this notion can be relaxed in order to become more suitable for real-world explainability tasks: Instead of asking for each 
completion of $\vy$ to yield the same result as $\vx$ on $\M$, we could allow for a small fragment of the completions of $\vy$ to be classified differently than 
$\vx$ \citep{DBLP:journals/jair/WaldchenMHK21,Izza2021EfficientEW,DBLP:conf/ijcai/WangKB21}. More precisely, we would like to ensure that a random completion of $\vy$ is classified as $\vx$ with probability at least $\delta \in (0,1]$, a threshold that the recipient of the explanation controls. In such case, we call $\vy$ a {\em $\delta$-sufficient reason for $\es$ under $\M$}. 

The study of the cost of computing minimum $\delta$-sufficient reasons for expressive Boolean ML models based on propositional formulas 
was started by \citet{DBLP:journals/jair/WaldchenMHK21}. They show, in particular, that the decision 
problem of checking if $\vx$ admits a $\delta$-sufficient reason of a certain size $k$ under a model $\M$, where $\M$ is 
specified as a CNF formula, is ${\rm NP}^{\rm PP}$-complete. This result shows that the problem is very difficult for complex models, at least in theoretical terms. 
Nonetheless, it leaves the door open for obtaining tractability results over simpler Boolean models, starting from those which are often deemed to be ``easy to interpret'', e.g., 
{\em decision trees} \citep{lipton2018mythos,DBLP:journals/corr/abs-2010-11034,DBLP:conf/dsaa/GilpinBYBSK18}. In particular, the study of the cost of computing 
both minimum and minimal $\delta$-sufficient reasons for decision trees 
was initiated by \citet{Izza2021EfficientEW, https://doi.org/10.48550/arxiv.2205.09569}, but nothing beyond the fact that the problem lies in \np\ has been obtained.  Work by~\citet{blanc2021provably} has shown that it is possible to obtain efficient algorithms that succeed with a certain probability, and that instead of finding a smallest (either cardinality or inclusion-wise) $\delta$-sufficient reason, find $\delta$-sufficient reasons that are small compared to the \emph{average} size of $\delta$-sufficient reasons for the considered model.
  
\paragraph{Our results.} In this paper we provide an in-depth study of the complexity of the problem of minimum and minimal $\delta$-sufficient reasons for decision trees.

\begin{enumerate}

\item We start by pinpointing the exact computational complexity of these problems by showing that, under the assumption that $\ptime \neq \np$, none of them can be solved in polynomial time. We start with minimum $\delta$-sufficient reasons and show that the problem is hard even if $\delta$ is an arbitrary fixed constant in $(0,1]$. Our proof takes as basis the fact that, assuming $\ptime \neq \np$, the problem of computing minimum sufficient reasons for decision trees is not tractable \citep{NEURIPS2020_b1adda14}. The reduction, however, is non-trivial and requires several
  %clever
  involved constructions and a careful analysis. The proof for 
minimal $\delta$-sufficient reasons is even more difficult, and the result more surprising, as in this case we cannot start from a similar problem over decision trees: 
computing minimal sufficient reasons over decision trees (or, equivalently, minimal $\delta$-sufficient reasons for $\delta = 1$) admits a simple polynomial time algorithm. Our result then implies that such a good behavior is lost when the input parameter $\delta$ is allowed to be smaller than $1$.
%$\delta < 1$.  

\item To deal with the high computational complexity of the problems, we look for structural restrictions of it 
that, at the same time, represent meaningful practical instances and ensure that these problems can be solved in polynomial time. 
The first restriction, called {\em bounded split number}, assumes there is a constant $c \geq 1$ such that, 
for each node $u$ of a decision tree $T$ of dimension $n$,  
the number of features that are mentioned in both $T_u$, the subtree of $T$ rooted at $u$, and $T - T_u$, the subtree of $T$ obtained by removing $T_u$, is at most $c$. 
We show that the problems of computing minimum and minimal $\delta$-sufficient reasons over decision trees with bounded split number can be solved in polynomial time. 
The second restriction is {\em monotonicity}. Intuitively, a Boolean 
ML model $\M$ is {\em monotone}, if the class of instances that are classified positively by $\M$ is closed under the operation of replacing $0$s with $1$s.  
For example, if $\M$ is of dimension 3 and it classifies the input $(1,0,0)$ as positive, then so it does for all the instances in $\{(1,0,1),(1,1,0),(1,1,1)\}$. 
We show that computing minimal $\delta$-sufficient reasons for monotone decision trees is a tractable problem. (This good behavior extends to any class of monotone
ML models for which counting the number of positive instances is tractable; e.g., monotone {\em free binary decision diagrams}
%(FBDDs) 
\citep{DBLP:journals/dam/Wegener04}).  
%None of these positive results extends, however, to the problem of computing minimum $\delta$-sufficient reasons. 

\item In spite of the intractability results we obtain in the paper, we show experimentally 
that our problems can be solved over practical instances by using SAT solvers. This requires finding 
efficient encodings of such problems as conjunctive normal form ($\cnf$) formulas, which we then check for satisfiability. 
This is particularly non-trivial for probabilistic sufficient reasons, as it requires dealing with the arithmetical nature of the probabilities involved through a Boolean encoding.  

\end{enumerate}  

\paragraph{Organization of the paper.}  

We introduce the main terminology used in the paper in Section
\ref{sec:Background}, and we define the problems of computing minimum
and minimal $\delta$-sufficient reasons in Section
\ref{sec:ProbabilisticSufficientReasons}.  The intractability of the
these problems is proved in Section \ref{sec:Complexity}, while some
tractable restrictions of them are provided in Section
\ref{sec:TractableCases}. Our Boolean encodings can be found in Section \ref{section:sat}. 
Finally, we discuss some future work in
Section \ref{sec_future}.

\section{Background}
\label{sec:Background}
An {\em instance} of dimension $n$, with $n \geq 1$, is a tuple $\vx
\in \{0,1\}^n$. We use notation $\vx[i]$ to refer to the $i$-th
component of this tuple, or equivalently, its $i$-th feature.
Moreover, we consider an abstract notion of a model of dimension $n$, and we define it as a Boolean function $\M : \{0,1\}^n \to \{0, 1\}$. That is, $\M$ assigns a Boolean value to each instance of dimension $n$, so that we focus on binary classifiers with Boolean input features. Restricting inputs and outputs to be Boolean makes our setting cleaner while still covering several relevant practical scenarios.
%We use notation $\dm(\M)$ for the dimension of a model $\M$. 

A {\em partial instance} of dimension $n$ is a tuple $\vy \in \{0,1,\bot\}^n$. Intuitively, if $\vy[i] = \bot$, then the 
value of the $i$-th feature is undefined. Notice that an instance is a particular case of a partial instance where all features are assigned value either $0$ or $1$. Given two partial instances $\vx$, $\vy$ of dimension $n$, we say that $\vy$ is {\em subsumed} by $\vx$, denoted by $\vy \subseteq \vx$, if for every $i \in \{1, \ldots, n\}$ such that $\vy[i] \neq \bot$, it holds that $\vy[i] = \vx[i]$. That is, $\vy$ is subsumed by $\vx$ if it is possible to obtain $\vx$ from $\vy$ by replacing some undefined values. Moreover, we say that $\vy$ is {\em properly subsumed} by $\vx$, denoted by $\vy \subsetneq \vx$, if $\vy \subseteq \vx$ and $\vy \neq \vx$. Notice that a partial instance $\vy$ can be thought of as a compact representation of the set of instances $\vz$ such that $\vy$ is subsumed by $\vz$, where such instances $\vz$ are called the {\em completions} of $\vy$ and are grouped in the set $\Comp(\vy)$.

A \emph{binary decision diagram} (BDD) of dimension $n$
%over instances of dimension $n$
is a rooted directed acyclic
graph~$\mathcal{M}$ with labels on edges and nodes such that: (i) each
leaf (a node with no outgoing edges) is labeled with~$\true$ or $\false$; (ii) each internal node (a
node that is not a leaf) is labeled with a feature $i \in \{1,\dots,n\}$; and
(iii) each internal node has two outgoing edges, one labeled~$1$ and
the other one labeled~$0$. Every instance $\vx \in \{0,1\}^n$ defines a
unique path $\pi_\vx = u_1 \cdots u_k$ from the root $u_1$ to a leaf
$u_k$ of $\M$ such that: if the label of
$u_i$ is $j \in \{1,\dots,n\}$, 
where $i \in \{1, \ldots, k-1\}$, then the edge from $u_i$ to $u_{i+1}$ is labeled with $\vx[j]$. Moreover, the instance $\vx$ is positive, denoted by
$\mathcal{M}(\vx) = 1$, if the label of $u_k$
is~$\true$; otherwise the instance $\vx$ is negative, which is denoted
by $\mathcal{M}(\vx) = 0$. 
A BDD~$\mathcal{M}$ is \emph{free} if for every path from the
root to a leaf, no two nodes on that path have the same label.
%Besides, $\M$ is {\em ordered} if there exists a linear order $<$ on the set 
%$\{1,\dots,n\}$ of features such that,
%if a node $u$ appears before a node $v$ in some path in $\M$ from the root to a leaf, then $u$ is labeled with $i$ and $v$ is labeled with $j$ for features $i,j$ such that 
%$i < j$.
A \emph{decision tree} is simply a free BDD whose underlying directed acyclic
graph is a rooted tree.

\section{Probabilistic Sufficient Reasons}
\label{sec:ProbabilisticSufficientReasons}
%\subsection{Definitions} 
Sufficient reasons are partial instances obtained by removing from an instance $\vx$ components that do not affect the final classification.
Formally, 
fix a dimension $n$.  Given a decision tree $T$, an instance
$\vx$, and a partial instance $\vy$ with $\vy \subseteq \vx$,
we call $\vy$ a {\em sufficient reason} for $\vx$ under $T$ if
%\[
$T(\vx) = T(\vz)$ for every $\vz \in \Comp(\vy)$.
%, \quad  \forall z \in \textsc{Comp}(y)
%\]
In other words, the features of $\vy$ that take value either $0$ or $1$ explain the decision taken by $T$ on $\vx$, as $T(\vx)$ would not change if the remaining features (i.e., those that are undefined in $\vy$) were to change in $\vx$, thus implying that the classification $T(\vx)$ is a consequence of the features defined in $\vy$.  We say that a sufficient reason $\vy$ for $\vx$ under $T$ is {\em minimal}, if it is minimal under the order induced by $\subseteq$, that is, if there is no sufficient reason $\vy'$ for $\vx$ under $T$ such that $\vy' \subsetneq \vy$. Also, we define a {\em minimum} sufficient reason for $\vx$ under $T$ as a sufficient reason $\vy$ for $\vx$ under $T$ that maximizes the value $|\vy|_\bot := |\{ i \in \{1, \ldots, n\} \mid \vy[i] = \bot\}|$.

It turns out that minimal sufficient reasons can be computed efficiently for decision trees with a very simple algorithm, assuming a sub-routine to check whether a given partial instance is a sufficient reason (not necessarily minimal) of another given instance. As shown in~Algorithm~\ref{alg:minimal}, the idea of the algorithm is as follows: start with a candidate answer $\vy$ which is initially equal to $\vx$, the instance to explain, and maintain the invariant that $\vy$ is a sufficient reason for $\vx$, while trying to remove defined components from $\vy$ until no longer possible. It is not hard to see that one can check whether a partial instance $\vy$ is a sufficient reason for an instance $\vx$ in linear time over decision trees~\citep{DBLP:conf/kr/AudemardBBKLM21}.
This algorithm is well known (see e.g.,~\citep{DBLP:conf/kr/HuangII021}), and relies on the following simple observation, tracing back to~\cite{Goldsmith2005}.

\begin{observation}
	For any class of models $\mathfrak{C}$, If a partial instance $\vy$ of dimension $n$ is a sufficient reason for an instance $\vx$ under a model $\M \in \mathfrak{C}$, but not a minimal sufficient reason, then there is a partial instance $\hat{\vy}$ which is equal to $\vy$ except that $\hat{\vy}[i] = \bot, \vy[i] \neq \bot$ for some $i \in \{1, \ldots, n\}$ which is also a sufficient reason for $\vx$ under $\M$.
	\label{obs:remove-one}
\end{observation}

\SetKwInput{KwInput}{Input}                % Set the Input
\SetKwInput{KwOutput}{Output}              % set the Output

\SetKw{Break}{break}

\begin{algorithm}
	\KwInput{Decision tree $T$ and instance $\vx$, both of dimension $n$}
	\KwOutput{A minimal sufficient reason $\vy$ for $\vx$ under $T$.}
	\vspace{0.5em}
	$\vy \gets \vx$\\
	\While{true} {
	$\text{reduced} \gets \false$\\
		\For{$i \in \{1, \ldots, n\}$}{
			$\hat{\vy} \gets \vy$\\
			$\hat{\vy}[i] \gets \bot$\\
			\If{CheckSufficientReason$(T, \hat{\vy}, \vx)$}{
				$\vy \gets \hat{\vy}$\\
				$\text{reduced} \gets \true$\\
				\Break
			}
		}
		\If{($\neg \text{reduced}$) or $|\vy|_\bot = n$}{
			\Return $\vy$
		}
	}
	\caption{Minimal Sufficient Reason}
	\label{alg:minimal}
\end{algorithm}

The following theorem shows 
a stark contrast between the complexity of computing minimal and minimum sufficient reasons over decision trees. 

\begin{theorem}[\cite{NEURIPS2020_b1adda14}]
Assuming $\ptime \neq
\np$, there is no polynomial-time algorithm that, given a decision
tree $T$ and an instance $\vx$ of the same dimension, computes a
minimum sufficient reason for $\vx$ under $T$.
\label{thm:neurips2020}
\end{theorem}

%of $T$ on $x$. 
%
%Thus,
Arguably, the notion of sufficient reason is a natural notion of
explanation for the result of a classifier. However, such a concept
imposes a severe restriction by asking all completions of a partial
instance to be classified in the same way. To overcome this
limitation, a probabilistic generalization of sufficient reasons was
proposed by~\cite{DBLP:journals/jair/WaldchenMHK21} and
\cite{Izza2021EfficientEW}. More precisely,
%which relaxes the condition
%that all completions of a partial instance $\es_1$ yield the same
%result as $\es$ on $T$. Instead,
this notion allows to settle a confidence $\delta \in (0,1]$ on the fraction of completions of a partial instance
%$\es_1$
that yield the same classification.
%result as $\es$ on $T$.
%

\begin{definition}[Probabilistic sufficient reasons]
Given a value $\delta \in (0, 1]$, a {\em $\delta$-sufficient reason} ($\delta$-SR for short) for an instance $\vx$ under a decision tree $T$ is a partial instance $\vy$ such that $\vy \subseteq \vx$ and
\[
\Pr{}_{\!{\vz}}[T(\vz) = T(\vx) \mid \vz \in \Comp(\vy)] \ := \ \frac{\left|\left\{ \vz \in \Comp(\vy) \mid T(\vz) = T(\vx)\right\}\right|}{2^{|\vy|_\bot}} \ \geq \ \delta. 
\]
\end{definition} 

Minimal and minimum $\delta$-sufficient reasons are defined analogously as in the case of minimal and minimum sufficient reasons.

\begin{example} \label{ex:psr} 
{\em Consider the decision tree $T$ over features $\{1,2,3\}$ shown in Figure \ref{fig:dt}  and the input instance $\vx = (1,1,1)$. Notice that $T(\vx) = 1$. 
For each $X \subseteq \{1,2,3\}$, we show the probability $p(X):= \Pr{}_{\!{\vz}}[T(\vz) = 1 \mid \vz \in \Comp(\vy_X)]$, where $\vy_X$ is the partial instance that 
is obtained from $\vx$ by fixing $\vy_X[i] = \bot$ for each $i \not\in X$.  We observe, for instance, that $\vx$ itself is neither a minimum nor a minimal $1$-SR for $\vx$ under $T$, as $\vy_{\{1,3\}} = (1,\bot,1)$ is also a 
$1$-SR. In turn, $\vy_{\{1,3\}}$ is both a minimal and a minimum $1$-SR for $\vx$ under $T$. The partial instance $\vy_{\{1,3\}}$ is not, however, a minimal or a minimum 
$\nicefrac{3}{4}$-SR for $\vx$ under $T$, as $\vy_{\{1\}} = (1,\bot,\bot)$ is also a $\nicefrac{3}{4}$-SR.
} 
\end{example} 

\begin{example}
	{\em 
	Consider the decision tree $T$ over features $\{1, 2, 3\}$ shown in Figure~\ref{fig:dt-2} and the input instance $\vx = (1,1,1)$. Notice that $T(\vx) = 1$. Exactly as in Example~\ref{ex:psr}, we display as well the probabilities $p(X)$ for each $X \subseteq \{1,2,3\}$. Interestingly, this example illustrates that Observation~\ref{obs:remove-one} does not hold when $\delta < 1$. Indeed, consider that $\vy_{\{1,2,3\}}$ is a $\nicefrac{5}{8}$-SR which is not minimal, as $\vy_{\emptyset}$ is also a $\nicefrac{5}{8}$-SR, but if we remove any single feature from $\vy_{\{1, 2, 3\}}$, we obtain a partial instance which is not a $\nicefrac{5}{8}$-SR.
	}
	\label{ex:psr-2}
\end{example}

As illustrated on Example~\ref{ex:psr-2}, it is not true in general that if $\vy' \subset \vy$ then 
\[
\Pr{}_{\!{\vz}}[T(\vz) = T(\vx) \mid \vz \in \Comp(\vy')] \leq \Pr{}_{\!{\vz}}[T(\vz) = T(\vx) \mid \vz \in \Comp(\vy)],
\]
which means that standard algorithms for finding minimal sets holding monotone predicates (see e.g., ) cannot be used to compute minimal $\delta$-SRs.

%\begin{example} \label{ex:psr} 
%{\em Consider the decision tree $T$ over features $\{1,2,3\}$ shown in Figure \ref{fig:dt}  and the input instance $\es = (1,1,1)$. Notice that $T(\es) = 1$. 
%For each $X \subseteq \{1,2,3\}$, we show the value $\delta(X)$ which is the maximum value of a $\delta \in (0,1]$ such that $\es_X$ is a $\delta$-SR for 
%$\es$ under $T$. Here, $\es_X$ is the partial instance that 
%is obtained from $\es$ by fixing $\es[i] = \bot$, for each $i \not\in X$.
%
%We observe, for instance, that $\es$ is neither a minimum nor a minimal $1$-SR (or simply SR) for $\es$ under $T$, as $\es' = (1,\bot,1)$ is also a 
%$1$-SR. In turn, $\es'$ is both a minimal and a minimum $1$-SR for $\es$ under $T$. The partial instance $\es'$ is not, however, a minimal or a minimum 
%$\nicefrac{3}{4}$-SR for $\es$ under $T$, as $\es'' = (1,\bot,\bot)$ is also a $\nicefrac{3}{4}$-SR.   \qed } 
%\end{example} 

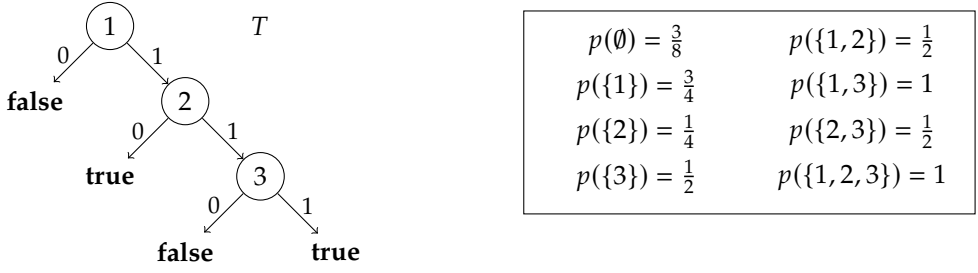
\begin{figure} 
\begin{tikzpicture}
\node[draw,circle] (1) {$1$};
\node[draw,circle] (2) at (1,-1) {$2$};
\draw[->] (1) -- (2) node [near end, above] {{\small 1}};
\node[draw,circle] (3) at (2,-2) {$3$};
\draw[->] (2) -- (3) node [near end, above] {{\small 1}};
\node[] (4) at (-1,-1) {$\false$};
\draw[->] (1) -- (4) node [near end, above] {{\small 0}};
\node[] (5) at (0,-2) {$\true$};
\draw[->] (2) -- (5) node [near end, above] {{\small 0}};
\node[] (6) at (1,-3) {$\false$};
\draw[->] (3) -- (6) node [near end, above] {{\small 0}};
\node[] (7) at (3,-3) {$\true$};
\draw[->] (3) -- (7) node [near end, above] {{\small 1}};
\node[] (8) at (2,0) {$T$};
\draw (5.5,0.2) rectangle (11.5,-2.5);
\node[] (8) at (7,-0.2) {$p(\emptyset) = \frac{3}{8}$};
\node[] (9) at (7,-0.8) {$p(\{1\}) = \frac{3}{4}$};
\node[] (10) at (7,-1.4) {$p(\{2\}) = \frac{1}{4}$};
\node[] (11) at (7,-2.0) {$p(\{3\}) = \frac{1}{2}$};
\node[] (12) at (10,-0.2) {$p(\{1,2\}) = \frac{1}{2}$};
\node[] (12) at (10,-0.8) {$p(\{1,3\}) = 1$};
\node[] (12) at (10,-1.4) {$p(\{2,3\}) = \frac{1}{2}$};
\node[] (12) at (10,-2.0) {$p(\{1,2,3\}) = 1$};
%\node[] (8) at (7,-0.2) {$\delta(\emptyset) = \frac{3}{8}$};
%\node[] (9) at (7,-1.2) {$\delta(\{1\}) = \frac{3}{4}$};
%\node[] (10) at (7,-1.8) {$\delta(\{2\}) = \frac{1}{4}$};
%\node[] (11) at (7,-2.4) {$\delta(\{3\}) = \frac{1}{2}$};
%\node[] (12) at (10,-0.2) {$\delta(\{1,2\}) = \frac{1}{2}$};
%\node[] (12) at (10,-0.8) {$\delta(\{1,3\}) = 1$};
%\node[] (12) at (10,-1.4) {$\delta(\{2,3\}) = \frac{1}{2}$};
%\node[] (12) at (10,-2.4) {$\delta(\{1,2,3\}) = 1$};
\end{tikzpicture} 
%\caption{The decision tree $T$ and the $\delta(X)$ values from Example \ref{ex:psr}.}
\caption{The decision tree $T$ and the values $p(X)$ from Example \ref{ex:psr}.}
\label{fig:dt}
\end{figure} 

\begin{figure} 
\centering
\begin{tikzpicture}
\node[draw,circle] (1) {$1$};
\node[draw,circle] (2) at (3,-1) {$2$};
\draw[->] (1) -- (2) node [near end, above] {{\small 1}};
\node[draw,circle] (3) at (4.5,-2) {$3$};
\draw[->] (2) -- (3) node [near end, above] {{\small 1}};
\node[draw, circle] (4) at (-2.5,-1) {$2$};
\draw[->] (1) -- (4) node [near end, above] {{\small 0}};
\node[draw, circle] (5) at (1.5,-2) {$3$};
\draw[->] (2) -- (5) node [near end, above] {{\small 0}};
\node[] (6) at (3.7,-3) {$\false$};
\draw[->] (3) -- (6) node [near end, above] {{\small 0}};
\node[] (7) at (5.3,-3) {$\true$};
\draw[->] (3) -- (7) node [near end, above] {{\small 1}};

\node[] (16) at (2.3, -3) {$\false$};
\draw[->] (5) -- (16) node [near end, above] {{\small 1}};

\node[] (17) at (0.7, -3) {$\true$};
\draw[->] (5) -- (17) node [near end, above] {{\small 0}};

\node[draw, circle] (18) at (-1,-2) {$3$};
\draw[->] (4) -- (18) node [near end, above] {{\small 1}};

\node[] (19) at (-4,-2) {$\true$};
\draw[->] (4) -- (19) node [near end, above] {{\small 0}};

\node[] (20) at (-1.8,-3) {$\true$};
\draw[->] (18) -- (20) node [near end, above] {{\small 0}};

\node[] (21) at (-0.3,-3) {$\false$};
\draw[->] (18) -- (21) node [near end, above] {{\small 1}};

\node[] (8) at (2,0.5) {$T$};
\draw (-2.5,-6.5) rectangle (3.5,-3.8);
\node[] (8) at (-1,-0.2-4) {$p(\emptyset) = \frac{5}{8}$};
\node[] (9) at (-1,-0.8-4) {$p(\{1\}) = \frac{1}{2}$};
\node[] (10) at (-1,-1.4-4) {$p(\{2\}) = \frac{1}{2}$};
\node[] (11) at (-1,-2.0-4) {$p(\{3\}) = \frac{1}{2}$};
\node[] (12) at (2,-0.2-4) {$p(\{1,2\}) = \frac{1}{2}$};
\node[] (13) at (2,-0.8-4) {$p(\{1,3\}) = \frac{1}{2}$};
\node[] (14) at (2,-1.4-4) {$p(\{2,3\}) = \frac{1}{2}$};
\node[] (15) at (2,-2.0-4) {$p(\{1,2,3\}) = 1$};
\end{tikzpicture} 
%\caption{The decision tree $T$ and the $\delta(X)$ values from Example \ref{ex:psr}.}
\caption{The decision tree $T$ and the values $p(X)$ from Example \ref{ex:psr-2}.}
\label{fig:dt-2}
\end{figure}
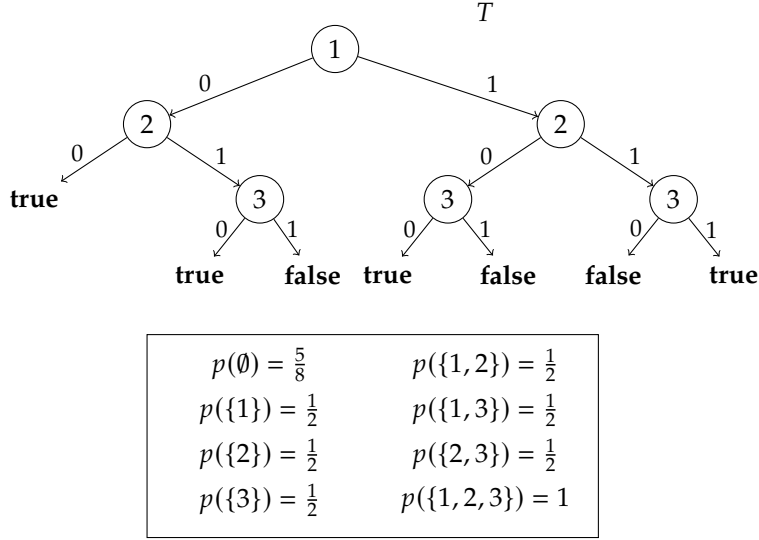

%\subsection{Computing probabilistic sufficient reasons} 
The problems of computing minimum and minimal $\delta$-SR on decision trees were 
defined and left open by~\citep{Izza2021EfficientEW, https://doi.org/10.48550/arxiv.2205.09569}. These problems are formally defined as follows.
%We start with the latter problem, which is formally defined next. 
%
\begin{center}
\fbox{\begin{tabular}{lp{8cm}}
{\small PROBLEM:} : & \minimum\\
{\small INPUT} : & A decision tree $T$ of dimension $n$, an instance $\vx$ of dimension $n$ and $\delta \in (0,1]$
%of dimension $d$, 
%\\ & $x \in \{0,1\}^d$ an instance, and $0 \leq \delta \leq 1$
\\ 
{\small OUTPUT} : & A minimum $\delta$-SR for $\vx$ under $T$  
\\
\end{tabular}}
\end{center}
\begin{center}
\fbox{\begin{tabular}{lp{8cm}}
{\small PROBLEM:} : & \minimal\\
{\small INPUT} : & A decision tree $T$ of dimension $n$, an instance $\vx$ of dimension $n$ and $\delta \in (0,1]$
%of dimension $d$, 
%\\ & $x \in \{0,1\}^d$ an instance, and $0 \leq \delta \leq 1$
\\ 
{\small OUTPUT} : & A minimal $\delta$-SR for $\vx$ under $T$  
\\
\end{tabular}}
\end{center}

\section{The Complexity of Probabilistic Sufficient Reasons on Decision Trees}
\label{sec:Complexity}
In what follows, we show that neither  \minimum\ nor \minimal\ can be solved in
polynomial time (unless \mbox{$\ptime = \np$}). We first consider the problem $\minimum$, and in fact prove a stronger result by considering the family of problems 
$\delta\text{-}\minimum$ where $\delta \in (0,1]$ is assumed to be fixed. 
%When $\delta$ is fixed we call this problem {\sf Compute-Minimum-$\delta$-SR}.
More precisely, 
we obtain as a corollary 
%It is a direct corollary of the second part
of Theorem \ref{thm:neurips2020} that $1\text{-}\minimum$ cannot be solved efficiently. 
%assuming $\textrm{P} \neq \textrm{NP}$, 
%this problem cannot be solved efficiently when $\delta = 1$.
Moreover, a non-trivial 
modification of the proof of this theorem shows that this
negative result continues to hold for every fixed $\delta \in (0,1]$.

\begin{theorem} 
  Fix $\delta \in (0,1]$. Then assuming that $\ptime \neq \np$, there is no polynomial-time algorithm for {\em $\delta\text{-}\minimum$}.
  %, under the assumption that $\textrm{P} \neq \textrm{NP}$, the problem 
%{\sf Compute-Minimum-$\delta$-SR} cannot be solved in polynomial time. 
\label{prop:delta-minimum-hardness}
\end{theorem}

%Let us now look at the problem of computing minimal $\delta$-SRs. 

%\begin{center}
%\fbox{\begin{tabular}{ll}
%\small{PROBLEM} : & {\sf Compute-Minimal-Prob-SR}
%\\{\small INPUT} : & $(T,x,\delta)$, for $T$ a decision tree of dimension $d$, 
%\\ & $x \in \{0,1\}^d$ an instance, and $0 \leq \delta \leq 1$
%\\ 
%{\small OUTPUT} : & A minimal $\delta$-SR for $x$ under $T$  
%\\
%\end{tabular}}
%\end{center}

Let us now look at the problem $\minimal$.
%of computing minimal $\delta$-SRs. 
When $\delta =1$, this problem can be solved in polynomial time as stated in
%the first part of
Theorem \ref{thm:neurips2020}. However, it is conjectured by \cite{Izza2021EfficientEW} that assuming $\ptime \neq \np$, this positive behavior does not extend to the general problem $\minimal$,
%{\sf Compute-Minimal-Prob-SR},
in which $\delta$ is an input confidence parameter. Our main result confirms that this conjecture is correct.
%is the case. 

\begin{theorem}
\label{thm:delta-minimal-hardness}
Assuming that $\ptime \neq \np$, there is no polynomial-time algorithm for {\em $\minimal$}.
%Under the assumption that $\textrm{P} \neq \textrm{NP}$, the problem 
%{\sf Compute-Minimal-Prob-SR} cannot be solved in polynomial time. 
\end{theorem} 

\begin{proof}[Proof sketch]
%Supplementary material.
We first show that the following decision problem, called 
$\checksub$, is \np-hard.
%\np-complete. 
This problem takes as input a tuple $(T,\vx)$, for $T$ a decision tree of dimension $n$, and 
  $\vx \in \{0,1,\bot\}^n$ a partial instance, and the goal is to decide whether there is a partial instance $\vy \subsetneq \vx$ 
with $\Pr{}_{\!{\vz}}[T(\vz) = 1 \mid \vz \in \Comp(\vy)] \geq \Pr{}_{\!{\vz}}[T(\vz) = 1 \mid \vz \in \Comp(\vx)]$. 
We then show that if \minimal\ admits 
a polynomial time algorithm, then $\checksub$ is in \ptime, which contradicts the assumption that $\ptime \neq \np$. 
The latter reduction %is very delicate and
requires an involved construction exploiting certain properties of the hard instances for $\checksub$. 
%requires a deep understanding of how can one construct in polynomial time 
%decision trees satisfying certain properties required by the proof. 

To show that $\checksub$ is \np-hard, we use a polynomial time reduction from a decision problem over formulas in $\cnf$, called {\sf Minimal-Expected-Clauses}, 
and which we also show 
to be \np-hard. Both the \np-hardness of {\sf Minimal-Expected-Clauses} and 
the reduction from {\sf Minimal-Expected-Clauses} to $\checksub$ may be of independent interest. 

We now define the problem {\sf Minimal-Expected-Clauses}. 
Let $\varphi$ be a $\cnf$ formula over variables $X = \{x_1,\dots, x_n\}$. Partial assignments of the variables in $X$,  as well as
the notions of subsumption and completions over them, are defined in exactly the same way as for partial instances over features.   
%
%
%A \emph{partial assignment} is a mapping $\mu: \{x_1,\dots, x_n\}\to \{0,1,\bot\}$, where the value $\bot$ indicates that a variable is undefined. 
%For partial assignments $\mu$ and $\mu'$, we write $\mu\subseteq \mu'$ if $\mu(x)\neq\bot$ implies $\mu(x) = \mu'(x)$, for all variables $x$ of $\varphi$. We write $\mu \subsetneq \mu'$ if $\mu \subseteq \mu'$ and $\mu\neq \mu'$. 
%An \emph{assignment} $\sigma$ is simply a partial assignment where $\sigma(x)\in\{0,1\}$, for 
%all variables $x$; that is, no variable is undefined. 
%If $\mu\subseteq \sigma$ for a partial assignment $\mu$ and an assignment $\sigma$, we say that $\sigma$ is a \emph{completion} of $\mu$. We denote by $\textsc{Comp}(\mu)$ the set of all completions of the partial assignment $\mu$. 
For a partial assignment $\mu$ over $X$, 
we denote by $E(\varphi, \mu)$ the expected number of clauses of $\varphi$ satisfied by a random completion of $\mu$. 
We then consider the following problem for fixed $k\geq 2$ (recall that a $k$-$\cnf$ formula is a $\cnf$ formula where each clause has at most $k$ literals):
%Note that by linearity of expectation, we can write $E(\varphi, \mu)$ as the sum:  
%$$E(\varphi, \mu) = \sum_{\text{$C$ clause of $\varphi$}} \Pr(\text{$C$ is satisfied for a random completion of $\mu$}).$$
%In turn, it can be shown that:
%\begin{itemize}
%\item $\Pr(\text{$C$ is satisfied for a random completion of $\mu$}) = 1$, if there is a positive literal $x$ in $C$ with $\mu(x)=1$, or there is a negative literal $\neg x$ in $C$ with $\mu(x)=0$.
%\item Else, $\Pr(\text{$C$ is satisfied for a random completion of $\mu$}) = 1-\frac{1}{2^\eta}$, where $\eta$ is the number of literals in $C$ of the form $x$ or $\neg x$ with $\mu(x) = \bot$. \\
%\end{itemize}
%Note that for an assignment $\sigma$, $E(\varphi, \sigma)$ is simply the number of clauses of $\varphi$ satisfied by $\sigma$. 
%A $k$-CNF formula is a CNF formula where each clause has at most $k$ literals. 
%We then consider the following problem for fixed $k\geq 2$ (recall that a $k$-CNF formula is a CNF formula where each clause has at most $k$ literals):
\begin{center}
\fbox{\begin{tabular}{ll}
\small{PROBLEM} : & {\sf $k$-Minimal-Expected-Clauses}
\\{\small INPUT} : & $(\varphi, \sigma)$, for $\varphi$ a $k$-$\cnf$ formula and $\sigma$ a partial assignment
\\ 
{\small OUTPUT} : & {\sf Yes}, if there is a partial assignment $\mu\subsetneq \sigma$  
such that $E(\varphi, \mu)\geq E(\varphi, \sigma)$ 
\\ & and {\sf No} otherwise
\end{tabular}}
\end{center}

We show that {\sf $k$-Minimal-Expected-Clauses} is \np-hard even for $k=2$, via a reduction from the well-known clique problem. 
Finally,  the reduction from {\sf $k$-Minimal-Expected-Clauses}  to $\checksub$ builds an instance $(T,\vx)$ from $(\varphi, \sigma)$ in a way that there is a direct correspondence
between partial assignments $\mu\subseteq \sigma$ and partial instances $\vy \subseteq \vx$, satisfying that 
\[
\Pr{}_{\!{\vx}}[T(\vz) = 1 \mid \vz \in \Comp(\vy)] = \frac{E(\varphi, \mu)}{m},
\]
where $m$ is the number of clauses  of $\varphi$. This implies that $(\varphi, \sigma)$ is a {\sf Yes}-instance for {\sf $k$-Minimal-Expected-Clauses} if and only if 
$(T,\vx)$ is a {\sf Yes}-instance for $\checksub$. 
\end{proof}

\section{Tractable Cases}
\label{sec:TractableCases}

We now study restrictions on decision trees that lead to polynomial time algorithms for $\minimum$ or $\minimal$, hence 
avoiding the general intractability results shown in the previous section. We identify two such restrictions: \emph{bounded split number} and \emph{monotonicity}. 

\subsection{Bounded split number}

Let $T$ be a decision tree of dimension $n$. 
For a set $U$ of nodes of $T$, we denote by $\feat(U)$ the set of features from $\{1,\dots, n\}$ labeling the nodes in $U$.
For each node $u$ of $T$, we denote by $N^{\downarrow}_u$ the set of nodes appearing in $T_u$, that is, the subtree of $T$ rooted at $u$. 
On the other hand, we denote by $N^{\uparrow}_u$ the set of nodes of $T$ minus the set of nodes $N^{\downarrow}_u$.  
We define the \emph{split number} of the decision tree $T$ to be
\[
\max_{\text{node $u$ in $T$}} \left|\feat\left(N^{\downarrow}_u\right) \cap \feat\left(N^{\uparrow}_u\right)\right|.
\]

Intuitively, the split number of a decision tree $T$ is a measure of the interaction (number of common features) between the subtrees of the form $T_u$ and their exterior. 
A small split number allows us to essentially treat each subtree $T_u$ independently (in particular, the left and right subtrees below any node), which in turn leads 
to efficient algorithms for the problems $\minimum$ and $\minimal$.

\begin{theorem}
Let $c\geq 1$ be a fixed integer. Both $\minimum$ and $\minimal$ can be solved in polynomial time for decision trees 
with split number at most $c$. 
\label{thm:split-number-poly}
\end{theorem}

\begin{proof}[Proof Sketch]
 It suffices to provide a polynomial time algorithm for $\minimum$. (The same algorithm works for $\minimal$  
as a minimum $\delta$-SR is in particular minimal.) In turn, using standard arguments, 
it is enough to provide a polynomial time algorithm for the following decision problem
$\checkmin$: Given a tuple $(T,\vx,\delta, k)$, where $T$ is a decision tree of dimension $n$, $\vx\in\{0,1,\bot\}^n$ is a partial instance, 
$\delta\in (0,1]$, and $k\geq 0$ is an integer, decide whether there is a partial instance $\vy \subseteq \vx$ such that $n-|\vy|_{\bot}\leq k$ and 
$\Pr{}_{\!{\vz}}[T(\vz) = 1\mid \vz \in \Comp(\vy)]\geq \delta$. 

In order to solve $\checkmin$ over an instance $(T,\vx,\delta, k)$, where $T$ has split number at most $c$, we apply dynamic programming 
over $T$ in a bottom-up manner. Let $Z\subseteq\{1,\dots, n\}$ be the set of features defined in $\vx$, that is, features $i$ with $\vx[i]\neq \bot$. 
Those are the features we could eventually remove when looking for $\vy$. 
For each node $u$ in $T$, we solve a polynomial number of subproblems over the subtree $T_u$. We define
\[
{\sf Int}(u):= \feat\left(N^{\downarrow}_u\right) \cap \feat\left(N^{\uparrow}_u\right)\cap Z \qquad \qquad {\sf New}(u):= \left(\feat\left(N^{\downarrow}_u\right)\setminus {\sf Int}(u)\right)\cap Z.
\]
In other words, ${\sf Int}(u)$ are the features appearing both inside and outside $T_u$, while ${\sf New}(u)$ are the features only inside $T_u$, 
that is, the new features introduced below $u$. Both sets are restricted to $Z$ as features not in $Z$ play no role in the process. 

Each particular subproblem is indexed by a possible size $s\in\{0,\dots, k\}$ and a possible set $J\subseteq {\sf Int}(u)$ and the goal is to compute the quantity:
\[
p_{u,s,J} := \max_{\vy\in\, \mathcal{C}_{u,s,J}} \Pr{}_{\!{\vz}}[T_u(\vz) = 1\mid \vz \in \Comp(\vy)],
\]
where $\mathcal{C}_{u,s,J}$ is the space of partial instances $\vy \subseteq \vx$  with $n-|\vy|_{\bot}\leq s$ and such that $\vy[i]=\vx[i]$ for 
$i\in J$ and $\vy[i]=\bot$ for $i\in {\sf Int}(u)\setminus J$. In other words, the set $J$ fixes the behavior on ${\sf Int}(u)$ (keep features in $J$, remove features in ${\sf Int}(u)\setminus J$) and hence the maximization occurs over choices on the set ${\sf New}(u)$ (which features are kept and which features are removed). The key idea is that $p_{u,s,J}$ can be computed inductively using the information already computed for the children $u_1$ and $u_2$ of $u$. 
Intuitively, this holds since the common features between $T_{u_1}$ and $T_{u_2}$ are at most $c$, which is a fixed constant, 
and hence we can efficiently synchronize the information stored for 
$u_1$ and $u_2$. Finally, to solve the instance $(T,\vx,\delta, k)$ we simply check whether $p_{r, k,\emptyset}\geq \delta$, for the root $r$ of $T$.
\end{proof}

\subsection{Monotonicity}

\emph{Monotonic} classifiers have been studied in the context of XAI as they often present tractable cases for different explanations, as shown by~\citet{pmlr-v139-marques-silva21a}. The computation of minimal sufficient reasons for monotone models was known to be in $\ptime$ since the work of~\cite{Goldsmith2005}. We show that this is also the case for computing minimal $\delta$-SRs under a mild assumption on the class of models. 
	
Let us define the ordering $\preceq$ for instances in $\{0, 1\}^n$ as follows:
\[
\vx \preceq \vz \quad \text{iff} \quad \vx[i] \leq  \vz[i], \text{ for all } i \in \{1,\dots, n\}.
\]
We can now define monotonicity as follows.

\begin{definition}
A model $\cal M$ of dimension $n$ is said to be monotone if for every pair of instances $\vx,\vz  \in \{0,1\}^n$, it holds that:
\[
   \vx  \preceq \vz \implies \M(\vx) \leq \M(\vz).
\]
\end{definition}

We now prove that the problem of computing minimal probabilistic sufficient reasons can be solved in polynomial time for 
any class $\frak C$ of monotone Boolean models for which the problem of counting positive completions can be solved efficiently. Formally, 
the latter problem is defined as follows: given a model $\M \in \frak C$ of dimension $n$ and a partial instance $\vy \in \{0,1,\bot\}^n$, 
compute $|\{ \vx \in \textsc{Comp}(\vy) \mid \M(\vx) = 1\}|$. We call this problem ${\frak C}\text{-}\counting$

\begin{theorem}
Let $\mathfrak{C}$ be a class of monotone Boolean models such that the
problem ${\mathfrak{C}}\text{-}\counting$ can be solved in polynomial time.
Then $\minimal$ can be solved in polynomial time over $\mathfrak{C}$.
\label{thm:monotone-poly}
\end{theorem}

\begin{proof}[Proof sketch]
Consider a partial instance $\vy$ of dimension $n$ and $i \in \{1,\dots, n\}$. Suppose $\vy[i] \neq \bot$. 
We write $\vy \setminus \{i\}$ for the partial instance $\vy'$ 
that is exactly equal to $\vy$ save for the fact that $\vy'[i] = \bot$.
 We make use of the following lemma, which is a probabilistic counterpart to Observation~\ref{obs:remove-one}.   

\begin{lemma} \label{lemma:monotone} 
Let $\mathfrak{C}$ be a class of monotone models, ${\cal M} \in \mathfrak{C}$ a model of dimension $n$, and $\vx \in \{0, 1\}^n$ an instance. 
Consider any $\delta \in (0,1]$. Then if $\vy \subseteq \vx$ is a $\delta$-SR for $\vx$ under $\cal M$ which is not minimal, then there is a partial instance $\vy \setminus \{i\}$, for some $i \in \{1,\dots,n\}$, that is also a $\delta$-SR for $\vx$ under $\cal M$.
\label{lemma:monotone-close-exp}
\end{lemma}

With this lemma we can prove Theorem~\ref{thm:monotone-poly}. In fact, 
consider the following algorithm, a slight variant of Algorithm~\ref{alg:minimal}: 

\begin{algorithm}
	\KwInput{Monotone model $\M$ and instance $\vx$, both of dimension $n$, together with $\delta \in (0, 1]$. Assume without loss of generality that $\M(\vx) = 1$.}
	\KwOutput{A $\delta$-SR $\vy$ for $\vx$ under $\M$.}
	\vspace{0.5em}
	$\vy \gets \vx$\\
	\While{true} {
		$\text{reduced} \gets \false$\\
		\For{$i \in \{1, \ldots, n\}$}{
			$\hat{\vy} \gets \vy \setminus \{i\}$\\
			\If{$|\{ \vx \in \Comp(\hat{\vy}) \mid \M(\vx) = 1 \}| \geq \delta 2^{|\hat{\vy}|_\bot}$}{
				$\vy \gets \hat{\vy}$\\
				$\text{reduced} \gets \true$\\
				\Break
			}
		}
		\If{($\neg \text{reduced}$) or $|\vy|_\bot = n$}{
			\Return $\vy$
		}
	}
	\caption{Compute a $\delta$-SR over a monotone model}
	\label{alg:minimal-delta}
\end{algorithm}

As by hypothesis $|\{ \vx \in \Comp(\hat{\vy}) \mid \M(\vx) = 1 \}|$  can be computed in polynomial time, the runtime of this algorithm is also polynomial, by the same analysis of Algorithm~\ref{alg:minimal}.

%For correctness, note that the invariant that $y$ is a $\delta$-SR is preserved throughout the algorithm. Assume expecting a contradiction that the algorithm returned some partial instance $w$ (in variable $y$ at that point of the algorithm) that is a $\delta$-SR but not minimal. Then because of Lemma~\ref{lemma:monotone-close-exp} there must exist a $z_i = w \setminus \{i\}$ that holds the condition, and thus step 2 will be repeated with $y \coloneqq z_i$, contradicting that $w$ was returned by the algorithm.
%
\end{proof}

As a corollary, the computation of minimal probabilistic sufficient reasons can be carried out in polynomial time not only over monotone decision trees, but 
also over monotone free BDDs. 

\begin{corollary} \label{coro:monotone}
The problem of computing minimal $\delta$-SRs can be solved in polynomial time over the class of monotone free BDDs. 
\end{corollary}

\begin{remark}
		The proof of Theorem~\ref{thm:neurips2020} uses only a monotone decision tree, and thus we cannot expect to compute minimum $\delta$-sufficient reasons in polynomial time for any $\delta \in (0, 1]$. However, the proof of Theorem~\ref{prop:delta-minimum-hardness} uses decision trees that are not monotone, and thus a new proof would be required to prove $\np$-hardness for every fixed $\delta \in (0, 1]$.
\end{remark}

\section{SAT to the Rescue!}
\label{section:sat}

Despite the theoretical intractability results presented earlier on, many $\np$-complete problems can be solved over practical instances with the aid of \emph{SAT solvers}.

By definition, any $\np$-complete problem can be \emph{encoded} as a $\cnf$ satisfiability (SAT) problem, which is then solved by a highly optimized program, a \emph{SAT solver}. In particular, if a satisfying assignment is found for the $\cnf$ instance, one can translate such an assignment back to a solution for the original problem. In fact, this paradigm has been successfully used in the literature for other explainability problems~\citep{iisms-aaai22a, Ignatiev2021, explainTrees}.  
 The effectiveness of this approach is highly dependent on the particular encoding being used~\citep{handsat}, where the aspect that arguably impacts performance the most is the size of the encoding, measured as the number of clauses of the resulting $\cnf$ formula. 
 
 In this section we present an encoding that uses some standard automated reasoning techniques (e.g., \emph{sequential encondings}~\citep{Sinz2005}) combined with ad-hoc \emph{bit-blasting}~\citep{handsat}, where the arithmetic operations required for probabilistic reasoning are implemented as Boolean circuits and then encoded as $\cnf$ by manual Tseitin transformations. The appendix describes experimentation both over synthetic datasets and standard datasets (MNIST) and reports empirical results. A recent report by~\citet{https://doi.org/10.48550/arxiv.2205.09569}
 %\footnote{We plan to include a comparison with our approach in a subsequent version of this paper.}
  also uses automated reasoning for this problem, although through an SMT solver.

Let us consider the decision version of the \minimum\ problem, which can be stated as follows. 

\begin{center}
\fbox{\begin{tabular}{lp{8cm}}
{\small PROBLEM:} : & {\sf Decide-Minimum-SR}\\
{\small INPUT} : & A decision tree $T$ of dimension $n$, an instance $\vx$ of dimension $n$, an integer $k \leq n$,  and $\delta \in (0,1]$
\\ 

{\small OUTPUT} : & \textsf{Yes} if there is a $\delta$-SR $\vy$ for $\vx$ under $T$ with  $|\vy|_\bot \leq k$, and \textsf{No} otherwise.

\\
\end{tabular}}
\end{center}

This problem is $\np$-complete. Membership is already proven by~\cite{Izza2021EfficientEW}. 
%Indeed, one can see that {\sf Decide-Minimum-SR} is in $\np$ as \textbf{Yes}-instances can be certified by their corresponding $\es'$, for which one can trivially check that $|\es'|_\bot \leq k$, and then verify that 
%\[
	%\Pr{}_{\!{\es_2}}[T(\es_2) = T(\es) \mid \es_2 \in \textsc{Comp}(\es')] \	\geq \delta, 
%\]
%which can be done in polynomial time as mentioned in Section~\ref{sec:ProbabilisticSufficientReasons}. 
Hardness follows directly from Theorem~\ref{prop:delta-minimum-hardness}, as if one were able to solve {\sf Decide-Minimum-SR} in polynomial time, then a simple binary search over $k$ would allow solving $\minimum$ in polynomial time.

\subsection{Deterministic Encoding}
Let us first propose an encoding for the particular case of $\delta = 1$. 
First, create Boolean variables $f_i$ for $i \in \{1, \ldots, n\}$, with $f_i$ representing that $\vy[i] \neq \bot$, where $\vy$ is the desired $\delta$-SR. Then, for every node $u$ of the tree $T$, create a variable $r_u$ representing that node $u$ is \emph{reachable} by a completion of $\vy$, meaning that there exists a completion of $\vy$ that goes through node $u$ when evaluated over $T$. We then want to enforce that: 
\begin{enumerate}
	\item $r_\textsc{root} = 1$, where $\textsc{root}$ is the root of $T$. This means that the root is always reachable.
	\item The desired $\delta$-SR $\vy$ satisfies $n - |\vy|_\bot \leq k$, meaning that 
	\[
		\sum_{i=1}^n f_i \leq k.
	\] 
	\item If $T(\vx) = 1$, and $F$ is the set of $\false$ leaves of $T$, then $r_\ell = 0$ for every $\ell \in F$. This means that if we want to explain a positive instance, the completions of the explanation $\vy$ must all be positive (recall we assume $\delta=1)$, and thus no $\false$ leaf should be reachable. Conversely, If $T(\vx) = 0$, and $G$ is the set of $\true$ leaves of $T$, then $r_\ell = 0$ for every $\ell \in G$.
	\item The semantics of reachability is \emph{consistent}: if a node $u$ is reachable, and its labeled with feature $i$, then if $f_i = 0$, meaning that feature $i$ is undefined in $\vy$, both children of $u$ should be reachable too. In case $f_i = 1$, only the child along the edge corresponding to $\vx[i]$ should be reachable.
\end{enumerate}

Let us analyze the size of this encoding. Condition 1 requires a single clause to be enforced. Condition 2 can be enforced with $O(nk)$ variables and clauses using the linear encoding of~\cite{Sinz2005}. Condition 3 requires at most one clause per leaf of $T$ and thus at most $O(|T|)$ clauses. Condition 4 can be implemented with a constant number of clauses per node. We thus incur in a total of $O(nk + |T|)$ variables and clauses, which is pretty much optimal considering $\Omega(n + |T|)$ is a lower bound on the representation of the input. 

Note as well that from a satisfying assignment we can trivially recover the desired explanation $\vy$ as
\[
	\vy[i] = \begin{cases}
		\vx[i] & \text{if } f_i =1,\\
		\bot &  \text{otherwise.}
	\end{cases}
\]

An efficient encoding supporting values of $\delta$ different than $1$ is significantly more challenging and involved, and thus we only provide an outline here. An exhaustive description, together with the implementation, is provided in the supplementary material.

\subsection{Probabilistic Encoding}

In order to encode that
\[
\Pr{}_{\!{\vz}}[T(\vz) = T(\vx) \mid \vz \in \Comp(\vy)] \	\geq \delta,
\]
we will directly encode that
\[
	\left|\{ \vz \in \Comp(\vy) \mid T(\vz) = T(\vx)\}\right| \geq \left\lceil \delta 2^{|\vy|_\bot} \right\rceil,
\]
where we assume for simplicity that the ceiling can be take safely, in order to have a value we can represent by an integer. This assumption is not crucial.
As before, we will have $f_i$ variables representing whether $\vy[i] \neq \bot$ or not, and enforce that 
\[
	\sum_{i=1}^n f_i \leq k.
\] Now define variables $c_j$ for $j \in \{0, \ldots, n\}$, such that $c_j$  is true exactly when $\sum_{i=1}^n f_i = j$. This can again be done efficiently via a linear encoding.
Let $t_i$ for $i \in \{0, \ldots, n\}$ a series of variables that represent the bits of an integer $t$. The integer $t$ will correspond to the value of $\left\lceil \delta 2^{|\vy|_\bot} \right\rceil$. Note that as $|\vy|_\bot$ can only take $n+1$ different values, there are only $n+1$ different values that $t$ can take. Moreover, the value of the $c_j$ variables completely determines the value of $t$, and thus we can manually encode how the $c_j$ variables determine the bits $t_i$.
We can now assume that we have access to $t$ through its bits $t_i$. Our goal now is to build a binary representation of 
\[
	\alpha = \left|\{ \vz \in \Comp(\vy) \mid T(\vz) = T(\vx)\}\right|,
\] 
so that we can then implement the condition $\alpha \geq t$ with a Boolean circuit on their bits.

In order to represent $\alpha$, we will decompose the number of instances $\vz$ according to the leaves of $T$ as follows. If $L$ is the set of leaves of $T$ whose label matches $T(\vx)$, then
\[
 \left|\{ \vz \in \Comp(\vy) \mid T(\vz) = T(\vx)\}\right| = \sum_{\ell \in L}  \left|\left\{ \vz \in \Comp(\vy) \mid T(\vz) \rightsquigarrow \ell)\right\}\right|,
\]
where notation $T(\vz) \rightsquigarrow \ell$ means that the evaluation of $\vz$ under $T$ ends in the leaf $\ell$. For a given leaf $\ell$, we can compute its \emph{weight}
\[
	w_\ell \coloneqq \left|\{ \vz \in \Comp(\vy) \mid T(\vz) \rightsquigarrow \ell)\}\right|
\]
as described next. Let $F_\ell$ be the set of labels appearing in the unique path from the root of $T$ down to $\ell$. Now let $u_\ell$ be the number of undefined features of $\vy$ along said unique path, and thus $u_\ell$ can be defined as follows.
\[
	u_\ell \coloneqq | F_\ell \setminus \{ i \mid f_i = 1\}|.
\]
The sets $F_\ell$ depend only on $T$ and thus can be precomputed. Therefore we can use a linear encoding again to define the values $u_\ell$. It is simple to observe now that 

\[
	w_\ell = 2^{u_\ell},
\]
which means that by representing the values $u_\ell$ we can trivially
represent the values $w_\ell$ in binary (they will simply consist of a $1$ in their $(u_\ell + 1)$-th position right-to-left), and then implement a Boolean addition circuit to compute 
\[
\alpha = \sum_{\ell} w_\ell.
\]
This concludes the outline of the encoding. Its number of variables and clauses can be easily seen to be at most $O(n^2 |T| + nk)$.

\section{Conclusions and Future Work} \label{sec_future} 

We have settled the complexity of two open problems in formal XAI, proving that both minimal and minimum size probabilistic explanations can be hard to obtain even for decision trees. These results further support the idea that decision trees may not always be interpretable~\citep{NEURIPS2020_b1adda14, lipton2018mythos, NEURIPS2021_60cb558c, Audemard}, while being the first  results to do so even in a probabilistic setting.
Our study has focused on Boolean models, and moreover it is based on the assumption that features are independent and identically distributed, with probability $\nicefrac{1}{2}$. 
Naturally, our hardness results in this restricted setting imply hardness in more general settings, but we leave as future work the study of tractable cases (e.g., bounded split-number trees, monotone classes that allow counting) under general, potentially correlated, distributions.

The results proven in this paper suggest that minimal (or minimum) explanations might be hard to obtain in practice even for decision trees, especially in problems where the feature space has large dimension. A way to circumvent the limitations proven in our work is to relax the guarantee of minimality, or introduce some probability of error, as done in the work of~\citet{blanc2021provably}. A promising direction of future research is to better understand the kind of guarantees and settings in which is still possible to obtain tractability.

Finally, our work leaves open some interesting technical questions:
\begin{enumerate}
\item What is the {\em parameterized} complexity of computing minimum $\delta$-SRs over decision trees, assuming that the parameter is  
the size of the explanation one is looking for?  It is not hard to see that $\mathrm{W}[2]$ hardness follows from our proofs (i.e., they are parameterized reductions all the way to Set Cover), but membership in any class is fully open. Parameterized complexity might be of particular interest for this class of problems as explanations might reasonably be expected to be very small in practice.
\item Does Theorem \ref{prop:delta-minimum-hardness} continue to hold for monotone decision trees? That is, is it the case that computing minimum $\delta$-SRs over monotone decision trees is hard for every fixed $\delta \in (0,1]$? 
\item Is it the case that the hardness of computation for minimal sufficient reasons holds for a {\em fixed} $\delta \in (0,1]$? If so, does it hold for every fixed such a $\delta$, 
or only for some? 
\item Is it possible to strengthen the hardness results in this paper and show that these problems cannot be approximated efficiently? Again by following our chain of reductions one can derive inapproximability results coming from the inapproximability of Set Cover (Theorem \ref{prop:delta-minimum-hardness}) or Clique (Theorem \ref{thm:delta-minimal-hardness}).
\item Is it possible to extend the positive behavior of decision trees with bounded split number to more powerful Boolean ML models such as free BDDs? 
\end{enumerate}

\newpage

\bibliographystyle{abbrvnat}
\bibliography{bibliography}

\begin{thebibliography}{34}
\providecommand{\natexlab}[1]{#1}
\providecommand{\url}[1]{\texttt{#1}}
\expandafter\ifx\csname urlstyle\endcsname\relax
  \providecommand{\doi}[1]{doi: #1}\else
  \providecommand{\doi}{doi: \begingroup \urlstyle{rm}\Url}\fi

\bibitem[Arenas et~al.(2021)Arenas, B\'{a}ez, Barcel\'{o}, P\'{e}rez, and
  Subercaseaux]{NEURIPS2021_60cb558c}
M.~Arenas, D.~B\'{a}ez, P.~Barcel\'{o}, J.~P\'{e}rez, and B.~Subercaseaux.
\newblock Foundations of symbolic languages for model interpretability.
\newblock In M.~Ranzato, A.~Beygelzimer, Y.~Dauphin, P.~Liang, and J.~W.
  Vaughan, editors, \emph{Advances in Neural Information Processing Systems},
  volume~34, pages 11690--11701. Curran Associates, Inc., 2021.
\newblock URL
  \url{https://proceedings.neurips.cc/paper/2021/file/60cb558c40e4f18479664069d9642d5a-Paper.pdf}.

\bibitem[Arrieta et~al.(2020)Arrieta, D{\'{\i}}az-Rodr{\'{\i}}guez, Ser,
  Bennetot, Tabik, Barbado, Garcia, Gil-Lopez, Molina, Benjamins, Chatila, and
  Herrera]{BarredoArrieta2020}
A.~B. Arrieta, N.~D{\'{\i}}az-Rodr{\'{\i}}guez, J.~D. Ser, A.~Bennetot,
  S.~Tabik, A.~Barbado, S.~Garcia, S.~Gil-Lopez, D.~Molina, R.~Benjamins,
  R.~Chatila, and F.~Herrera.
\newblock Explainable artificial intelligence ({XAI}): Concepts, taxonomies,
  opportunities and challenges toward responsible {AI}.
\newblock \emph{Information Fusion}, 58:\penalty0 82--115, June 2020.
\newblock \doi{10.1016/j.inffus.2019.12.012}.
\newblock URL \url{https://doi.org/10.1016/j.inffus.2019.12.012}.

\bibitem[Audemard et~al.(2021{\natexlab{a}})Audemard, Bellart, Bounia, Koriche,
  Lagniez, and Marquis]{DBLP:conf/kr/AudemardBBKLM21}
G.~Audemard, S.~Bellart, L.~Bounia, F.~Koriche, J.~Lagniez, and P.~Marquis.
\newblock On the computational intelligibility of boolean classifiers.
\newblock In M.~Bienvenu, G.~Lakemeyer, and E.~Erdem, editors,
  \emph{Proceedings of the 18th International Conference on Principles of
  Knowledge Representation and Reasoning, {KR} 2021, Online event, November
  3-12, 2021}, pages 74--86, 2021{\natexlab{a}}.
\newblock \doi{10.24963/kr.2021/8}.
\newblock URL \url{https://doi.org/10.24963/kr.2021/8}.

\bibitem[Audemard et~al.(2021{\natexlab{b}})Audemard, Bellart, Bounia, Koriche,
  Lagniez, and Marquis]{Audemard}
G.~Audemard, S.~Bellart, L.~Bounia, F.~Koriche, J.-M. Lagniez, and P.~Marquis.
\newblock On the explanatory power of decision trees, 2021{\natexlab{b}}.
\newblock URL \url{https://arxiv.org/abs/2108.05266}.

\bibitem[Barcel\'{o} et~al.(2020)Barcel\'{o}, Monet, P\'{e}rez, and
  Subercaseaux]{NEURIPS2020_b1adda14}
P.~Barcel\'{o}, M.~Monet, J.~P\'{e}rez, and B.~Subercaseaux.
\newblock Model interpretability through the lens of computational complexity.
\newblock In H.~Larochelle, M.~Ranzato, R.~Hadsell, M.~Balcan, and H.~Lin,
  editors, \emph{Advances in Neural Information Processing Systems}, volume~33,
  pages 15487--15498. Curran Associates, Inc., 2020.
\newblock URL
  \url{https://proceedings.neurips.cc/paper/2020/file/b1adda14824f50ef24ff1c05bb66faf3-Paper.pdf}.

\bibitem[Biere et~al.(2009)Biere, Biere, Heule, van Maaren, and Walsh]{handsat}
A.~Biere, A.~Biere, M.~Heule, H.~van Maaren, and T.~Walsh.
\newblock \emph{Handbook of Satisfiability: Volume 185 Frontiers in Artificial
  Intelligence and Applications}.
\newblock IOS Press, NLD, 2009.
\newblock ISBN 1586039296.

\bibitem[Biere et~al.(2020)Biere, Fazekas, Fleury, and
  Heisinger]{BiereFazekasFleuryHeisinger-SAT-Competition-2020-solvers}
A.~Biere, K.~Fazekas, M.~Fleury, and M.~Heisinger.
\newblock {CaDiCaL}, {Kissat}, {Paracooba}, {Plingeling} and {Treengeling}
  entering the {SAT Competition 2020}.
\newblock In T.~Balyo, N.~Froleyks, M.~Heule, M.~Iser, M.~J{\"a}rvisalo, and
  M.~Suda, editors, \emph{Proc.~of {SAT Competition} 2020 -- Solver and
  Benchmark Descriptions}, volume B-2020-1 of \emph{Department of Computer
  Science Report Series B}, pages 51--53. University of Helsinki, 2020.

\bibitem[Blanc et~al.(2021)Blanc, Lange, and Tan]{blanc2021provably}
G.~Blanc, J.~Lange, and L.-Y. Tan.
\newblock Provably efficient, succinct, and precise explanations.
\newblock In A.~Beygelzimer, Y.~Dauphin, P.~Liang, and J.~W. Vaughan, editors,
  \emph{Advances in Neural Information Processing Systems}, 2021.
\newblock URL \url{https://openreview.net/forum?id=9UjRw5bqURS}.

\bibitem[Choi et~al.(2017)Choi, Shi, Shih, and Darwiche]{choi2017compiling}
A.~Choi, W.~Shi, A.~Shih, and A.~Darwiche.
\newblock Compiling neural networks into tractable boolean circuits.
\newblock \emph{intelligence}, 2017.

\bibitem[Deng(2012)]{deng2012mnist}
L.~Deng.
\newblock The mnist database of handwritten digit images for machine learning
  research.
\newblock \emph{IEEE Signal Processing Magazine}, 29\penalty0 (6):\penalty0
  141--142, 2012.

\bibitem[Gilpin et~al.(2018)Gilpin, Bau, Yuan, Bajwa, Specter, and
  Kagal]{DBLP:conf/dsaa/GilpinBYBSK18}
L.~H. Gilpin, D.~Bau, B.~Z. Yuan, A.~Bajwa, M.~A. Specter, and L.~Kagal.
\newblock Explaining explanations: An overview of interpretability of machine
  learning.
\newblock In F.~Bonchi, F.~J. Provost, T.~Eliassi{-}Rad, W.~Wang, C.~Cattuto,
  and R.~Ghani, editors, \emph{DSAA}, pages 80--89, 2018.

\bibitem[Goldsmith et~al.(2005)Goldsmith, Hagen, and Mundhenk]{Goldsmith2005}
J.~Goldsmith, M.~Hagen, and M.~Mundhenk.
\newblock Complexity of {DNF} and isomorphism of monotone formulas.
\newblock In \emph{Mathematical Foundations of Computer Science 2005}, pages
  410--421. Springer Berlin Heidelberg, 2005.
\newblock \doi{10.1007/11549345_36}.
\newblock URL \url{https://doi.org/10.1007/11549345_36}.

\bibitem[Huang et~al.(2021)Huang, Izza, Ignatiev, and
  Marques{-}Silva]{DBLP:conf/kr/HuangII021}
X.~Huang, Y.~Izza, A.~Ignatiev, and J.~Marques{-}Silva.
\newblock On efficiently explaining graph-based classifiers.
\newblock In M.~Bienvenu, G.~Lakemeyer, and E.~Erdem, editors,
  \emph{Proceedings of the 18th International Conference on Principles of
  Knowledge Representation and Reasoning, {KR} 2021, Online event, November
  3-12, 2021}, pages 356--367, 2021.
\newblock \doi{10.24963/kr.2021/34}.
\newblock URL \url{https://doi.org/10.24963/kr.2021/34}.

\bibitem[Ignatiev and Marques-Silva(2021)]{Ignatiev2021}
A.~Ignatiev and J.~Marques-Silva.
\newblock {SAT}-based rigorous explanations for decision lists.
\newblock In \emph{Theory and Applications of Satisfiability Testing
  {\textendash} {SAT} 2021}, pages 251--269. Springer International Publishing,
  2021.
\newblock \doi{10.1007/978-3-030-80223-3_18}.
\newblock URL \url{https://doi.org/10.1007/978-3-030-80223-3_18}.

\bibitem[Ignatiev et~al.(2022)Ignatiev, Izza, Stuckey, and
  Marques-Silva]{iisms-aaai22a}
A.~Ignatiev, Y.~Izza, P.~J. Stuckey, and J.~Marques-Silva.
\newblock Using maxsat for efficient explanations of tree ensembles.
\newblock In \emph{AAAI}, 2022.

\bibitem[Izza et~al.(2020{\natexlab{a}})Izza, Ignatiev, and
  Marques{-}Silva]{DBLP:journals/corr/abs-2010-11034}
Y.~Izza, A.~Ignatiev, and J.~Marques{-}Silva.
\newblock On explaining decision trees.
\newblock \emph{CoRR}, abs/2010.11034, 2020{\natexlab{a}}.

\bibitem[Izza et~al.(2020{\natexlab{b}})Izza, Ignatiev, and
  Marques-Silva]{explainTrees}
Y.~Izza, A.~Ignatiev, and J.~Marques-Silva.
\newblock On explaining decision trees, 2020{\natexlab{b}}.
\newblock URL \url{https://arxiv.org/abs/2010.11034}.

\bibitem[Izza et~al.(2021)Izza, Ignatiev, Narodytska, Cooper, and
  Marques-Silva]{Izza2021EfficientEW}
Y.~Izza, A.~Ignatiev, N.~Narodytska, M.~C. Cooper, and J.~Marques-Silva.
\newblock Efficient explanations with relevant sets.
\newblock \emph{ArXiv}, abs/2106.00546, 2021.

\bibitem[Izza et~al.(2022)Izza, Ignatiev, Narodytska, Cooper, and
  Marques-Silva]{https://doi.org/10.48550/arxiv.2205.09569}
Y.~Izza, A.~Ignatiev, N.~Narodytska, M.~C. Cooper, and J.~Marques-Silva.
\newblock Provably precise, succinct and efficient explanations for decision
  trees, 2022.
\newblock URL \url{https://arxiv.org/abs/2205.09569}.

\bibitem[Lipton(2018)]{lipton2018mythos}
Z.~C. Lipton.
\newblock The mythos of model interpretability.
\newblock \emph{Queue}, 16\penalty0 (3):\penalty0 31--57, 2018.

\bibitem[Lundberg and Lee(2017)]{shaps}
S.~M. Lundberg and S.-I. Lee.
\newblock A unified approach to interpreting model predictions.
\newblock In \emph{NeurIPS}, pages 4765--4774, 2017.

\bibitem[Marques-Silva and Ignatiev(2022)]{formal-xai}
J.~Marques-Silva and A.~Ignatiev.
\newblock Delivering trustworthy ai through formal xai.
\newblock In \emph{AAAI}, 2022.

\bibitem[Marques{-}Silva et~al.(2020)Marques{-}Silva, Gerspacher, Cooper,
  Ignatiev, and Narodytska]{DBLP:conf/nips/0001GCIN20}
J.~Marques{-}Silva, T.~Gerspacher, M.~C. Cooper, A.~Ignatiev, and
  N.~Narodytska.
\newblock Explaining naive bayes and other linear classifiers with polynomial
  time and delay.
\newblock In \emph{NeurIPS}, 2020.

\bibitem[Marques{-}Silva et~al.(2021)Marques{-}Silva, Gerspacher, Cooper,
  Ignatiev, and Narodytska]{DBLP:conf/icml/0001GCIN21}
J.~Marques{-}Silva, T.~Gerspacher, M.~C. Cooper, A.~Ignatiev, and
  N.~Narodytska.
\newblock Explanations for monotonic classifiers.
\newblock In \emph{ICML}, pages 7469--7479, 2021.

\bibitem[Marques-Silva et~al.(2021)Marques-Silva, Gerspacher, Cooper, Ignatiev,
  and Narodytska]{pmlr-v139-marques-silva21a}
J.~Marques-Silva, T.~Gerspacher, M.~C. Cooper, A.~Ignatiev, and N.~Narodytska.
\newblock Explanations for monotonic classifiers.
\newblock In M.~Meila and T.~Zhang, editors, \emph{Proceedings of the 38th
  International Conference on Machine Learning}, volume 139 of
  \emph{Proceedings of Machine Learning Research}, pages 7469--7479. PMLR,
  18--24 Jul 2021.
\newblock URL \url{https://proceedings.mlr.press/v139/marques-silva21a.html}.

\bibitem[Pedregosa et~al.(2011)Pedregosa, Varoquaux, Gramfort, Michel, Thirion,
  Grisel, Blondel, Prettenhofer, Weiss, Dubourg, Vanderplas, Passos,
  Cournapeau, Brucher, Perrot, and Duchesnay]{scikit-learn}
F.~Pedregosa, G.~Varoquaux, A.~Gramfort, V.~Michel, B.~Thirion, O.~Grisel,
  M.~Blondel, P.~Prettenhofer, R.~Weiss, V.~Dubourg, J.~Vanderplas, A.~Passos,
  D.~Cournapeau, M.~Brucher, M.~Perrot, and E.~Duchesnay.
\newblock Scikit-learn: Machine learning in {P}ython.
\newblock \emph{Journal of Machine Learning Research}, 12:\penalty0 2825--2830,
  2011.

\bibitem[Ribeiro et~al.(2018)Ribeiro, Singh, and
  Guestrin]{DBLP:conf/aaai/Ribeiro0G18}
M.~T. Ribeiro, S.~Singh, and C.~Guestrin.
\newblock Anchors: High-precision model-agnostic explanations.
\newblock In \emph{AAAI}, pages 1527--1535, 2018.

\bibitem[Shih et~al.(2018)Shih, Choi, and Darwiche]{shih2018symbolic}
A.~Shih, A.~Choi, and A.~Darwiche.
\newblock A symbolic approach to explaining bayesian network classifiers.
\newblock \emph{arXiv preprint arXiv:1805.03364}, 2018.

\bibitem[Sinz(2005)]{Sinz2005}
C.~Sinz.
\newblock Towards an optimal {CNF} encoding of boolean cardinality constraints.
\newblock In \emph{Principles and Practice of Constraint Programming - {CP}
  2005}, pages 827--831. Springer Berlin Heidelberg, 2005.
\newblock \doi{10.1007/11564751_73}.
\newblock URL \url{https://doi.org/10.1007/11564751_73}.

\bibitem[W{\"{a}}ldchen et~al.(2021)W{\"{a}}ldchen, MacDonald, Hauch, and
  Kutyniok]{DBLP:journals/jair/WaldchenMHK21}
S.~W{\"{a}}ldchen, J.~MacDonald, S.~Hauch, and G.~Kutyniok.
\newblock The computational complexity of understanding binary classifier
  decisions.
\newblock \emph{J. Artif. Intell. Res.}, 70:\penalty0 351--387, 2021.

\bibitem[Wang et~al.(2021)Wang, Khosravi, and den
  Broeck]{DBLP:conf/ijcai/WangKB21}
E.~Wang, P.~Khosravi, and G.~V. den Broeck.
\newblock Probabilistic sufficient explanations.
\newblock In Z.~Zhou, editor, \emph{IJCAI}, pages 3082--3088, 2021.

\bibitem[Wegener(2000)]{Wegener2000}
I.~Wegener.
\newblock \emph{Branching Programs and Binary Decision Diagrams}.
\newblock Society for Industrial and Applied Mathematics, Jan. 2000.
\newblock \doi{10.1137/1.9780898719789}.
\newblock URL \url{https://doi.org/10.1137/1.9780898719789}.

\bibitem[Wegener(2004)]{DBLP:journals/dam/Wegener04}
I.~Wegener.
\newblock Bdds--design, analysis, complexity, and applications.
\newblock \emph{Discret. Appl. Math.}, 138\penalty0 (1-2):\penalty0 229--251,
  2004.

\bibitem[Yan and Procaccia(2021)]{DBLP:conf/aaai/YanP21}
T.~Yan and A.~D. Procaccia.
\newblock If you like shapley then you'll love the core.
\newblock In \emph{AAAI}, pages 5751--5759, 2021.

\end{thebibliography}

\newpage

%\input{checklist}
%
%\newpage

\appendix

\section*{Appendix}

{\bf Organization.} The supplementary material is organized as
follows. In section \ref{app:prop:delta-minimum-hardness}, we provide
a proof that there is no polynomial-time algorithm for the problem of
computing minimum $\delta$-sufficient reasons (unless $\ptime = \np$),
even when $\delta$ is a fixed value. In section
\ref{app:thm:delta-minimal-hardness}, we provide a proof that there is
no polynomial-time algorithm for the problem of computing minimal
$\delta$-sufficient reasons (unless $\ptime = \np$). In particular, we
define in this section a decision problem that we prove to be
$\np$-hard in Section \ref{subsec:hardness-decision}, and then we show
in Section \ref{subsec:decision-to-computation} that this decision
problem can be solved in polynomial-time if there exists a
polynomial-time algorithm for the problem of computing minimal
$\delta$-sufficient reasons. In Section
\ref{app:thm:split-number-poly}, we provide a proof that minimum and
minimal $\delta$-sufficient reasons can be computed in polynomial-time
when the split number (defined in Section 5.1) is bounded. Finally, we
provide in Section \ref{app:lemma:monotone} a proof of Lemma 1, which
completes the proof of tractability of the problem of computing
minimal $\delta$-sufficient reasons for each class of monotone Boolean
models for which the problem of counting positive completions can be
solved in polynomial time. Finally, we describe in Section
\ref{sec:experiments} our experimental evaluation of the encodings
given in Section 6.

%\todo[inline]{Bernardo: Fix notation at some point. At least this proof should make sure all the technical bits are correct.}

%\begin{propbis}{prop:delta-minimum-hardness}
%Fix $\delta \in (0,1]$. Then, under the assumption that $\textrm{P} \neq \textrm{NP}$, the problem 
%{\sf Compute-Minimum-$\delta$-SR} cannot be solved in polynomial time. 
%\end{propbis}

\section{Proof of Theorem \ref{prop:delta-minimum-hardness}}
\label{app:prop:delta-minimum-hardness}

\begin{thmbis}{prop:delta-minimum-hardness} 
  Fix $\delta \in (0,1]$. Then assuming that $\ptime \neq \np$, there is no polynomial-time algorithm for {\em $\delta\text{-}\minimum$}.
  %, under the assumption that $\textrm{P} \neq \textrm{NP}$, the problem 
%{\sf Compute-Minimum-$\delta$-SR} cannot be solved in polynomial time. 
%\label{prop:delta-minimum-hardness}
\end{thmbis}

%\begin{propbis}{prop:delta-minimum-hardness}
%Fix $\delta \in [0,1]$. Then, under the assumption that $\textrm{P} \neq \textrm{NP}$, the problem 
%{\sf Compute-Minimum-$\delta$-SR} cannot be solved in polynomial time. 
%\end{propbis}

Before we can prove this theorem, we will require some auxiliary
lemmas. Given rational numbers $a,b$ with $a \leq b$, recall that
notation $[a,b]$ refers to the set of rational numbers $x$ such that
$a \leq x \leq b$ (and likewise for $[a,b)$). 

\begin{lemma}
Fix $\delta \in (0, 1)$. Given as input an integer $n$ one can build in $n^{O(1)}$ time a decision tree $T_\delta$ of dimension $n$, such that
\[
    \left|\Pr_{\vz}\left[ T_\delta(\vz) = 1 \mid \vz \in \Comp(\bot^n)\right] - \delta \right| \leq \frac{1}{2^n},
\]
and moreover, there exists an instance $\vx^\dagger$ for $T_\delta$ such that every partial instance $\vy \subseteq \vx^\dagger$ holds
\[
\Pr_{\vz}\left[ T_\delta(\vz) = 1 \mid \vz \in \Comp(\vy)\right]  \leq \Pr_{\vz}\left[ T_\delta(\vz) = 1 \mid \vz \in \Comp(\bot^n)\right].
\]
\label{lemma:T_delta}
\end{lemma}

\begin{proof}
Let $c = \lfloor \delta 2^n \rfloor$, and note that 
\(
\delta - \frac{1}{2^n} \leq \frac{c}{2^n} \leq \delta
\), and thus $|\delta - \frac{c}{2^n}| \leq \frac{1}{2^n}$.
This implies that we can prove the first part of the lemma by building in polynomial time a tree $T_c$ over $n$ variables, that has exactly $c$ different positive instances, as then its probability of accepting a random completion of $\bot^n$ will be exactly $\frac{c}{2^n}$. Note as well that $c < 2^n$ as $\delta < 1$.

As a first step, let us write $c$ in binary, obtaining 
\[
    c = \alpha_0 2^{0} + \alpha_1 2^{1} + \cdots + \alpha_{n-1}2^{n-1},
\]
with $\alpha_i \in \{0, 1\}$ for each $i$. 
Then to build $T_c$ start by creating $n$ vertices, labeled $0$ through $n-1$. These $n$ labels are the variables of $T_c$. For each $i \in \{1, \ldots, n-1\}$, connect vertex labeled $i$ to vertex labeled $i-1$ with a $0$-edge, making vertex labeled $n-1$ the root of $T_c$. 
Then, for each vertex with label $i \in \{0, \ldots, n-1\}$, set its $1$-edge towards a leaf with label $\true$ if $\alpha_{i} = 1$, and towards a leaf with label $\false$ if $\alpha_{i} = 0$. The $0$-edge of vertex labeled $0$ goes towards a leaf with label $\false$. Now let us count how many different positive instances does $T_c$ have. We can do this by summing over all true leaves of $T_c$. Each true leaf comes from a $1$-edge from a vertex labeled $i \in \{0, \ldots, n-1\}$. For every $i \in \{0, \ldots, n-1\}$, if the vertex labeled $i$ has a true leaf when following its $1$-edge, then the number of instances reaching that true leaf is exactly $2^{i}$, as the variables whose value is not determined by the path to that leaf are those with labels less than $i$, which are exactly $i$ variables. Therefore, the number of different positive instances of $T_c$ along a $1$-edge is the sum of $2^i$ for every $i$ such that $\alpha_i = 1$, which is exactly $c$.  An example is given in Figure~\ref{fig:T_delta}. This concludes the proof of the first part of the lemma as the construction is clearly polynomial in $n$. For the second part, let us build $\vx^\dagger$ by setting 
\(
    \vx^\dagger[i] = 1 - \alpha_{i}
\) for every $i \in \{0, \ldots, n-1\}$.
%(we consider instances as 0-indexed in this proof for the sake of simplicity).
In the example presented in Figure~\ref{fig:T_delta}, we would build
\[
     \vx^\dagger = \begin{pmatrix} 1, & 1, & 0, & 0, & 1 \end{pmatrix}.
\]

We will now prove that for any $\vy \subseteq \vx^\dagger$, it holds that 
\[
\Pr_{\vz}\left[ T_c(\vz) = 1 \mid \vz \in \Comp(\vy)\right]  \leq \Pr_{\vz}\left[ T_c(\vz) = 1 \mid \vz \in \Comp(\bot^n)\right].
\]
We do this via a finite induction argument by strengthening our induction hypothesis; for $i \in \{0, \ldots, n-1\}$, let $T^i_c$ be the sub-tree of $T_c$ rooted at the vertex labeled $i$, and let us claim that for every $i \in \{0, \ldots, n-1\}$ we have that
\[
\Pr_{\vz}\left[ T^i_c(\vz) = 1 \mid \vz \in \Comp(\vy)\right]  \leq \Pr_{\vz}\left[ T^i_c(\vz) = 1 \mid \vz \in \Comp(\bot^n)\right],
\]
which implies what we want to show when taking $i = n-1$. The base case of the induction is when $i = 0$, in which case the claim trivially holds as if $\vy[0] = \bot$ we have equality, and if $\vy[0] \neq \bot$ then by construction 
\[
\Pr_{\vz}\left[ T^i_c(\vz) = 1 \mid \vz \in \Comp(\vy)\right] = 0.
\] 
 For the inductive case, let $i > 0$, and  proceed by cases; if $\vy[i] = \bot$, then by letting $t_i \in \{0, 1\}$ be an indicator variable for whether the leaf across the $1$-edge from vertex $i$ is labeled $\true$ we have that
\begin{align*}
\Pr_{\vz}\left[ T^i_c(\vz) = 1 \mid \vz \in \Comp(\vy)\right] &= \frac{1}{2}t_i + \frac{1}{2} \Pr_{\vz}\left[ T^{i-1}_c(\vz) = 1 \mid \vz \in \Comp(\vy)\right] \\
&\leq \frac{1}{2}t_i + \frac{1}{2} \Pr_{\vz}\left[ T^{i-1}_c(\vz) = 1 \mid \vz \in \Comp(\bot^n)\right]\\
&= \Pr_{\vz}\left[ T^i_c(\vz) = 1 \mid \vz \in \Comp(\bot^n)\right],
\end{align*}
where the inequality has used the inductive hypothesis. On the other hand, if $\vy[i] = 1$, that implies $\vx^\dagger[i] = 1$ and thus $\alpha_i = 0$, which means the leaf across the $1$-edge from vertex $i$ is labeled with $\false$, and thus
\[
Pr_{\vz}\left[ T^i_c(\vz) = 1 \mid \vz \in \Comp(\vy)\right] = 0,
\]
which trivially satisfies the claim. For the last case, if $\vy[i] = 0$, then $\vx^\dagger[i] = 0$ and thus $\alpha_i = 1$, which means the leaf across the $1$-edge from vertex $i$ is labeled with $\true$. Therefore we have

\begin{align*}
\Pr_{\vz}\left[ T^i_c(\vz) = 1 \mid \vz \in \Comp(\vy)\right] &= \Pr_{\vz}\left[ T^{i-1}_c(\vz) = 1 \mid \vz \in \Comp(\vy)\right] \\
&\leq \Pr_{\vz}\left[ T^{i-1}_c(\vz') = 1 \mid \vz \in \Comp(\bot^n)\right]\\
& \leq \frac{1}{2} + \frac{1}{2} \Pr_{\vz}\left[ T^{i-1}_c(\vz) = 1 \mid \vz \in \Comp(\bot^n)\right]\tag{as $\Pr[\cdot] \leq 1$}\\
&= \Pr_{\vz}\left[ T^{i}_c(\vz) = 1 \mid \vz \in \Comp(\bot^n)\right].
\end{align*}
This completes the induction argument, and thus we conclude the proof of the lemma.
\end{proof}
\begin{figure}
    \centering
    \begin{tikzpicture}
        \node[draw, circle] at (0,0) (a) {$4$};
        
        \node[draw, circle] at (-1,-1) (b) {$3$};
        
        \node[draw, circle] at (-2,-2) (c) {$2$};
        
        \node[draw, circle] at (-3,-3) (d) {$1$};
        
        \node[draw, circle] at (-4,-4) (e) {$0$};
        
        \node[] at (-5, -5) (f) {$\false$};
        
        \node[] at (-3, -5) (g) {$\false$};
        
        \node[] at (-2, -4) (h) {$\false$};
        
        \node[] at (-1, -3) (i) {$\true$};
        
         \node[] at (0, -2) (j) {$\true$};
         
         \node[] at (1, -1) (k) {$\false$};
        
        \draw[->] (a) -- (b) node [near end, above] {{\small 0}};
        \draw[->] (b) -- (c) node [near end, above] {{\small 0}};
        \draw[->] (c) -- (d)  node [near end, above] {{\small 0}};
        \draw[->] (d) -- (e)  node [near end, above] {{\small 0}};
        
        \draw[->] (e) -- (f)  node [near end, above] {{\small 0}};
        \draw[->] (e) -- (g)  node [near end, above] {{\small 1}};
        \draw[->] (d) -- (h)  node [near end, above] {{\small 1}};
        \draw[->] (c) -- (i) node [near end, above] {{\small 1}};
        \draw[->] (b) -- (j) node [near end, above] {{\small 1}};
        \draw[->] (a) -- (k) node [near end, above] {{\small 1}};
    \end{tikzpicture}
    \caption{Example of $T_c$, the tree constructed in the proof of Lemma~\ref{lemma:T_delta}, for $n=5$ and $\delta = \frac{2}{5}$. In this case $c = 12 = 2^2 + 2^3$. Note that $\Pr_{\vz}[T_c(\vz) = 1 \mid \vz \in \Comp(\bot^n)] = \frac{c}{2^n} =\frac{12}{32}$, and $\left|\frac{2}{5} - \frac{12}{32}\right| = \frac{1}{40} < \frac{1}{32}$.}
    \label{fig:T_delta}
\end{figure}
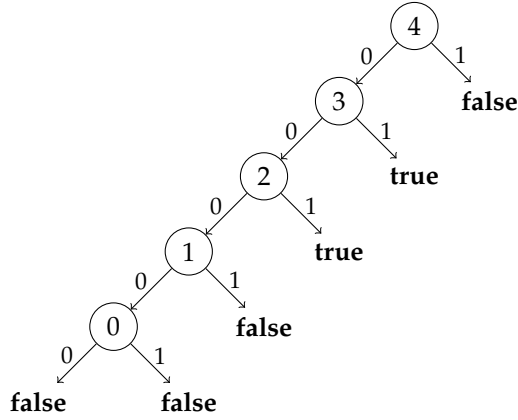

We are now ready to prove Proposition~\ref{prop:delta-minimum-hardness}. We will use notation $\log(x)$ to refer to the logarithm in base $2$ of $x$.

\begin{proof}[Proof of Theorem~\ref{prop:delta-minimum-hardness}]

We will prove that deciding whether a $\delta$-SR of size $k$ exists is $\np$-hard. We will reduce from the case $\delta=1$, proved $\np$-hard by~\citet{NEURIPS2020_b1adda14}. We assume of course that $\delta < 1$, as otherwise the result is already known.

Let $(T, \vx, k)$ be an input of the \textsf{Minimum Sufficient Reason} problem (i.e., $\delta = 1$), and let $n$ be the dimension of $T$ and $\vx$. Assume without loss of generality that $T(\vx) = 1$. If the given input of \textsf{Minimum Sufficient Reason} is positive, then there is a partial instance $\vy \subseteq \vx$ with $|\vy|_\bot \geq n-k$ such that $\Pr_{\vz}[T(\vz) = 1 \mid \vz \in \Comp(\vy)] = 1$, and otherwise for every partial instance $\vy \subseteq \vx$ with  $|\vy|_\bot \geq n-k$ it holds that 
\[
\Pr_{\vz}\left[T(\vz) = 0 \mid \vz \in \Comp(\vy)\right] \geq \frac{1}{2^n}.
\]

Let us build a tree $F_\delta$ with $3n +3 +\lceil\log(1/\delta)\rceil $ variables as follows. First build $T_\delta$ of dimension $2n + 3 + \lceil\log(1/\delta)\rceil$ by using Lemma~\ref{lemma:T_delta}, and then replace every true leaf of $T_\delta$ by a copy of $T$. Assume the $2n +3 + \lceil\log(1/\delta)\rceil $ variables of $T_\delta$ are disjoint from the $n$ variables that appear in $T$, and thus $F_\delta$ has the proposed number of variables. An example of the construction of $F_\delta$ is illustrated in Figure~\ref{fig:F}. 

Define
\[
   \delta' \coloneqq \Pr_{\vz}\left[T_\delta(\vz) = 1 \mid \vz \in \Comp\left(\bot^{2n + 2 + \lceil \log(1/\delta) \rceil }\right)\right],
\]
and recall that $|\delta' - \delta| \leq \frac{1}{2^{2n+3+\lceil \log(1/\delta)\rceil}}$.
Now, let us build a final decision tree $\tstar$ with $4n+k+4+\lceil\log(1/\delta')\rceil + \lceil\log(1/\delta)\rceil$ variables as follows. 
Create $\ell \coloneqq n+k+1+\lceil\log(1/\delta')\rceil$ vertices, labeled $r_i$ for $i \in \{1, \ldots, \ell\}$, and assume these labels are disjoint from the ones used in $F_\delta$.
 Let $r_1$ be the root of $\tstar$, and for each $i \in \{1, \ldots, \ell-1\}$, connect vertex labeled with $r_i$ to vertex labeled with $r_{i+1}$ using a $0$-edge. The $0$-edge from vertex labeled $r_{\ell}$ goes towards a leaf labeled with $\true$. 
 The $1$-edge from every vertex $r_i$ goes towards the root of a different copy of $F_\delta$. Note that this construction, illustrated in Figure~\ref{fig:tstar}, takes polynomial time. Now, consider the instance $\vx^\star$ that is defined (i) exactly as $\vx$ for the variables of $T$, (ii) exactly as in the instance $\vx^\dagger$ coming from Lemma~\ref{lemma:T_delta} for the variables of $T_\delta$ in $F_\delta$, and (iii) with all variables $r_i$ set to $0$.
Note that $\tstar(\vx^\star) = 1$. Now we prove both directions of the reduction separately. Assume fir that the instance $(T, \vx, k)$ is a positive instance for \textsf{Minimum Sufficient Reason}. Then we claim that there is a $\delta$-SR for $\tstar$ of size at most $k$. Indeed, let $\vy \subseteq \vx$ be a sufficient reason for $\vx$ under $T$ with at most $k$ defined components. Then consider the partial instance $\vy^\star \subseteq \vx^\star$, that is only defined in the components where $\vy$ is defined. Now let us study \( \Pr_{\vz}[\tstar(\vz') = 1 \mid \vz \in \Comp(\vy^\star)] \). The probability that $\vz$ ends up in the true leaf on the $0$-edge from vertex $r_{\ell}$ is $\frac{1}{2^{\ell}}$. In any other case, $\vz$ takes a path that goes into a copy of $F_\delta$, where its probability of acceptance is $\delta' \geq \delta - \frac{1}{2^{2n +3+\lceil \log(1/\delta)\rceil}}$  because of Lemma~\ref{lemma:T_delta} and using that $\vy^\star$ is undefined for all the variables of $T_c$. These two facts imply that 
\[
\Pr_{\vz}[\tstar(\vz) = 1 \mid \vz\in \Comp(\vy^\star)] \geq \frac{1}{2^{\ell}} + \delta - \frac{1}{2^{2n+3+\lceil \log(1/\delta)\rceil}}.
\]
Now consider that 
\begin{align*}
\delta &\leq \delta' +  \frac{1}{2^{2n +3+\lceil \log(1/\delta)\rceil}} \\
&\leq \delta' + \frac{1}{2^{2n+3}} \cdot \frac{1}{2^{\lceil \log(1/\delta)\rceil}}\\
&\leq \delta' + \frac{1}{2^{2n+3}} \cdot \frac{1}{2^{ \log(1/\delta)}}\\
&= \delta' + \frac{\delta}{2^{2n+3}},
\end{align*}
from where
\[
\delta \leq \delta' \left( 1- \frac{1}{2^{2n+3}}\right),
\]
and thus 
\begin{align*}
    \log(1/\delta) &\geq \log(1/\delta') + \log \left(1-\frac{1}{2^{2n+3}}
    \right)\\
    &= \log(1/\delta') - \log \left(\frac{2^{2n+3}}{2^{2n+3} - 1}
    \right)\\
    &= \log(1/\delta') - 2n - 3 + \log(2^{2n+3}-1) \\
    &\geq\log(1/\delta') -2n-3 + 2n + 2 \tag{using $2^{2n+3}-1 \geq 2^{2n+2}$}\\
    &= \log(1/\delta') - 1.
\end{align*}

From this we obtain that
\begin{align*}
    \ell &= n+k+1+\lceil\log(1/\delta')\rceil\\
    &\leq
    2n + 1 +\lceil\log(1/\delta')\rceil \\
    &\leq 2n+1 + \log(1/\delta') + 1\\
    &\leq 2n+3 + \log(1/\delta)\\ &\leq  2n+3 + \lceil\log(1/\delta)\rceil,
\end{align*}
which allows us to conclude that
\[
\Pr_{\vz}[\tstar(\vz) = 1 \mid \vz\in \Comp(\vy^\star)] \geq \frac{1}{2^{\ell}} + \delta - \frac{1}{2^{2n+3+\lceil \log(1/\delta)\rceil}} \geq \delta.
\]
On the other hand, if  $(T, \vx, k)$ is a negative instance for \textsf{Minimum Sufficient Reason}, consider any partial $\vy^\star$ with at most $k$ defined components, and note that by hypothesis we have that $\Pr_{\vz}[T(\vz) = 1 \mid \vz \in \Comp(\vy^\star)] \leq 1 - \frac{1}{2^n}$. This implies, together with the second part of Lemma~\ref{lemma:T_delta}, that 
\[
\Pr_{\vz}[F_\delta(\vz) = 1 \mid \vz \in \Comp(\vy^\star)] \leq \delta \left(1 - \frac{1}{2^n}\right),
\]
and thus subsequently
\[
\Pr_{\vz}[\tstar(\vz) = 1 \mid \vz \in \Comp(\vy^\star)] \leq \delta \left(1 - \frac{1}{2^n}\right) + \frac{1}{2^{\ell - k}},
\]
by using that with at most $k$ defined components in $\tstar$, the probability of reaching the true leaf across the $0$-edge from $r_\ell$ is at most $\frac{1}{2^{\ell-k}}$. 
To conclude, note that 
\begin{align*}
    \Pr_{vz}[\tstar(\vz) = 1 \mid \vz \in \Comp(\vy^\star)] &\leq \delta \left(1 - \frac{1}{2^n}\right) + \frac{1}{2^{\ell - k}}\\
    &= \delta - \frac{\delta}{2^n} +
    \frac{1}{2^{n+1+\lceil \log(1/\delta')\rceil}}\\
     &\leq \delta - \frac{\delta}{2^n} +
    \frac{1}{2^{n+1+ \log(1/\delta')}}\\
    &= \delta + \frac{(\delta'-\delta)-\delta}{2^{n+1}}\\
    &\leq \delta + 
    \frac{\frac{1}{2^{2n+3+\lceil\log(1/\delta)\rceil}} - \delta}{2^{n+1}}\\
    &\leq  \delta + 
    \frac{\frac{\delta}{2^{2n+3}} - \delta}{2^{n+1}}\\
    &= \delta + \delta\left(\frac{-1 + \frac{1}{2^{2n+3}}}{2^{n+1}} \right)\\
    &< \delta.
\end{align*}

We have thus concluded that $\vy^\star$ is a $\delta$-SR for $\vx^\star$ over $\tstar$ if and only if $(T, \vx, k)$ is a positive instance of \textsf{Minimum Sufficient Reason}, which completes our reduction.
\end{proof}

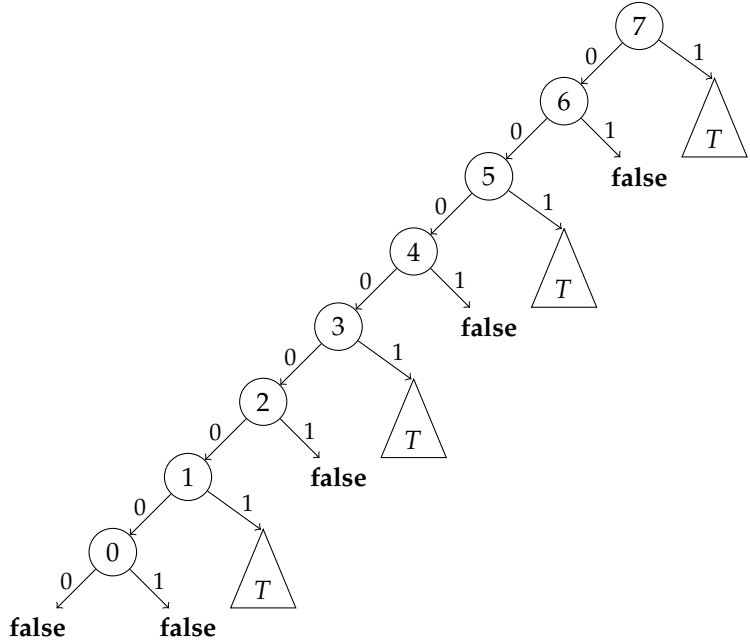
\begin{figure}
    \centering
    \begin{tikzpicture}
    	 \node[draw, circle] at (3,3) (a) {$7$};
    	 
    	  \node[draw, circle] at (2,2) (b) {$6$};
    	  
    	    	  \node[draw, circle] at (1,1) (c) {$5$};

        \node[draw, circle] at (0,0) (d) {$4$};
        
        \node[draw, circle] at (-1,-1) (e) {$3$};
        
        \node[draw, circle] at (-2,-2) (f) {$2$};
        
        \node[draw, circle] at (-3,-3) (g) {$1$};
        
        \node[draw, circle] at (-4,-4) (h) {$0$};
        
        \node[isosceles triangle, draw, shape border rotate=45] at (4, 1.5) (i) {$T$};
        
         \node[] at (3, 1) (j) {$\false$};
         
         \node[isosceles triangle, draw, shape border rotate=90] at (2, -0.5) (k) {$T$};
        
         \node[] at (1, -1) (l) {$\false$};
         
          \node[isosceles triangle, draw, shape border rotate=90] at (0, -2.5) (m) {$T$};
        
         \node[] at (-1, -3) (n) {$\false$};
         
          \node[isosceles triangle, draw, shape border rotate=90] at (-2, -4.5) (o) {$T$};
        
         \node[] at (-3, -5) (p) {$\false$};
         
          \node[] at (-5, -5) (q) {$\false$};

        \draw[->] (a) -- (b)  node [near end, above] {{\small 0}};
        
         \draw[->] (b) -- (c)  node [near end, above] {{\small 0}};
         
          \draw[->] (c) -- (d)  node [near end, above] {{\small 0}};
           \draw[->] (d) -- (e)  node [near end, above] {{\small 0}};
            \draw[->] (e) -- (f)  node [near end, above] {{\small 0}};
             \draw[->] (f) -- (g)  node [near end, above] {{\small 0}};
              \draw[->] (g) -- (h)  node [near end, above] {{\small 0}};
               \draw[->] (h) -- (q)  node [near end, above] {{\small 0}};
               
                \draw[->] (a) -- (3.97,2.3)  node [near end, above] {{\small 1}};
                	
                	\draw[->] (b) -- (j)  node [near end, above] {{\small 1}};
                	
                	\draw[->] (c) -- (1.97, 0.3)  node [near end, above] {{\small 1}};
                	
                	\draw[->] (d) -- (l)  node [near end, above] {{\small 1}};
                	
                	\draw[->] (e) -- (-0.03,-1.7)  node [near end, above] {{\small 1}};
                	
                	\draw[->] (f) -- (n)  node [near end, above] {{\small 1}};
                	
                	\draw[->] (g) -- (-2.03, -3.7)  node [near end, above] {{\small 1}};
                	
                	\draw[->] (h) -- (p)  node [near end, above] {{\small 1}};

    \end{tikzpicture}
    \caption{Illustration of the construction of $F_\delta$ for $\delta = \frac{2}{3}$ and $n=2$. Thus $2n + 3 + \lceil\log(1/\delta)\rceil=8$ and $c = 
   \left\lfloor \frac{2}{3} \cdot 2^8 \right\rfloor = 170 = 2^7 + 2^5 + 2^3 + 2^1$.}
    \label{fig:F}
\end{figure}

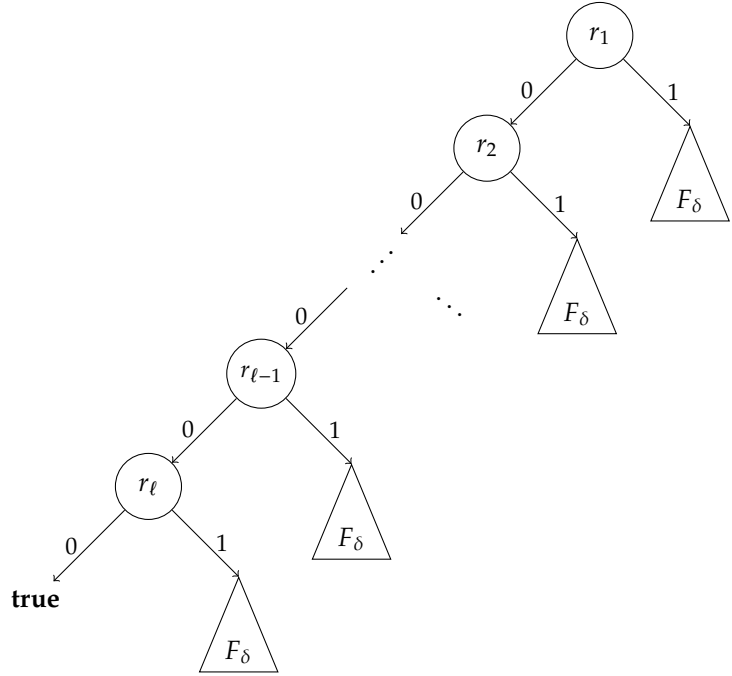
\begin{figure}
    \centering
    \begin{tikzpicture}
        \node[draw, circle, minimum size=2.5em] at (0,0) (a) {$r_1$};
        
        \node[draw, circle, minimum size=2.5em] at (-1.5,-1.5) (b) {$r_2$};
        
        \node[] at (-3,-3) (c) {\rotatebox{84}{$\ddots$}};
        
        \node[draw, circle,  minimum size=2.5em] at (-4.5,-4.5) (d) {\small$r_{\ell-1}$};
        
        \node[draw, circle,  minimum size=2.5em] at (-6,-6) (e) {$r_{\ell}$};
        
        \node[] at (-7.5, -7.5) (f) {$\true$};
        
        \node[isosceles triangle, draw, shape border rotate=90]  at (-4.8, -8.2) (g) {$F_\delta$};
        
        \node[isosceles triangle, draw, shape border rotate=90] at (-3.3, -6.7) (h) {$F_\delta$};
        
        \node[]  at (-2, -3.5) (i) {$\ddots$};
        
         \node[isosceles triangle, draw, shape border rotate=90] at (-0.3, -3.7) (j) {$F_\delta$};
         
         \node[isosceles triangle, draw, shape border rotate=90]  at (1.2, -2.2) (k) {$F_\delta$};
        
        \draw[->] (a) -- (b) node [near end, above] {{\small 0}};
        
         \draw[->] (b) -- (c) node [near end, above] {{\small 0}};
         
          \draw[->] (c) -- (d) node [near end, above] {{\small 0}};
          
          \draw[->] (d) -- (e) node [near end, above] {{\small 0}};
          
         \draw[->] (e) -- (f) node [near end, above] {{\small 0}};
         
         \draw[->] (a) -- (1.2, -1.2) node [near end, above] {{\small 1}};
         
            \draw[->] (b) -- (-0.3, -2.7) node [near end, above] {{\small 1}};
            
                \draw[->] (d) -- (-3.3, -5.7) node [near end, above] {{\small 1}};
                	
                	 \draw[->] (e) -- (-4.8, -7.2) node [near end, above] {{\small 1}};

    \end{tikzpicture}
    \caption{Illustration of the construction of $\tstar$. Recall that $\ell = n + k + 1 + \lceil \log(1/\delta')\rceil$.}
    \label{fig:tstar}
\end{figure}

\section{Proof of Theorem \ref{thm:delta-minimal-hardness}}
\label{app:thm:delta-minimal-hardness}

\begin{thmbis}{thm:delta-minimal-hardness}
Assuming that $\ptime \neq \np$, there is no polynomial-time algorithm for {\em $\minimal$}.
\end{thmbis}

We start by explaining the high-level idea of the proof.  First, we will show that the following decision problem, called 
$\checksub$, is \np-hard.

\begin{center}
\fbox{\begin{tabular}{ll}
\small{PROBLEM} : & $\checksub$
\\{\small INPUT} : & $(T, \vy)$, for a decision tree $T$ of dimension $n$, and partial instance $\vy \in \{0,1,\bot\}^n$ \\
{\small OUTPUT} : & {\sf Yes}, if there is a partial instance $\vy' \subsetneq \vy$ such that \\
& $\Pr{}_{\!{\vz}}[T(\vz) = 1 \mid \vz \in \Comp(\vy')] 
\geq \Pr{}_{\!{\vz}}[T(\vz) = 1 \mid \vz \in \Comp(\vy)]$, \\ 
& and {\sf No} otherwise
\end{tabular}}
\end{center}

We then show that if \minimal\ admits 
a polynomial time algorithm, then $\checksub$ is in \ptime, which contradicts the assumption that $\ptime \neq \np$. 
The latter reduction %is very delicate and
requires an involved construction exploiting certain properties of the hard instances for $\checksub$. 
%requires a deep understanding of how can one construct in polynomial time 
%decision trees satisfying certain properties required by the proof. 

To show that $\checksub$ is \np-hard, we use a polynomial time reduction from a decision problem over formulas in CNF, called {\sf Minimal-Expected-Clauses}, 
and which we also show 
to be \np-hard. Both the \np-hardness of {\sf Minimal-Expected-Clauses} and 
the reduction from {\sf Minimal-Expected-Clauses} to $\checksub$ may be of independent interest. 

We now define the problem {\sf Minimal-Expected-Clauses}. 
Let $\varphi$ be a CNF formula over variables $X = \{x_1,\dots, x_n\}$. 
Assignments and partial assignments of the variables in $X$,  as well as
the notions of subsumption and completions over them, are defined in exactly the same way as for partial instances over features.   
%
%
%A \emph{partial assignment} is a mapping $\mu: \{x_1,\dots, x_n\}\to \{0,1,\bot\}$, where the value $\bot$ indicates that a variable is undefined. 
%For partial assignments $\mu$ and $\mu'$, we write $\mu\subseteq \mu'$ if $\mu(x)\neq\bot$ implies $\mu(x) = \mu'(x)$, for all variables $x$ of $\varphi$. We write $\mu \subsetneq \mu'$ if $\mu \subseteq \mu'$ and $\mu\neq \mu'$. 
%An \emph{assignment} $\sigma$ is simply a partial assignment where $\sigma(x)\in\{0,1\}$, for 
%all variables $x$; that is, no variable is undefined. 
%If $\mu\subseteq \sigma$ for a partial assignment $\mu$ and an assignment $\sigma$, we say that $\sigma$ is a \emph{completion} of $\mu$. We denote by $\textsc{Comp}(\mu)$ the set of all completions of the partial assignment $\mu$. 
For a partial assignment $\mu$ over $X$, 
we denote by $E(\varphi, \mu)$ the expected number of clauses of $\varphi$ satisfied by a random completion of $\mu$. 
We then consider the following problem for fixed $k\geq 2$ (recall that a $k$-CNF formula is a CNF formula where each clause has at most $k$ literals):
%Note that by linearity of expectation, we can write $E(\varphi, \mu)$ as the sum:  
%$$E(\varphi, \mu) = \sum_{\text{$C$ clause of $\varphi$}} \Pr(\text{$C$ is satisfied for a random completion of $\mu$}).$$
%In turn, it can be shown that:
%\begin{itemize}
%\item $\Pr(\text{$C$ is satisfied for a random completion of $\mu$}) = 1$, if there is a positive literal $x$ in $C$ with $\mu(x)=1$, or there is a negative literal $\neg x$ in $C$ with $\mu(x)=0$.
%\item Else, $\Pr(\text{$C$ is satisfied for a random completion of $\mu$}) = 1-\frac{1}{2^\eta}$, where $\eta$ is the number of literals in $C$ of the form $x$ or $\neg x$ with $\mu(x) = \bot$. \\
%\end{itemize}
%Note that for an assignment $\sigma$, $E(\varphi, \sigma)$ is simply the number of clauses of $\varphi$ satisfied by $\sigma$. 
%A $k$-CNF formula is a CNF formula where each clause has at most $k$ literals. 
%We then consider the following problem for fixed $k\geq 2$ (recall that a $k$-CNF formula is a CNF formula where each clause has at most $k$ literals):
\begin{center}
\fbox{\begin{tabular}{ll}
\small{PROBLEM} : & {\sf $k$-Minimal-Expected-Clauses}
\\{\small INPUT} : & $(\varphi, \sigma)$, for $\varphi$ a $k$-CNF formula and $\sigma$ a partial assignment
\\ 
{\small OUTPUT} : & {\sf Yes}, if there is a partial assignment $\mu\subsetneq \sigma$  
such that $E(\varphi, \mu)\geq E(\varphi, \sigma)$ 
\\ & and {\sf No} otherwise
\end{tabular}}
\end{center}

We show that {\sf $k$-Minimal-Expected-Clauses} is \np-hard even for $k=2$, via a reduction from the well-known clique problem. 
Finally,  the reduction from {\sf $k$-Minimal-Expected-Clauses}  to $\checksub$ builds an instance $(T,\vy)$ from $(\varphi, \sigma)$ in a way that there is a direct correspondence
between partial assignments $\mu\subseteq \sigma$ and partial instances $\vy'\subseteq \vy$, satisfying that 
\[
\Pr{}_{\!{\vz}}[T(\vz) = 1 \mid \vz \in \Comp(\vy')] = \frac{E(\varphi, \mu)}{m},
\]
where $m$ is the number of clauses  of $\varphi$. This implies that $(\varphi, \sigma)$ is a {\sf Yes}-instance for {\sf $k$-Minimal-Expected-Clauses} if and only if 
$(T,\vy)$ is a {\sf Yes}-instance for $\checksub$. 

Below in Section~\ref{subsec:hardness-decision}, we show the \np-hardness of {\sf $k$-Minimal-Expected-Clauses} and the reduction from {\sf $k$-Minimal-Expected-Clauses} to $\checksub$, obtaining the \np-hardness of $\checksub$. We conclude 
in Section~\ref{subsec:decision-to-computation} with the reduction from $\checksub$ to $\minimal$, obtaining Theorem~\ref{thm:delta-minimal-hardness}. 

\subsection{Hardness of the decision problem}
\label{subsec:hardness-decision}

We start with some simple observations regarding the number $E(\varphi, \mu)$ for a CNF formula $\varphi$ and a partial assignment $\mu$. 
By linearity of expectation, we can write $E(\varphi, \mu)$ as the sum
\begin{equation}
E(\varphi, \mu) = \sum_{\text{$C$ clause of $\varphi$}} \prob_{C,\mu}, 
\label{eq:expectation}
\end{equation}
where $\prob_{C,\mu}$ is the probability that a random completion of $\mu$ satisfies the clause $C$. 

In turn, the probabilities $\prob_{C,\mu}$ can be easily computed as:
\begin{itemize}
\item $\prob_{C,\mu}=1$, if there is a positive literal $x$ in $C$ with $\mu(x)=1$; or there is a negative literal $\neg x$ in $C$ with $\mu(x)=0$.
\item $\prob_{C,\mu}=1-\frac{1}{2^\eta}$, where $\eta$ is the number of literals in $C$ of the form $x$ or $\neg x$ with $\mu(x) = \bot$. \\
\end{itemize}

Finally, note that for an assignment $\sigma$, $E(\varphi, \sigma)$ is simply the number of clauses of $\varphi$ satisfied by $\sigma$. 

Now we are ready to show our first hardness result:

\begin{proposition}
\label{prop:auxiliar-hard}
{\sf $k$-Minimal-Expected-Clauses} is NP-hard even for $k=2$. 
\end{proposition}

\begin{proof}
We reduce from the clique problem. Recall this problem asks, given a graph $G$ and an integer $k\geq 3$, whether there 
is a \emph{clique} of size $k$, that is, a set $K$ of $k$ vertices such that there is an edge between any pair of distinct vertices from $K$. Let $G$ be a graph and $k\geq 3$. We can assume without loss of generality that $k$ is odd and the degree of every vertex $x$ of $G$, denoted by $\deg(x)$, is at least $k-1$; if $k$ is even we can consider the equivalent instance given by the graph $G'$ that extends $G$ with a fresh node connected via an edge with all the other nodes and $k'=k+1$. 
On the other hand, we can iteratively remove vertices of degree less than $k-1$ as those cannot be part of any clique of size $k$.
We define an instance $(\varphi, \sigma)$ for {\sf $2$-Minimal-Expected-Clauses} as follows. The variables of $\varphi$ are the nodes of $G$. For each variable $x$ we have the following clauses in $\varphi$:
\begin{itemize}
\item A clause $A_x=(\neg x)$. This clause $A_x$ is repeated $\frac{k-1}{2} + \deg(x) - (k-1)$ times in $\varphi$. Note this quantity is always a positive integer. 
\item A set of clauses $\mathcal{B}_x = \{(x\lor \neg y_1), \dots, (x\lor \neg y_{\deg(x)})\}$, where $y_1,\dots, y_{\deg(x)}$ are the neighbors of $x$ in $G$. Each clause in $\mathcal{B}_x$ appears only once in $\varphi$. 
\end{itemize}
Additionally, for each set $\{x,y\}$ where $x\neq y$ and $\{x,y\}$ is \emph{not} an edge in $G$, we have a clause $Z_{x,y}=(x\lor y)$ repeated $4e$ times in $\varphi$, where $e$ is the number of edges in $G$. 

We define the assignment $\sigma$ such that $\sigma(x) = 1$, for all variables $x$ of $\varphi$. 

%The idea behind our reduction is as follows. TODO...

For an arbitrary partial assignment $\mu$ to the variables of $\varphi$, with $\mu\subseteq \sigma$, we define
\[
 {\util}_{\varphi, \sigma}(\mu):= E(\varphi, \mu) - E(\varphi, \sigma).
 \]
 In particular, the instance $(\varphi,\sigma)$ is a {\sf Yes}-instance for {\sf $2$-Minimal-Expected-Clauses} 
 if and only if there is $\mu\subsetneq \sigma$ with ${\util}_{\varphi, \sigma}(\mu) \geq 0$. 
 By equation (\ref{eq:expectation}), we can write
 \[
{\util}_{\varphi, \sigma}(\mu) = \sum_{\text{$C$ clause of $\varphi$}} {\util}_{\varphi, \sigma}(\mu, C),
\]
where ${\util}_{\varphi, \sigma}(\mu, C)$ is defined as:
\[
{\util}_{\varphi, \sigma}(\mu, C) := \prob_{C, \mu} - \prob_{C, \sigma}.
\]
We have that:
\[
\prob_{C, \sigma}
=
\begin{cases}
0 & \text{if $C=A_x$ for some variable $x$}\\
1 & \text{if $C\in \mathcal{B}_x$ for some variable $x$ or $C=Z_{x,y}$ for some set $\{x,y\}$}
\end{cases}
\]
On the other hand, for the probability $\prob_{C,\mu}$ we have the following:
\begin{enumerate}
\item Assume $C=A_x=(\neg x)$ for some variable $x$. Then $\prob_{C,\mu}$ is 
\begin{enumerate}
\item $\frac{1}{2}$, if  $\mu(x)=\bot$ (and hence ${\util}_{\varphi, \sigma}(\mu, C)=\frac{1}{2}$), and 
\item  $0$ otherwise (then ${\util}_{\varphi, \sigma}(\mu, C)= 0$).
\end{enumerate}
\item Suppose $C=(x\lor \neg y)\in \mathcal{B}_x$ for some variable $x$. Then $\prob_{C,\mu}$ is 
\begin{enumerate}
\item $\frac{3}{4}$ if $\mu(x) = \bot$ and $\mu(y)=\bot$ (and hence ${\util}_{\varphi, \sigma}(\mu, C)= -\frac{1}{4}$), 
\item $\frac{1}{2}$ if $\mu(x) = \bot$ and $\mu(y)=1$ (and then ${\util}_{\varphi, \sigma}(\mu, C)= -\frac{1}{2}$), and 
\item $1$ otherwise (then ${\util}_{\varphi, \sigma}(\mu, C)= 0$).
\end{enumerate}
\item Suppose $C=Z_{x,y}=(x\lor y)$ for some set $\{x,y\}$. Then $\prob_{C,\mu}$ is 
\begin{enumerate}
\item $\frac{3}{4}$ if $\mu(x) = \bot$ and $\mu(y)=\bot$ (and hence ${\util}_{\varphi, \sigma}(\mu, C)= -\frac{1}{4}$), and 
\item $1$ otherwise (then ${\util}_{\varphi, \sigma}(\mu, C)= 0$).  
\end{enumerate}
\end{enumerate}

We now show the correctness of our construction.
Suppose $G$ has a clique $K$ of size $k\geq 3$. 
Let $\mu$ be the partial assignment that sets $\mu(x)=\bot$ if $x\in K$ and $\mu(x)=1$ if $x\notin K$. 
Note that $\mu\subsetneq \sigma$. 
We claim that ${\util}_{\varphi, \sigma}(\mu)=0$ and hence $(\varphi, \sigma)$ is a {\sf Yes}-instance.
Let $C$ be a clause in $\varphi$. If $C$ is of the form $Z_{x,y}$, then ${\util}_{\varphi, \sigma}(\mu, C) = 0$. 
Indeed, by construction, $\{x,y\}$ is not an edge, and since $K$ is a clique, then $\mu(x)=1$ or $\mu(y)=1$. 
This means we are always in case 3(b) above.
If $x\notin K$ and $C$ is of the form $A_x$ or belongs to $\mathcal{B}_x$, then ${\util}_{\varphi, \sigma}(\mu, C) = 0$, 
since $\mu(x) = 1$ and hence we fall either in case 1(b) or 2(c) above. 
It follows that ${\util}_{\varphi, \sigma}(\mu)$ is the sum of the utilities of all the clauses involved with variables $x\in K$. That is:
\begin{equation}
\label{eq:util1}
{\util}_{\varphi, \sigma}(\mu) = \sum_{x\in K} \left[\left(\frac{k-1}{2} + \deg(x) - (k-1)\right)\, {\util}_{\varphi, \sigma}(\mu, A_x) + \sum_{C\in \mathcal{B}_x} {\util}_{\varphi, \sigma}(\mu, C) \right].
\end{equation}
Take $x\in K$. Then ${\util}_{\varphi, \sigma}(\mu, A_x) = \frac{1}{2}$ as $\mu(x)=\bot$, 
and then case 1(a) applies. On the other hand, for a clause $C\in \mathcal{B}_x$ we have two cases:
\begin{itemize}
\item $C=(x\lor \neg y)$ for $y\in K$. In this case, ${\util}_{\varphi, \sigma}(\mu, C)= -\frac{1}{4}$ as we are in case 2(a) above. 
\item $C=(x\lor \neg y)$ for $y\notin K$. In this case, ${\util}_{\varphi, \sigma}(\mu, C)= -\frac{1}{2}$ as we are in case 2(b) above. 
\end{itemize}
Moreover, note that the first case occurs exactly for $k-1$ clauses in $\mathcal{B}_x$, as $x$ has precisely $k-1$ neighbors in the clique $K$. The second case occurs exactly for $\deg(x)-(k-1)\geq 0$ clauses in $\mathcal{B}_x$. 
Replacing in equation (\ref{eq:util1}), we obtain:
\begin{align*}
{\util}_{\varphi, \sigma}(\mu) & = \sum_{x\in K} \left(\frac{k-1}{4} + \frac{\deg(x)}{2} - \frac{k-1}{2}\right) + \left(-\frac{k-1}{4} -\frac{\deg(x)}{2} + \frac{k-1}{2}\right)\\
& = 0.
\end{align*}
We conclude that $(\varphi, \sigma)$ is a {\sf Yes}-instance. 

Suppose now that there is a partial assignment $\mu$, with $\mu\subsetneq \sigma$ and  ${\util}_{\varphi, \sigma}(\mu) \geq 0$. 
Let $K$ be the set of variables $x$ such that $\mu(x)=\bot$. For $x\notin K$ and $C=A_x$ or $C\in \mathcal{B}_x$, we have ${\util}_{\varphi, \sigma}(\mu, C) = 0$, as we are in cases 1(b) or 2(c) above. Then we can write:
\begin{align}
{\util}_{\varphi, \sigma}(\mu) & = \sum_{x\in K} \left[\left(\frac{k-1}{2} + \deg(x) - (k-1)\right)\, {\util}_{\varphi, \sigma}(\mu, A_x) + \sum_{C\in \mathcal{B}_x} {\util}_{\varphi, \sigma}(\mu, C) \right] \nonumber \\
& + \sum_{\text{$\{x,y\}$ non-edge}} 4e\, ({\util}_{\varphi, \sigma}(\mu, Z_{x,y})).
\label{eq:util2}
\end{align}
We claim that $|K|\geq k$. 
Towards a contradiction, suppose $|K|=\ell<k$. 
As ${\util}_{\varphi, \sigma}(\mu, Z_{x,y})\leq 0$ for every pair $\{x,y\}$, the last term in equation (\ref{eq:util2}) is $\leq 0$, and then:
\begin{equation}
{\util}_{\varphi, \sigma}(\mu) \leq \sum_{x\in K} \left[\left(\frac{k-1}{2} + \deg(x) - (k-1)\right)\, {\util}_{\varphi, \sigma}(\mu, A_x) + \sum_{C\in \mathcal{B}_x} {\util}_{\varphi, \sigma}(\mu, C) \right].
\label{eq:util3}
\end{equation}
Take $x\in K$. Following the same argument as before, we have that ${\util}_{\varphi, \sigma}(\mu, A_x) = \frac{1}{2}$ and for a clause $C\in \mathcal{B}_x$ we have the two cases:
\begin{itemize}
\item $C=(x\lor \neg y)$ for $y\in K$, and ${\util}_{\varphi, \sigma}(\mu, C)= -\frac{1}{4}$. 
\item $C=(x\lor \neg y)$ for $y\notin K$, and ${\util}_{\varphi, \sigma}(\mu, C)= -\frac{1}{2}$. 
\end{itemize}
Let say the first case occurs precisely for $r$ clauses from $\mathcal{B}_x$. Then:
\begin{equation}
\sum_{C\in \mathcal{B}_x} {\util}_{\varphi, \sigma}(\mu, C) = -\frac{r}{4} - \frac{\deg(x)-r}{2} = \frac{r}{4} - \frac{\deg(x)}{2}.
\label{eq:util4}
\end{equation}
Note that $r\leq \ell-1$ and from equation (\ref{eq:util4}) we obtain (recall $\ell < k$):
\begin{equation}
\sum_{C\in \mathcal{B}_x} {\util}_{\varphi, \sigma}(\mu, C) \leq  \frac{\ell-1}{4} - \frac{\deg(x)}{2} < \frac{k-1}{4} - \frac{\deg(x)}{2}\nonumber.
\end{equation}
Replacing in equation (\ref{eq:util3}), we obtain:
\begin{align*}
{\util}_{\varphi, \sigma}(\mu) & < \sum_{x\in K} \left[\left(\frac{k-1}{4} + \frac{\deg(x)}{2} - \frac{k-1}{2}\right) +   \frac{k-1}{4} - \frac{\deg(x)}{2}\right]\\
& = \sum_{x\in K} \left[\frac{\deg(x)}{2} - \frac{k-1}{4} +   \frac{k-1}{4} - \frac{\deg(x)}{2}\right]\\
& = 0.
\end{align*}
We conclude that ${\util}_{\varphi, \sigma}(\mu) < 0$ which is a contradiction. Hence $|K|\geq k$.

Finally, we show that $K$ is a clique. By contradiction, assume there is a pair $\{\tilde{x},\tilde{y}\}$ such that $\tilde{x}\neq \tilde{y}$, $\tilde{x},\tilde{y}\in K$ and $\{\tilde{x},\tilde{y}\}$ is not an edge in $G$. Then there is a clause $Z_{\tilde{x},\tilde{y}}$ which is repeated $M$ times in $\varphi$.  Since $\mu(\tilde{x})=\bot$ and $\mu(\tilde{y})=\bot$, we have ${\util}_{\varphi, \sigma}(\mu, Z_{\tilde{x},\tilde{y}})= -\frac{1}{4}$, as we are in case 3(a) above. As ${\util}_{\varphi, \sigma}(\mu, Z_{x,y})\leq 0$ for all pairs $\{x,y\}$, we obtain:
\begin{equation*}
\sum_{\text{$\{x,y\}$ non-edge}} 4e\, ({\util}_{\varphi, \sigma}(\mu, Z_{x,y})) \leq 4e\, ({\util}_{\varphi, \sigma}(\mu, Z_{\tilde{x},\tilde{y}})) \leq - e
\end{equation*}
For $x\in K$, since ${\util}_{\varphi, \sigma}(\mu, C)\leq 0$, for all $C\in \mathcal{B}_x$, we have $\sum_{C\in \mathcal{B}_x} {\util}_{\varphi, \sigma}(\mu, C) \leq 0$ and hence:  
\begin{equation*}
\sum_{x\in K} \sum_{C\in \mathcal{B}_x} {\util}_{\varphi, \sigma}(\mu, C)\leq 0
\end{equation*}
On the other hand, for $x\in K$, we have ${\util}_{\varphi, \sigma}(\mu, A_x)=\frac{1}{2}$. Combining all this with equation~(\ref{eq:util2}) we obtain:
\begin{align*}
{\util}_{\varphi, \sigma}(\mu) & \leq \sum_{x\in K} \left(\frac{\deg(x)}{2} - \frac{k-1}{4}\right)\, - e\\
& < \sum_{x\in K} \frac{\deg(x)}{2}\, - e \qquad\qquad \text{(since $k\geq 3$)}\\
&  \leq \sum_{\text{$x$ in $G$}} \frac{\deg(x)}{2}\, - e\\
& = 0.
\end{align*}
We conclude that 
${\util}_{\varphi, \sigma}(\mu) < 0$, and thus obtain a contradiction. Hence $G$ contains a clique of size $k$. 
\end{proof}

We now provide the reduction from {\sf $2$-Minimal-Expected-Clauses} to $\checksub$, 
showing the hardness of the latter problem. 

\begin{proposition}
\label{prop:main-hardness}
$\checksub$ is $\np$-hard. 
\end{proposition}

\begin{proof}
We reduce from {\sf $2$-Minimal-Expected-Clauses}. 
Let $(\varphi, \sigma)$ be an instance of {\sf $2$-Minimal-Expected-Clauses}. 
Let $m$ be the number of clauses of $\varphi$ and assume that $\varphi$ has $n$ variables $x_1,\dots, x_n$. 
Without loss of generality we assume that $m$ is a power of $2$. 
Define a decision tree $T$ of dimension $n+m-1$ as follows. 
Start with a \emph{perfect} binary tree $S$ of depth $\log_2 m$, that is, 
each internal node has two children, and each leaf is at depth $\log_2 m$. 
In particular, $S$ has $m$ leaves and $m-1$ internal nodes. 
All the internal nodes of $S$ are labeled with a different fresh feature from $\{n+1,\dots, n+m-1\}$. 
For each clause $C$ in $\varphi$, pick a different leaf $\ell_{C}$ of $S$. 
It is easy to see that for each clause $C$ we can define a decision tree $S_C$ over the features $\{1,\dots, n\}$ such that for every assignment $\mu:\{x_1,\dots, x_n\}\to\{0,1\}$ to the variables of $\varphi$, the corresponding instance $\vx\in\{0,1\}^n$ where $\vx[i]=\mu(x_i)$ satisfies that $S_C(\vx)=1$ if and only if $\mu$ satisfies $C$.  
The decision tree $T$ is obtained from $S$ by identifying for each clause $C$, the leaf $\ell_C$ with the root 
of the decision tree $S_C$. 

For any partial assignment $\mu:\{x_1,\dots, x_n\}\to\{0,1,\bot\}$ for $\varphi$, we denote by $\vy_{\mu}$ the partial instance of dimension ${n+m-1}$ such that 
$\vy_{\mu}[i] = \mu(x_i)$ for every $i\in\{1,\dots, n\}$ and $\vy_{\mu}[i]=\bot$ for every $i\in\{n+1,\dots, n+m-1\}$. 
The output of the reduction is $(T,\vy_\sigma)$. 
Observe that the transformation from $\mu$ to $\vy_\mu$ is a bijection between the sets $\{\mu\mid \mu\subseteq \sigma\}$ 
and $\{\vy_\mu \mid \vy_\mu \subseteq \vy_\sigma\}$. 
By construction, for any partial assignment $\mu\subseteq \sigma$, we have:
\[
\Pr{}_{\!{\vz}}[T(\vz) = 1 \mid \vz \in \Comp(\vy_\mu)] = \frac{E(\varphi, \mu)}{m}.
\]
Hence  $(\varphi, \sigma)$ is a {\sf Yes}-instance of {\sf $2$-Minimal-Expected-Clauses} if and only if $(T,\vy_{\sigma})$ is a {\sf Yes}-instance of $\checksub$.
\end{proof}

\begin{remark}
We can assume that the instance $(T,\vy_\sigma)$ constructed in the proof of Proposition~\ref{prop:main-hardness}, satisfies that
\[
\Pr{}_{\!{\vz}}[T(\vz) = 1 \mid \vz \in \Comp(\vy_\sigma)]  > \frac{1}{2}.
\] 
\label{remark:balance}
\end{remark}
Indeed, the above probability is simply $\frac{E(\varphi, \sigma)}{m}$, where $m$ is the number of clauses of $\varphi$. 
On the other hand, from the proof of Proposition~\ref{prop:auxiliar-hard}, we can choose $\sigma$ such that 
$\sigma(x_i) = 1$ for every variable $x_i \in\{x_1,\dots, x_n\}$ of $\varphi$. It follows that $E(\varphi, \sigma)$ 
is simply the number of clauses satisfied by $\sigma$, which are all the clauses in $\mathcal{B}_x$ for some variable $x$, 
and all the clauses of the form $Z_{x,y}$. Note that the total number of clauses from the sets $\mathcal{B}_x$ is greater that the 
total number clauses of the form $A_x$, and hence $\frac{E(\varphi, \sigma)}{m}>\frac{1}{2}$. 
Indeed, there are $\deg(x)$ clauses in $\mathcal{B}_x$, and summing over all the variables $x$, we obtain $2e$, where $e$ are 
the number of edges in the graph $G$. On the other hand, each clause $A_x$ is repeated $\frac{k-1}{2} + \deg(x) - (k-1)= \deg(x) - \frac{k-1}{2}$ times. Taking the sum over all the variables $x$, we obtain $2e - n \left(\frac{k-1}{2}\right) < 2e$. This property will be useful in the Section~\ref{subsec:decision-to-computation}.

\subsection{From hardness of decision to hardness of computation}
\label{subsec:decision-to-computation}

\newcommand{\dis}{\mathbb{U}}

%\todo[inline]{Marcelo: cambié el nombre {\sf Check-Minimal-Prob-SR} por {\sf Check-SUB-SR}, ya que lo que estamos haciendo es verificar si hay una instancia SUBsumed que satisface cierta propiedad.}

We will show a Turing-reduction from a variant of $\checksub$ to $\minimal$, thus establishing that the latter cannot be solved in polynomial time unless $\mathrm{P} = \mathrm{NP}$.

For the sake of readability, given a partial instance $\vy$, in this
proof we use notation $\vz \sim \dis(\vy)$ to indicate that $\vz$
is generated uniformly at random from the set $\Comp(\vy)$. For
instance, we obtain the following simplification by using this
terminology:
\[
\Pr_{\!{\vz}}[T(\vz) = 1 \mid \vz \in \Comp(\vy)] \ = \
\Pr_{\vz \sim \dis(\vy)}[T(\vz) = 1]
\]
We will require a particular kind of hard instances for the
$\checksub$ in order to make our reduction work. In particular, we now
define the notion of \emph{strongly-balanced} inputs, which
intuitively captures the idea that defined features in a partial
instance $\vy$ appear at the same depth in different branches of a the
decision tree $T$. In order to make this definition precise, consider
that every path $\pi$ from the root to a leaf in a decision tree can
be identified with a sequence of labels $s_\pi$ corresponding to the
labels of the nodes of $\pi$, where the last label of $\pi$ is either
$\true$ or $\false$. We use notation $s_\pi[i]$ for the $i$-th label
in the sequence $s_\pi$.
With this notation, we
can introduce the following definition.

\begin{definition}
  Given a decision tree $T$ of dimension $d$ and $\vy \in \{0,1,\bot\}^n$ a partial instance, we say that the pair $(T, \vy)$ is strongly-balanced if
	\[
	\Pr_{\vz \sim \dis(\vy)} \left[T(\vz) = 1 \right] > \frac{1}{2},
	\]        
  and there exists $k \in \mathbb{N}$ such that for every root-to-leaf path $\pi$ in $T$, the sequence $s_\pi$ satisfies 
	\[
		 \vy[s_\pi[i]] = \bot \iff i \leq k.
	\] 
\end{definition}
If $(T,\vy)$ is strongly-balanced, then there exists a unique value $k
\in \mathbb{N}$ that satisfies the second condition of the
definition. We denote this value by $u(T,\vy)$. In particular, if $\vy \in
\{0,1\}^n$, then $(T,\vy)$ is strongly-balanced and $u(T,\vy) = 0$.

Now let us define the following problem.

\begin{center}
\fbox{\begin{tabular}{lp{10cm}}
\small{PROBLEM} : & $\sbchecksub$
\\{\small INPUT} : & $(T,\vy)$, for $T$ a decision tree of dimension $n$ and $\vy \in \{0,1,\bot\}^n$ a partial instance, where $(T, \vy)$ is strongly-balanced.
\\ 
{\small OUTPUT} : & {\sf Yes}, if there is a partial instance $\vy'\subsetneq \vy$ 
such that $\Pr_{\vz \sim \dis(\vy')} \left[T(\vz) = 1\right] \geq \Pr_{\vz \sim \dis(\vy)} \left[T(\vz) = 1\right]$, and {\sf No} otherwise.
\end{tabular}}
\end{center}

One can now check that the proof of Proposition~\ref{prop:main-hardness} directly proves NP-hardness for this problem, and thus we can reduce from it to prove hardness for the computation variant.
Indeed, the first part of the definition of strongly-balanced follows from Remark~\ref{remark:balance}. 
The second part follows from the fact that the construction in the proof of Proposition~\ref{prop:main-hardness} 
starts with a perfect binary tree $S$.

\begin{lemma}
If there is a polynomial time algorithm for $\minimal$, then there is a polynomial time algorithm for $\sbchecksub$.
\label{lemma:computation-harder-decision}
\end{lemma}

\begin{proof}

Let us enumerate the features in $T$ as $1, \ldots, n$. Also, let $S$ be the set of features defined in $\vy$, that is, $\vy[i] \neq \bot \iff i \in S$.
We will first build a decision tree $T'$ of dimension $2n - |S|$, with the following features:

\begin{enumerate}
	\item Create a feature $i$ for $i \in S$.

	\item For every $i \in \{1, \ldots, n\} \setminus S$ create features $i$ and $i'$.
\end{enumerate}
Note that this amounts to the promised number of features.
We will build $T'$ through a recursive process $\mathcal{R}$ defined next. First, note that any decision tree can be described inductively as either a $\true$/$\false$ leaf, or a tuple $(r, L, R)$, where $r$ is the root node, $L$ is a decision tree whose root is the left child of $r$, and $R$ is a decision tree whose root is the right child of $r$.
%$L$ and $R$ are decision trees hanging from $r$.
We can now define $\mathcal{R}$ as a recursive procedure that when called with argument $\tau$ proceed as follows:
\begin{enumerate}
	\item If $\tau$ is a leaf then simply return $\tau$.
	\item If $\tau = (r, L, R)$, and node $r$ is labeled with feature $i \in S$, then simply return 
		\[ (r,\, \mathcal{R}(L),\, \mathcal{R}(R)).\]
	\item If $\tau = (r, L, R)$, and node $r$ is labeled with feature $i \in \{1, \ldots, n \} \setminus S$, then return the following decision tree:
	\[\left(r, \, (u, \mathcal{R}(L), \false), \, (v, \true, \mathcal{R}(R))\right),
	\]
	where nodes $u$ and $v$ are new nodes, both labeled with feature $i'$.
\end{enumerate}

\begin{figure}
	\begin{subfigure}{\textwidth}
		\centering
		\begin{tikzpicture}
			\node[draw, circle] at (0,0) (a) {$2$};
			
			\node[draw, circle] at (-1,-1) (b) {$3$};
			
			\node[draw, circle] at (1,-1) (c) {$3$};
			
			\node[draw, circle] at (-1.5,-2) (d) {$1$};
			\node[draw, circle] at (-0.5,-2) (e) {$4$};
			\node[draw, circle] at (0.5,-2) (f) {$4$};
			\node[draw, circle] at (1.5,-2) (g) {$1$};
			
			\node[] at (-3.0,-3) (h) {$\true$};
			\node[] at (-2.2,-3) (i) {$\false$};
			\node[] at (-1.3,-3) (j) {$\false$};
			\node[] at (-0.5,-3) (k) {$\true$};
			
			\node[] at (0.4,-3) (l) {$\true$};
			\node[] at (1.2,-3) (m) {$\false$};
			\node[] at (2.1,-3) (n) {$\false$};
			\node[] at (2.9,-3) (o) {$\true$};

			\draw[->] (a) -- (b) node [near end, above] {{\small 0}};
			\draw[->] (a) -- (c) node [near end, above] {{\small 1}};
			\draw[->] (b) -- (d) node [near end, above] {{\small 0}};
			\draw[->] (b) -- (e) node [near end, above] {{\small 1}};
			\draw[->] (c) -- (f) node [near end, above] {{\small 0}};
			\draw[->] (c) -- (g) node [near end, above] {{\small 1}};
			
			\draw[->] (d) -- (h) node [near end, above] {{\small 0}};
			\draw[->] (d) -- (i) node [near end,  right] {{\small 1}};
			
			\draw[->] (e) -- (j) node [near end, above] {{\small 0}};
			\draw[->] (e) -- (k) node [near end, above right] {{\small 1}};
			
			\draw[->] (f) -- (l) node [near end, above left] {{\small 0}};
			\draw[->] (f) -- (m) node [near end, above] {{\small 1}};
			
			\draw[->] (g) -- (n) node [near end,  left] {{\small 0}};
			\draw[->] (g) -- (o) node [near end, above] {{\small 1}};

		\end{tikzpicture}
		\caption{Original decision tree $T$.}
	\end{subfigure}
	
	\vspace{1em}
	\begin{subfigure}{\textwidth}
		\centering
		\begin{tikzpicture}
			\node[draw, circle] at (0,0) (a) {$2$};
			
			\node[draw, circle] at (-1,-1) (b) {$2'$};
			
			\node[draw, circle] at (1,-1) (c) {$2'$};

			\node[draw, circle] at (-2,-2) (d) {$3$};
			
%			\node[draw, circle] at (-3,-3) (e) {$1'$};
%			
%			\node[draw, circle] at (-1,-3) (f) {$1'$};
			
			\node[] at (-0.7,-2) (g) {$\false$};
			
			\node[] at (0.7,-2) (h) {$\true$};
			
				\node[draw, circle] at (2,-2) (i) {$3$};
			
			\node[draw, circle] at (1,-3) (j) {$3'$};
			
			\node[draw, circle] at (3,-3) (k) {$3'$};
			
			\node[draw, circle] at (-3,-3) (e) {$3'$};
			
			\node[draw, circle] at (-1,-3) (f) {$3'$};
			
%			\node[] at (-1.5,-4) (n) {$\textrm{true}$};
%			
%			\node[] at (-0.5,-4) (o) {$\textrm{false}$};
			
			\node[draw, circle] at (0.5,-4) (p) {$4$};
			
			\node[] at (1.5,-4) (q) {$\false$};
			
			\node[] at (2.5,-4) (r) {$\true$};
			
			\node[draw, circle] at (3.5,-4) (s) {$1$};
			
			\node[draw, circle] at (0.5-4,-4) (p2) {$1$};
			
			\node[] at (1.5-4,-4) (q2) {$\false$};
			
			\node[] at (2.5-4,-4) (r2) {$\true$};
			
			\node[draw, circle] at (3.5-4,-4) (s2) {$4$};
			
			\node[] at (-4, -5) (l1) {$\true$};
			
			\node[] at (-3, -5) (l2) {$\false$};
			
			\node[] at (-1.5, -5) (l3) {$\false$};
			
			\node[] at (-0.5, -5) (l4) {$\true$};
			
			\node[] at (0.5, -5) (l5) {$\false$};
			
			\node[] at (1.5, -5) (l6) {$\true$};
			
			\node[] at (3, -5) (l7) {$\false$};
			
			\node[] at (4, -5) (l8) {$\true$};

			\draw[->] (a) -- (b) node [near end,  above] {{\small 0}};
			\draw[->] (a) -- (c)  node [near end,  above] {{\small 1}};
			\draw[->] (b) -- (d)  node [near end,  above] {{\small 0}};
			\draw[->] (b) -- (g)  node [near end,  above right] {{\small 1}};
			\draw[->] (c) -- (i)  node [near end,  above] {{\small 1}};
			\draw[->] (c) -- (h)  node [near end,  above left] {{\small 0}};
			\draw[->] (d) -- (e)  node [near end,  above] {{\small 0}};
			\draw[->] (d) -- (f)  node [near end,  above] {{\small 1}};
			\draw[->] (i) -- (j)  node [near end,  above] {{\small 0}};
			\draw[->] (i) -- (k)  node [near end,  above] {{\small 1}};
%			\draw (e) -- (l);
%			\draw (e) -- (m);
%			\draw (f) -- (n);
%			\draw (f) -- (o);
			\draw[->] (j) -- (p)  node [near end,  above] {{\small 0}};
			\draw[->] (j) -- (q)  node [near end,  above] {{\small 1}};
			\draw[->] (k) -- (r)  node [near end,  above] {{\small 0}};
			\draw[->] (k) -- (s)  node [near end,  above] {{\small 1}};
			\draw[->] (e) -- (p2)  node [near end,  above] {{\small 0}};
			\draw[->] (e) -- (q2)  node [near end,  above] {{\small 1}};
			\draw[->] (f) -- (r2)  node [near end,  above] {{\small 0}};
			\draw[->] (f) -- (s2)  node [near end,  above] {{\small 1}};
			\draw[->] (p2) -- (l1)  node [near end,  above] {{\small 0}};
			\draw[->] (p2) -- (l2)  node [near end,  above] {{\small 1}};
			\draw[->] (s2) -- (l3)  node [near end,  above] {{\small 0}};
			\draw[->] (s2) -- (l4)  node [near end,  above right] {{\small 1}};
			\draw[->] (p) -- (l5)  node [near end,  above left] {{\small 0}};
			\draw[->] (p) -- (l6)  node [near end,  above] {{\small 1}};
			\draw[->] (s) -- (l7)  node [near end,  above] {{\small 0}};
			\draw[->] (s) -- (l8)  node [near end,  above] {{\small 1}};
		\end{tikzpicture}
		\caption{Resulting decision tree $T' = \mathcal{R}(T)$.}
	\end{subfigure}
\caption{Illustration of the recursive process $\mathcal{R}$ over an example where $\vy = (
	 0, \, \bot, \, \bot, \, 0
         ).$ Note that the pair $(T, \vy)$ is strongly-balanced.
}
\label{fig:illustration_doubling_tree}
\end{figure}
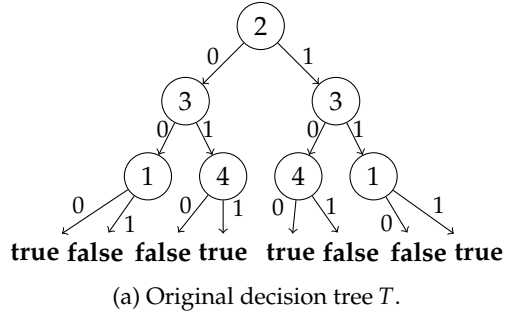
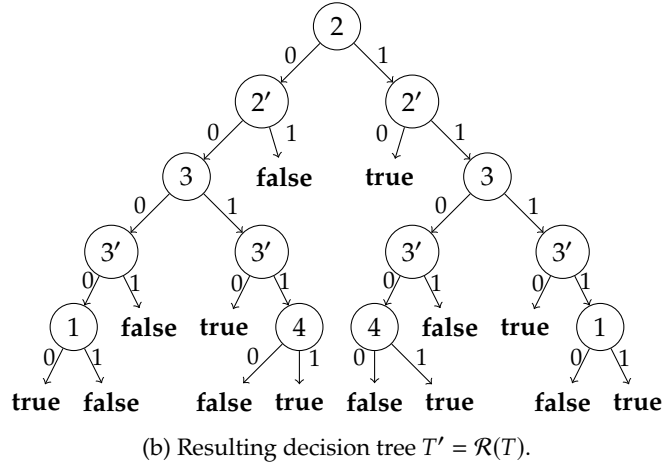

%\todo[inline]{Marcelo: el árbol de la Figura \ref{fig:illustration_doubling_tree} debería ser strongly-balanced.}
        
As anticipated, $T' = \mathcal{R}(T)$. An example illustrating the process is depicted in Figure~\ref{fig:illustration_doubling_tree}. 
Now we will create a tree $T''$ of dimension $2n-|S|+m$ that on top of the previous features incorporates features $b_j$ for each $j \in \{1, \ldots, m\}$, where $m$ is an integer we will specify later on. In order to construct $T''$, we start by defining $\vy_0$ as the partial instance of dimension $2n-|S|+m$ such that $\vy_0[i] = \vy[i]$ for every $i \in S$ and 
\[
	\vy_0\left[b_j\right] = 0, \quad  \forall j \in \{1,\ldots, m\},
\]
with the remaining components of $\vy_0$ being left undefined.
Let $T_{\vy_0}$ be a tree of dimension $2n-|S|+m$ that accepts exactly the completions of $\vy_0$; this can be trivially done by creating a tree that accepts exactly the features that are defined in $\vy_0$, and then observing that when running an instance whose feature space is a superset of this, then the instance will be accepted if and only if it is a completion of $\vy_0$. Now let $T_1$ be a tree of dimension $2n+|S| +m$ that implements the following Boolean formula:
\[
	\phi = \sum_{j=1}^m b_j \geq 2.
\]

\begin{claim}
	Decision tree $T_1$, implementing the function $\phi$, can be constructed in polynomial time.
	\label{claim:2-tribes}
\end{claim}
\begin{proof}[Proof of Claim~\ref{claim:2-tribes}]
This proof can be easily done by a direct construction. Indeed, consider the following Boolean formulas:
\begin{align*}
	f(x_1, \ldots, x_n) &\coloneqq \sum_{i=1}^n x_i \geq 2,\\
g(x_1, \ldots, x_n) &\coloneqq \sum_{i=1}^n x_i \geq 1 = \bigvee_{i=1}^n x_i.
\end{align*}

We then note that
\[
	f(x_1, \ldots, x_n) = \left[\neg x_n \land f(x_1, \ldots, x_{n-1})\right] \lor
	\left[x_n \land g(x_1, \ldots, x_{n-1})\right],
\]
and thus we can build a decision tree for $f$ recursively as illustrated in Figure~\ref{fig:illustration_t_1}.
Note that $g(x_1, \ldots, x_{k})$ can be trivially implemented by a decision tree of size $O(k)$. Thus the recursive equation characterizing the size $\alpha(n)$ of a decision tree for $f(x_1, \ldots, x_n)$, is simply
\[
	\alpha(n) = 1 + \alpha(n-1) + O(n),
\]
from where we get $\alpha(n) \in O(n^2)$, thus concluding the proof of the claim.

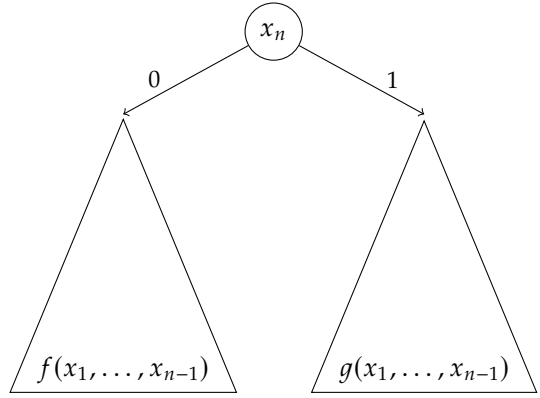
\begin{figure}
	\centering
	\begin{tikzpicture}
	\node[] (tag) at (-3, 1) {$f(x_1, \ldots, x_n)$};
		\node[draw, circle] (a) at (0,0) {$x_n$};
		\node[draw, isosceles triangle, shape border rotate=90] (b) at (-2, -4.5) {$f(x_1, \ldots, x_{n-1})$};
		
		\node[draw, isosceles triangle, shape border rotate=90] (c) at (2, -4.5) {$g(x_1, \ldots, x_{n-1})$};
		
		\draw[->] (a) -- (-2, -1.1) node [near end, above] {{\small 0}};
		
		\draw[->] (a) -- (2, -1.1) node [near end, above] {{\small 1}};
	\end{tikzpicture}
	\caption{Illustation of the construction for Claim~\ref{claim:2-tribes}.}
	\label{fig:illustration_t_1}
\end{figure}

%  
%Let us define $\mathrm{Th}^m_k$ be the $k$-threshold function on $m$ variables, meaning that it accepts is input if and only if there are at least $k$ ones in it. Now consider the natural recursive equation:
%\[
%\mathrm{Th}^m_k(w) = \left( w[1] \land \mathrm{Th}^{m-1}_{k-1}(w) \right)  \lor \left( \neg w[1] \land \mathrm{Th}^{m-1}_{k}(w) \right), 
%\]
%which can be used to implement $T_1$. They key observation is that as $k$ is simply a constant in our case,  so the recursion goes only until constant depth, and thus $T_1$ can be seen as the union of constantly many polynomial size decision trees, which can be performed in polynomial time with a standard algorithm (see e.g.,~\cite{Wegener2000}).%The proof is purely constructive, and instead of writing out the details we present an illustration of the general case in Figure~\ref{fig:illustration_t_1}.
%
%\todo[inline]{Marcelo: como conversamos con Bernardo, aquí es mejor poner la construcción explícita de $T_1$, la cual claramente es de tamaño polinomial.}
%
\end{proof}

Now, let us build an instance $\vx$ of dimension $2n-|S|+m$ as follows. For each $i \in S$ let $\vx[i] = \vy[i]$, thus ensuring $\vx$ will be a completion of $\vy$. Then for each $j \in \{1, \ldots, m\}$ let $\vx[b_j] = 0$, and finally for each $i \in \{1, \ldots, n \} \setminus S$, let $\vx[i] = 0$ and  $\vx[i'] = 1$.

For example, if $\vy = \begin{pmatrix}
 0, & \bot, & \bot, & 1
\end{pmatrix}$, and $m = 3$ then 
\[
	\vx = \begin{pmatrix}
		0, & 0, & 1, & 0, & 1, & 1, & 0, & 0, & 0
	\end{pmatrix},
\]
where the features $b_j$ have been placed at the end of the vector.

Let $\vy^\star$ be the partial instance of dimension $2n-|S|+m$ such that $\vy^\star[i] = \vy[i]$ for every $i \in S$, and undefined otherwise. Let us abuse notation and assume now that $T'$ has dimension $2n-|S|+m$, even though it only explicitly uses the first $2n-|S|$ features, as this would make it compatible with other decision trees and instance we have constructed.
Finally, let $T^\star$ be the decision tree defined as 
\[
T^\star  = T_{\vy_0} \cup (T' \cap T_1),
\]
and note that the union and intersection of decision trees can be computed in polynomial time through a standard algorithm (see e.g.,~\cite{Wegener2000}).

Let us now define 
\[
\delta \coloneqq \Pr_{\vz \sim \dis(\vy^\star)} \left[(T' \cap T_1)(\vz) = 1 \right].
\]

We now claim that the result of $\minimal$$(T^\star, \vx, \delta)$ is different from $\vy^\star$ if and only if $(T, \vy)$ is a positive instance of $\sbchecksub$. But before we can prove this, we will need some intermediary tools and claims that we develop next. 

Let us start by distinguishing two kinds of leaves of $T'$. Let us say that the  leaves of $T'$ introduced in step 1 of the recursive procedure $\mathcal{R}$ are \emph{natural}, while those introduced in step 3 are \emph{artificial}. We denote by $\mathcal{N}$ the set of natural  leaves of $T'$ and by $\mathcal{A}$ the set of artificial leaves of $T'$. Moreover, let $\mathcal{N}_t, \mathcal{N}_f$ represent the $\true$ and $\false$ natural leaves, and define $\mathcal{A}_t, \mathcal{A}_f$ analogously.  We will also use $T'_{\downarrow \vz}$ to denote the leaf where instance $\vz$ ends when evaluated over tree $T'$.
%represent that instance $w$ ends in leaf $\ell$ when evaluated over tree $T'$.
With this notation, $T'(\vz) = 1$ is equivalent to $T'_{\downarrow \vz} \in \mathcal{A}_t \cup \mathcal{N}_t$. We now make the following claims.

\begin{claim}
	For every partial instance $\vy'^\star \subseteq \vy^\star$, it holds that 
	\begin{eqnarray*}
	  \Pr_{\vz \sim \dis(\vy'^\star)} \left[T'_{\downarrow \vz} \in \mathcal{A}_t \mid T'_{\downarrow \vz} \in \mathcal{A} \right] \ = \
          \Pr_{\vz \sim \dis(\vy'^\star)} \left[T'_{\downarrow \vz} \in \mathcal{A}_f \mid T'_{\downarrow \vz} \in \mathcal{A} \right] \ = \ \frac{1}{2}.
	\end{eqnarray*}
	\label{claim:2}
\end{claim} 
\begin{proof}[Proof of Claim~\ref{claim:2}]
 Observe that every leaf $\ell \in \mathcal{A}$ has a parent node $v$ in $T'$ labeled with some feature $i'$ whose parent node $u$ in $T'$ is labeled with feature $i$. Let $G(\ell) = u$ be said the grandparent of $\ell$, and assume that $G^{-1}(u) = \{ \ell' \mid G(\ell') = u\}$. With this notation, we can split the set $\mathcal{A}$ as follows:
\[
	\mathcal{A} \ = \ \bigcup_{\text{node } u \text{ with label } i \not \in S} \left\{ T'_{\downarrow \vz} \mid G\left(T'_{\downarrow \vz}\right) = u\right\},
\]
where the union is actually disjoint. Thus, we have that for every partial instance $\vy'^\star \subseteq \vy^\star$:
\begin{align*}
  \Pr_{\vz \sim \dis(\vy'^\star)} \left[T'_{\downarrow \vz} \in \mathcal{A}_t \mid T'_{\downarrow \vz} \in \mathcal{A} \right]  =  \hspace{22em}\\
   \sum_{\text{node } u \text{ with label } i \not \in S} \; \Pr_{\vz \sim \dis(\vy'^\star)} \left[T'_{\downarrow \vz} \in \mathcal{A}_t \mid T'_{\downarrow \vz} \in \mathcal{A} \cap G^{-1}(u) \right] \cdot
  \Pr_{\vz \sim \dis(\vy'^\star)} \left[ T'_{\downarrow \vz} \in \mathcal{A} \cap G^{-1}(u) \mid T'_{\downarrow \vz}  \in \mathcal{A} \right],
\end{align*}
but we have the following equation for the last term 
\[
\Pr_{\vz \sim \dis(\vy'^\star)} \left[T'_{\downarrow \vz} \in \mathcal{A}_t \mid T\
'_{\downarrow w} \in \mathcal{A} \cap G^{-1}(u) \right] \ = \ \frac{1}{2},
\]
as the event is equivalent to $\vz[i] = 1, \vz[i'] = 0$, and this is equally likely to the complement event $\vz[i] = 0, \vz[i'] = 1$, given that $\vy'^\star[i] = \vy'^\star[i'] = \bot$. 
%given that either $w[i] = 1, w[i'] = 0$ or $w[i] = 0, w[i'] = 1$, and those are equally likely given that $z'^\star[i] = z'^\star[i'] = \bot$.
Therefore
\begin{multline*}
  \Pr_{\vz \sim \dis(\vy'^\star)}\left[T'_{\downarrow \vz} \in \mathcal{A}_t \mid T'_{\downarrow \vz} \in \mathcal{A} \right]
  \ =\\
  \sum_{\text{node } u \text{ with label } i \not \in S} \frac{1}{2} \cdot \Pr_{\vz \in \dis(\vy'^\star)} \left[ T'_{\downarrow \vz} \in \mathcal{A} \cap G^{-1}(u) \mid T'_{\downarrow \vz}  \in \mathcal{A} \right] \ =\\
  \frac{1}{2} \cdot \sum_{\text{node } u \text{ with label } i \not \in S} \Pr_{\vz \in \dis(\vy'^\star)} \left[ T'_{\downarrow \vz} \in \mathcal{A} \cap G^{-1}(u) \mid T'_{\downarrow \vz}  \in \mathcal{A} \right] \ = \ \frac{1}{2}.
\end{multline*}
\end{proof}

\begin{claim}
	Given a partial instance $\vy' \subseteq \vy$, we can naturally define $\vy'^\star$ as the partial instance of dimension $2n-|S|+m$ that matches $\vy'$ on its defined features, and holds $\vy'^\star[i] = \vy'^\star[i'] = \bot$ for every feature $i$ such that $\vy'[i] = \bot$. Then it holds that
	\[
	\Pr_{\vz \sim \dis(\vy')} \left[T(\vz) =  1 \right] \ = \ \Pr_{\vz \sim \dis(\vy'^\star)} \left[T'_{\downarrow \vz} \in \mathcal{N}_t  \mid T'_{\downarrow \vz} \in \mathcal{N} \right].
	\]
\label{claim:5}
\end{claim}
\begin{proof}[Proof of Claim~\ref{claim:5}]
%Given that the resulting leaf is natural, for every gadget $g_u$
%corresponding to a node $u$ labeled with feature $i$ such that the
%path of $w$ in $T'$ went through $g_u$, it must be the case that $w[i]
%= w[i'] = 0$ or $w[i] = w[i'] = 1$ as otherwise $T'_{\downarrow w} \in
%\mathcal{A}$.  But each of those possibilities is equally by
%definition. Thus, if gadget $g_u$ was introduced when recursively
%considering the tree $(u, L, R)$, the probability of going into $L$ or
%$R$ are the same.  By a simple induction argument we can see that this
%implies that for any leaf $\ell \in T$ with a corresponding natural
%leaf $\ell' \in T'$, we have that
Given that the resulting leaf is natural, for every node $u$ of $T'$
such that $u$ is labeled with feature $i \not\in S$, the tuple
$(u,L,R)$ was considered when constructing $T'$ and the path of $w$ in
$T'$ goes through $u$, it must be the case that $\vz[i] = \vz[i'] = 0$ or
$\vz[i] = \vz[i'] = 1$, as otherwise $T'_{\downarrow \vz} \in \mathcal{A}$.
But these two alternatives are equally likely by definition of $T'$.
Thus, by using a simple induction argument, for every leaf $\ell$ of
$T$ with a corresponding natural leaf $\ell'$ of $T'$, it holds that
\[
\Pr_{\vz \sim \dis(\vy')} \left[T_{\downarrow \vz} =  \ell \right] \ = \ \Pr_{\vz \sim \dis(\vy'^\star)}  \left[T'_{\downarrow \vz} \ = \ \ell'  \mid T'_{\downarrow \vz} \in \mathcal{N} \right],
\]
from where the claim immediately follows.
\end{proof}

\begin{claim} By choosing $m \geq \max\{2u(T,\vy)+2n,9\}$,
  %where $h(T)$ is the height of $T$,
  we have that
	\[
		\Pr_{\vz \sim \dis(\vy^\star)} \left[(T' \cap T_1)(\vz) = 1\right] \ > \ \frac{1}{2}.
	\]
	\label{claim:3}
\end{claim}
\begin{proof}[Proof of Claim~\ref{claim:3}]
First, consider that for $T_1$, we have that
\[
\Pr_{\vz \sim \dis(\vy^\star)} \left[T_1(\vz) = 1\right] \ = \ 1 - \left(\frac{1}{2}\right)^m - m\left(\frac{1}{2}\right)^m =  1 - (m+1) \left(\frac{1}{2}\right)^m,
\]
while on the other hand
\begin{eqnarray*}
	\Pr_{\vz \sim \dis(\vy^\star)} \left[T'(\vz) = 1\right] &=& \Pr_{\vz \sim \dis(\vy^\star)}\left[T'_{\downarrow \vz} \in \mathcal{A}_t\right] +
	\Pr_{\vz \sim \dis(\vy^\star)}[T'_{\downarrow \vz}  \in \mathcal{N}_t]\\
	&=& \Pr_{\vz \sim \dis(\vy^\star)}[T'_{\downarrow \vz}  \in \mathcal{A}_t \mid T'_{\downarrow \vz}  \in \mathcal{A}] \cdot \Pr_{\vz \sim \dis(\vy^\star)}[T'_{\downarrow \vz} \in \mathcal{A}]\\
	&& \hspace{30pt} +  \Pr_{\vz \sim \dis(\vy^\star)}[T'_{\downarrow \vz}  \in \mathcal{N}_t \mid T'_{\downarrow \vz}  \in \mathcal{N}] \cdot \Pr_{\vz \sim \dis(\vy^\star)}[T'_{\downarrow \vz} \in \mathcal{N}]\\
	&=& \frac{1}{2} \cdot \Pr_{\vz \sim \dis(\vy^\star)}[T'_{\downarrow \vz} \in \mathcal{A}]\\
        && \hspace{30pt} + \Pr_{\vz \sim \dis(\vy^\star)}[T'_{\downarrow \vz}  \in \mathcal{N}_t \mid T'_{\downarrow \vz}  \in \mathcal{N}] \cdot \Pr_{\vz \sim \dis(\vy^\star)}[T'_{\downarrow \vz} \in \mathcal{N}],
\end{eqnarray*}
where the last equality uses Claim~\ref{claim:2}. Let us now show that\[
\Pr_{\vz \sim \dis(\vy^\star)} [T'_{\downarrow \vz}  \in \mathcal{N}_t \mid T'_{\downarrow \vz}  \in \mathcal{N}] \ \geq \ \frac{1}{2} + \left(\frac{1}{2}\right)^{n}.
\]
By Claim~\ref{claim:5} this is the same as showing that
\[
	\Pr_{\vz \sim \dis(\vy)}[T(\vz) = 1] \ \geq \ \frac{1}{2} + \left(\frac{1}{2}\right)^{n},
\]
which is guaranteed by the definition of the $\sbchecksub$ problem, as we know that 
\[
	\Pr_{\vz \sim \dis(\vy)}[T(\vz) = 1] \ > \ \frac{1}{2},
\]
and also that $\Pr_{\vz \sim \dis(\vy)}[T(\vz) = 1]$ must be of the form $\left(\frac{k}{2^n}\right)$ with $k \in \mathbb{N}$, given that $n$ is the dimension of $T$.

Now, consider that 
\[
	\Pr_{\vz \sim \dis(\vy^\star)} [T'_{\downarrow \vz} \in \mathcal{A}] \ = \ 1-\left(\frac{1}{2}\right)^{u(T,\vy)}.
\]
%where $h(T)$ is the height of the tree $T$.
Notice that this holds because $T$ is strongly-balanced, so falling
into a natural leaf in $T'$ requires going through $u(T,\vy)$ layers
without choosing an artificial leaf, which happens with probability
$\frac{1}{2}$ at each layer. Thus, we have that
\begin{eqnarray*}
\Pr_{\vz \sim \dis(\vy^\star)}[T'(\vz) = 1]  &\geq &\frac{1}{2}\left(1-\left(\frac{1}{2}\right)^{u(T,\vy)}\right) +
\left(\frac{1}{2} + \left(\frac{1}{2}\right)^{n}\right) \left(\frac{1}{2}\right)^{u(T,\vy)}\\
&=& \frac{1}{2} + \left(\frac{1}{2}\right)^{n + u(T,\vy)}.
\end{eqnarray*}
Moreover, given that $T'$ and $T_1$ impose restrictions over disjoint sets of features, we have that
\begin{eqnarray*}
  \Pr_{\vz \sim \dis(\vy^\star)} \left[T'(\vz) = 1 \mid T_1(\vz) = 1\right]  &=& \Pr_{\vz \sim \dis(\vy^\star)}\left[T'(\vz) = 1\right].
\end{eqnarray*}
Putting together all the previous results, we obtain that 
\begin{align*}
  \Pr_{\vz \sim \dis(\vy^\star)} \left[(T' \cap T_1)(\vz) = 1\right] &= \Pr_{\vz \sim \dis(\vy^\star)}  \left[T'(\vz) = 1 \mid T_1(\vz) = 1\right] \cdot \Pr_{\vz \sim \dis(\vy^\star)}\left[T_1(\vz) = 1\right]\\
  &= \Pr_{\vz \sim \dis(\vy^\star)}\left[T'(\vz) = 1\right] \cdot \Pr_{\vz \sim \dis(\vy^\star)}\left[T_1(\vz) = 1\right]\\
	&\geq \left(\frac{1}{2} + \left(\frac{1}{2}\right)^{n + u(T,\vy)}\right) \left(1 - (m+1) \left(\frac{1}{2}\right)^m \right)\\
	&= \frac{1}{2} - (m+1)\left(\frac{1}{2}\right)^{m+1}  + \left(\frac{1}{2}\right)^{n+u(T,\vy)} - (m+1)\left(\frac{1}{2}\right)^{m+n+u(T,\vy)}\\
	&\geq \frac{1}{2} - (m+1)\left(\frac{1}{2}\right)^{m} + \left(\frac{1}{2}\right)^{n+u(T,\vy)}\\
	&\geq \frac{1}{2} - \left(\frac{1}{2}\right)^{m - \lceil\log (m+1)\rceil}+ \left(\frac{1}{2}\right)^{n+u(T,\vy)}.
\end{align*}
But we are assuming $m \geq \max\{2n + 2u(T,\vy),9\}$, which implies that $m \geq 2n + 2u(T,\vy)$. Hence, we have that
\[
m - \lceil\log (m+1)\rceil > n +u(T,\vy),
\]
as $m - \lceil\log (m+1)\rceil > \frac{m}{2}$ since $m \geq 9$.
%, which is true for $m \geq 9$, and this can be assumed without loss of generality as otherwise $u(T,z) \leq 4$ and, thus, the original instance of the decision problem will be of constant size.
We conclude that
\begin{eqnarray*}
  \frac{1}{2} - \left(\frac{1}{2}\right)^{m - \lceil\log (m+1)\rceil}+ \left(\frac{1}{2}\right)^{n+u(T,\vy)} &>& \frac{1}{2},
\end{eqnarray*}
from which the claim follows. 
\end{proof}

\begin{claim}
	For every partial instance $\vy'^\star \subseteq \vy^\star$, it holds that 
	\begin{eqnarray*}
		\Pr_{\vz \sim \dis(\vy'^\star)} \left[T'_{\downarrow \vz} \in \mathcal{A}\right] & = & \Pr_{\vz \sim \dis(\vy^\star)} \left[T'_{\downarrow \vz} \in \mathcal{A} \right],\\
		\Pr_{\vz \sim \dis(\vy'^\star)} \left[T'_{\downarrow \vz} \in \mathcal{N}\right] & = & \Pr_{\vz \sim \dis(\vy^\star)} \left[T'_{\downarrow \vz} \in \mathcal{N}\right].
	\end{eqnarray*}
	\label{claim:4}
\end{claim}
\begin{proof}[Proof of Claim~\ref{claim:4}]
We only need to prove that $\Pr_{\vz \sim \dis(\vy'^\star)}[T'_{\downarrow
    \vz} \in \mathcal{A}] = \Pr_{\vz \sim \dis(\vy^\star)} [T'_{\downarrow
    \vz} \in \mathcal{A}]$. As shown in the proof of Claim \ref{claim:3}, this
follows from the strongly-balanced property of $T$.
%\todo[inline]{Marcelo: no pondría más aquí, el razonamiento es el mismo que en la demostración del Claim \ref{claim:3}.}
%One can think of
%the features of $w$ being generated lazily as $w$ goes through $T'$,
%so the first feature in $w$ to be defined is the label of the root of
%$T$, and then depending on its actual value the following features
%will get defined. Note that for each node $u$ labeled with some
%feature $i \not \in S$ in $T$, we have introduced a \emph{gadget}
%$g_u$ in $T'$ for which if the instance $w$ gets to $u$, it will have
%a probability of $\frac{1}{2}$ of ending in some leaf $\ell \in
%\mathcal{A} \cap G^{-1}(u)$. But because $T$ is strongly-balanced, all
%root-to-leaf paths in $T$ go through the same number $p$ of nodes $u$
%labeled with a feature $i \not \in S$ before any node labeled with
%some feature $i \in S$, and thus random completions of any partial
%instance $z'^\star \subseteq z^\star$ will go through the same random
%process in terms of entering gadgets $g_u$ and ending in an artificial
%leaf from them with probability $\frac{1}{2}$.
\end{proof}

With these claims we can finally prove the reduction is correct. That
is, we will show that $\minimal(T^\star, \vx, \delta)$ is different
from $\vy^\star$ if and only if $(T, \vy)$ is a positive instance of
$\sbchecksub$.

\paragraph{Forward direction.} Assume the result of $\minimal(T^\star, \vx, \delta)$ is some partial instance $\vy'^\star$ different from $\vy^\star$. Note immediately that it is not possible that $\vy^\star \subsetneq \vy'^\star$ as by definition of $\delta$, we have that
\[
	\Pr_{\vz \sim \dis(\vy^\star)} \left[ T^\star(\vz) = 1\right]
        \ \geq \ \delta,
\]
which would contradict the minimality of $\vy'^\star$. Let us first prove that $\vy'^\star \subseteq \vy^\star$. As a first step, we show that $\vy'^\star[i] = \vy'^\star[i'] = \bot$ for every $i \not \in S$. We do this by way of contradiction, assuming first that $\vy'^\star[i] = 0$ or $\vy'^\star[i'] = 1$ for some $i \not \in S$, and exposing how either case generates a contradiction.\footnote{Recall that $\vx[i] = 0$ and $\vx[i'] = 1$, so if $\vy'^\star[i] \neq \bot$ or $\vy'^\star[i'] \neq \bot$, then $\vy'^\star[i] = 0$ or $\vy'^\star[i'] = 1$ as $\vy'^\star \subseteq \vx$.}  
 
	\begin{enumerate}
		\item If there is some feature $i \not \in S$ such that $\vy'^\star[i] = 0$, let us define $\vy^\dagger$ to be equal to $\vy'^\star$ except that $\vy^\dagger[i] = \bot$. This means that $\vy^\dagger \subsetneq \vy'^\star$. Moreover, let $\vy^\dagger_1$ be equal to $\vy'^\star$ except that $\vy^\dagger_1[i] = 1$.  We will now show that $\vy^\dagger$ is also a valid output of the computation problem, which will contradict the minimality of $\vy'^\star$. Indeed, given that $(T_{\vy_0} \cap T_1)(\vz) = 0$ for every instance $\vz$, it holds that
		\begin{align*}
		\Pr_{\vz \sim \dis(\vy^\dagger)}\left[ T^\star(\vz) = 1\right] &= \Pr_{\vz \sim \dis(\vy^\dagger)}\left[ T_{\vy_0}(\vz) = 1 \right]  	
			+
			\Pr_{\vz \sim \dis(\vy^\dagger)}\left[ (T' \cap T_1)(\vz) = 1 \right]  \\
			&= \Pr_{\vz \sim \dis(\vy'^\star)}\left[ T_{\vy_0}(\vz) = 1 \right]
			+  
			\Pr_{\vz \sim \dis(\vy^\dagger)}\left[ (T' \cap T_1)(\vz) = 1 \right],
				\end{align*}
				and we can then observe that the events $T'(\vz) = 1$ and $T_1(\vz) = 1$ are independent, thus implying that
				\[
			\Pr_{\vz \sim \dis(\vy^\dagger)}\left[ (T' \cap T_1)(\vz) = 1 \right] \ = \ \Pr_{\vz \sim \dis(\vy^\dagger)}\left[ T'(\vz) = 1 \right]  \cdot \Pr_{\vz \sim \dis(\vy^\dagger)}\left[T_1(\vz) = 1 \right].
				\]
				By construction of $\vy^\dagger$, we also have that
				\[
			\Pr_{\vz \sim \dis(\vy^\dagger)}\left[T_1(\vz) = 1 \right] = 	\Pr_{\vz \sim \dis(\vy'^\star)}\left[T_1(\vz) = 1 \right].				\]
		Thus, it now suffices to show that 
		\[
			\Pr_{\vz \sim \dis(\vy^\dagger_1)} \left[ T'(\vz) = 1 \right] \geq 	\Pr_{\vz \sim \dis(\vy'^\star)}\left[ T'(\vz) = 1 \right], 
		        \]
                        as this implies by definition of $\vy'^\star$, $\vy^\dagger$ and $\vy^\dagger_1$ that 
                        \[
                        \Pr_{\vz \sim \dis(\vy^\dagger)} \left[ T'(\vz) = 1 \right] \geq    \Pr_{\vz \sim \dis(\vy'^\star)}\left[ \
T'(\vz) = 1 \right],
                        \]
                        which in turn implies by the previous discussion that
		\[
		\Pr_{\vz \sim \dis(\vy^\dagger)} \left[ T^\star(\vz) = 1 \right] \ \geq \ \Pr_{\vz \sim \dis(\vy'^\star)} \left[ T^\star(\vz) = 1 \right] \ \geq \ \delta,
		\]
                and leads to a contradiction.
                
		In order to prove that $\Pr_{\vz \sim \dis(\vy^\dagger_1)} [ T'(\vz) = 1] \geq \Pr_{\vz \sim \dis(\vy'^\star)}[T'(\vz) = 1]$, we will consider two cases, either $\vy'^\star[i'] = 1$ or $\vy'^\star[i'] = \bot$.
		\begin{enumerate}
			\item If $\vy'^\star[i'] = 1$, then we have that
			\[
			\Pr_{\vz \sim \dis(\vy^\dagger_1)} \left[ T'(\vz) = 1 \right] \geq \Pr_{\vz \sim \dis(\vy'^\star)} \left[ T'(\vz) = 1 \right],			\]
			as for every node $u$ in $T'$ labeled with $i$, any completion $\vz$ of $\vy'^\star$ that goes through $u$ will be rejected by construction (landing in an artificial $\false$ leaf), and for paths of $T'$ that do not go through a node labeled $u$ there is no difference between completions of $\vy'^\star$ and those for $\vy^\dagger_1$.
% But as $z'^\star$ is an output of the computation problem, then we know that 
%			\[
%				\Pr_w\left[ T'(w) = 1 \mid w \in  \Comp(z'^\star) \right] \geq \delta,
%			\]
%			from where the result follows.
		      \item  If $\vy'^\star[i'] = \bot$, then we have that for every node $u$ in $T'$ which corresponds to $(u, L, R)$ according to the recursive definition of a decision tree, and is labeled with $i$, the probability of acceptance of a random completion $\vz$ conditioned on its path going through $u$, which is denoted by
                        %event which we abbreviate as
                        $\vz \rightsquigarrow u$, is as follows:
			\[
			\Pr_{\vz \sim \dis(\vy'^\star)}\left[ T'(\vz) = 1 \mid \vz \rightsquigarrow u \right] = \frac{1}{2} \cdot	\Pr_{\vz \sim \dis(\vy'^\star)} \left[ L(\vz) = 1 \mid \vz \rightsquigarrow u \right],
			\] 
			while on the other hand we have
			\begin{multline*}
			  \Pr_{\vz \sim \dis(\vy^\dagger_1)} \left[ T'(\vz) = 1 \mid \vz \rightsquigarrow u \right] \ = \ \frac{1}{2} + \frac{1}{2} \cdot \Pr_{\vz \sim \dis(\vy^\dagger_1)} \left[ R(\vz) = 1 \mid \vz \rightsquigarrow u \right].
			\end{multline*}
			By considering that
			\[
			\frac{1}{2} \ \geq \  	\frac{1}{2} \cdot \Pr_{\vz \sim \dis(\vy'^\star)} \left[ L(\vz) = 1 \mid \vz \rightsquigarrow u \right],
			\]
			we conclude that
			\[
				\Pr_{\vz \sim\dis(\vy^\dagger_1)} \left[ T'(\vz) = 1 \mid \vz \rightsquigarrow u \right] \ \geq \ \frac{1}{2} \cdot \Pr_{\vz \sim \dis(\vy'^\star)} \left[ L(\vz) = 1 \mid \vz \rightsquigarrow u \right]. 
			\]
%			and thus by applying the same argument inductively\footnote{The base case is when $L$ is a leaf, in which case the result is trivial.} we obtain that
%			\[
%			\Pr_w\left[ T'(w) = 1 \mid w \in  \Comp(z^\dagger), w \rightsquigarrow u \right] \geq 
%			\Pr_w\left[ T'(w) = 1 \mid w \in  \Comp(z'^\star), w \rightsquigarrow u \right], 
%			\]
%			which in turn implies that 
                        from which we conclude that
			\[
				\Pr_{\vz \sim\dis(\vy^\dagger_1)} \left[ T'(\vz) = 1 \mid \vz \rightsquigarrow u \right] \ \geq \ \Pr_{\vz \sim \dis(\vy'^\star)} \left[ T'(\vz) = 1 \mid \vz \rightsquigarrow u \right]. 
			        \]
                                Therefore, by considering that $\vz \rightsquigarrow u_1$ and $\vz \rightsquigarrow u_2$ are disjoint events of $u_1$, $u_2$ are distinct nodes of $T'$ with labeled $i$, we have that 
			\begin{align*}
			  \Pr_{\vz \sim \dis(\vy^\dagger_1)} \left[ T'(\vz) = 1 \right] &=
                          \sum_{\substack{u \text{ is a node of } T' \\ \text{ with label } i}} \Pr_{\vz \sim \dis(\vy^\dagger_1)} \left[ T'(\vz) = 1 \mid \vz \rightsquigarrow u \right] \cdot \Pr_{\vz \sim \dis(\vy^\dagger_1)} \left[ \vz \rightsquigarrow u \right]\\
                          &\geq \sum_{\substack{u \text{ is a node of } T' \\ \text{ with label } i}} \Pr_{\vz \sim \dis(\vy'^\star)} \left[ T'(\vz) = 1 \mid \vz \rightsquigarrow u \right] \cdot \Pr_{\vz \sim \dis(\vy^\dagger_1)} \left[ \vz \rightsquigarrow u \right]\\
                          &= \sum_{\substack{u \text{ is a node of } T' \\ \text{ with label } i}} \Pr_{\vz \sim \dis(\vy'^\star)} \left[ T'(\vz) = 1 \mid \vz \rightsquigarrow u \right] \cdot \Pr_{\vz \sim \dis(\vy'^\star)} \left[ \vz \rightsquigarrow u \right]\\
                          &= \Pr_{\vz \sim \dis(\vy'^\star)} \left[ T'(\vz) = 1 \right] 
			\end{align*}
			This concludes the proof of this case.
		\end{enumerate}
	      \item It remains to analyze the case when $\vy'^\star[i] = \bot$ and $\vy'^\star[i'] = 1$, which can be proved in the same way as the previous case $\vy'^\star[i] = 0$ and $\vy'^\star[i'] = \bot$.
                %, analyzed before, so we laeva to the reader. 
	\end{enumerate}

After this case analysis, we can safely assume that $\vy'^\star[i] = \vy'^\star[i'] = \bot$ for every $i \not \in S$. We will now show that $\vy'^\star[b_j] = \bot$ for every $j \in \{1, \ldots, m\}$. To see this, consider that in general it could be that $\vy'^\star$ \emph{forces} a certain number $k$ of features $b_j$ to get value $0$, meaning that there is a set $K \subseteq \{1, \ldots, m\}$ with $|K| = k$ such that $\vy'^\star[b_j] = 0$ for $j \in K$. We will argue that $k = 0$. Let us start by arguing that $k \leq m-2$. Indeed, assume expecting a contradiction that $k > m-2$, then by definition 
\[
\Pr_{\vz \sim \dis(\vy'^\star)} \left[ T_1(\vz) = 1 \right] = 0,
\]
and thus
\[
\Pr_{\vz \sim \dis(\vy'^\star)} \left[ T^\star(\vz) = 1 \right] =  \Pr_{\vz \sim \dis(\vy'^\star)} \left[ T_{\vy_0}(\vz) = 1 \right], 
\]
but 
\[
\Pr_{\vz \sim \dis(\vy'^\star)} \left[ T_{\vy_0}(\vz) = 1 \right] \leq \frac{1}{2},
\]
as $\vy'^\star$ cannot be a superset of $\vy^\star$, and thus at least one feature $i$ of $\vy^\star$ is undefined in $\vy'^\star$, and the event $\vz[i] = \vy^\star[i]$, which happens with probability $\frac{1}{2}$, is required for $T_{\vy_0}(\vz) = 1$. But by definition of $\delta$, if $\vy'^\star$ is the output of the computational problem, then its probability of acceptance is at least 
\[
	\Pr_{\vz \sim \dis(\vy^\star)} \left[(T' \cap T_1)(\vz) = 1\right],
\]  
and this probability is greater than $\frac{1}{2}$ (Claim~\ref{claim:3}), and thus we have a contradiction. We now safely assume $k \leq m-2$ and thus $m-k \geq 2$. 
Observe that as at least one component of $\vy^\star$ is undefined in $\vy'^\star$, we have
\[
	\Pr_{\vz \sim \dis(\vy'^\star)} \left[ T_{\vy_0}(\vz) = 1 \right] \leq \frac{1}{2} \cdot \left(\frac{1}{2}\right)^{m-k}, 
\]
and thus
\[
	\Pr_{\vz \sim \dis(\vy'^\star)} \left[(T' \cap T_1)(\vz) = 1 \right] \geq \delta - \frac{1}{2} \cdot \left(\frac{1}{2}\right)^{m-k},
\]
which, considering that 
\[
	\Pr_{\vz \sim \dis(\vy'^\star)} \left[T_1(\vz) = 1 \right]= 1 - \left( \frac{1}{2}\right)^{m-k} - (m-k)\left(\frac{1}{2}\right)^{m-k},
\]
implies that 
\[
\Pr_{\vy \sim \dis(\vz'^\star)} \left[T'(\vz) = 1 \right] \geq \frac{\delta - \frac{1}{2} \cdot \left(\frac{1}{2}\right)^{m-k}}{1 - \left( \frac{1}{2}\right)^{m-k} - (m-k)\left(\frac{1}{2}\right)^{m-k}},
\]
as $T'(\vz) = 1$ and $T_1(\vz) = 1$ are independent events.
We now show that the RHS of the previous equation is greater than $\delta$. Indeed, for ease of notation set $r \coloneqq (m-k+1)$ and note how the RHS can be rewritten as
\[
\frac{\delta - \left(\frac{1}{2}\right)^r}{1 - r\left(\frac{1}{2}\right)^{r-1}}.
\]
Now consider that as $m-k \geq 2$ we have that $r \geq 3$ and thus $2^{r-1} > r$, which implies $r\left(\frac{1}{2}\right)^{r-1} < 1$, and thus the denominator of the previous equation is positive, implying that what we want to show is equivalent to
\[
	\delta - \left(\frac{1}{2}\right)^r > \delta \left(1 - r\left(\frac{1}{2}\right)^{r-1} \right),
\]
which is in turn equivalent to
\[
\delta r \left(\frac{1}{2}\right)^{r-1} > \left(\frac{1}{2}\right)^{r},
\]
but as by Claim~\ref{claim:3} we have $\delta > \frac{1}{2}$, and $r \geq 3$, the previous equation is trivially true. We have therefore showed that
\[
\Pr_{\vz \sim \dis(\vy'^\star)} \left[T'(\vz) = 1 \right] > \delta.
\] 

Now let $\vy^\ominus$ be the partial instance such that $\vy^\ominus[i] = \vy'^\star[i]$ for every $i \in S$, and is undefined in all other features. Note that $\vy^\ominus \subseteq \vy'^\star$ and also $\vy^\ominus \subseteq \vy^\star$. If $\vy^\ominus = \vy'^\star$, then $\vy'^\star \subseteq \vy^\star$ which is what we are hoping to prove. 
So we now assume $\vy^\ominus \subsetneq \vy'^\star$ expecting a contradiction. Note that  $T'$ does not use the $b_j$ features at all, and therefore we have that
\[
\Pr_{\vz \sim \dis(\vy^\ominus)} \left[T'(\vz) = 1 \right] = \Pr_{\vz \sim \dis(\vy'^\star)} \left[T'(\vz) = 1 \right]  > \delta.
\]
We will use this to prove that $\vy^\ominus$ would have been a valid outcome of the computing problem, thus contradicting the minimality of $\vy'^\star$. Indeed,
\begin{align*}
	\Pr_{\vz \sim \dis(\vy^\ominus)} \left[T^\star(\vz) = 1 \right] &> \Pr_{\vz \sim \dis(\vy^\ominus)}\left[T'(\vz) = 1 \right] \cdot \Pr_{\vz \sim \dis(\vy^\ominus)}\left[T_1(\vz) = 1 \right], 
\end{align*}
and note that as 
\[
\Pr_{\vz \sim \dis(\vy^\ominus)}\left[T'(\vz) = 1 \right] > \delta,
\]
it must be the case that 
\[
\Pr_{\vz \sim \dis(\vy^\ominus)}\left[T'(\vz) = 1 \right] \geq \delta + \left(\frac{1}{2}\right)^{2n - |S|},
\]
as only $2n - |S|$ features appear as labels in $T'$ and, thus, the completion probability of any partial instance over $T'$ must be an integer multiple of $\left(\frac{1}{2}\right)^{2n - |S|}$. Now let us abbreviate $2n- |S|$ as $\ell$ and choose $m \geq 2\ell$. We thus have that
\begin{align*}
	\Pr_{\vz \sim \dis(\vy^\ominus)}\left[T^\star(\vz) = 1 \right] &\geq \Pr_{\vz \sim \dis(\vy^\ominus)}\left[T'(\vz) = 1\right] \cdot \Pr_{\vz \sim \dis(\vy^\ominus)}\left[T_1(\vz) = 1 \right]\\
	&\geq \left(\delta + \left(\frac{1}{2}\right)^{2n - |S|}\right) \left(1 - (m+1)\left(\frac{1}{2}\right)^{m} \right)\\
	&\geq \left(\delta + \left(\frac{1}{2}\right)^{2n - |S|}\right) \left(1 - (2\ell+1)\left(\frac{1}{2}\right)^{2\ell} \right)\\
	&= \delta - \delta(2\ell+1)\left(\frac{1}{2}\right)^{2\ell} + \left(\frac{1}{2}\right)^\ell - (2\ell+1)\left(\frac{1}{2}\right)^{3\ell}\\
	&\geq  \delta - (2\ell+1)\left(\frac{1}{2}\right)^{2\ell} + \left(\frac{1}{2}\right)^\ell - (2\ell+1)\left(\frac{1}{2}\right)^{2\ell}\\
	&= \delta - (4\ell+2)\left(\frac{1}{2}\right)^{2\ell} + \left(\frac{1}{2}\right)^\ell\\
	&= \delta + \left(\frac{1}{2}\right)^\ell\left(1- (4\ell+2)\left(\frac{1}{2}\right)^{\ell}\right),
\end{align*}
where the last parenthesis is positive for $\ell \geq 5$, which can be assumed without loss of generality as otherwise the original instance of the decision problem would have constant size. We have thus concluded that
\begin{align}
  \label{eq:ominus-delta}
\Pr_{\vz \sim \dis(\vy^\ominus)} \left[T^\star(\vz) = 1 \right] \geq \Pr_{\vz \sim \dis(\vy^\ominus)} \left[(T' \cap T_1)(\vz) = 1 \right] \geq \delta,
\end{align}
thus showing that $\vy^\ominus$ is a valid outcome for the computing problem, which contradicts the minimality of $\vy'^\star$. This in turn implies that $\vy'^\star = \vy^\ominus$, and thus subsequently that $\vy'^\star \subseteq \vy^\star$. Let us now show how  by combining Claims~\ref{claim:2}, \ref{claim:5} and \ref{claim:4}, we can conclude the forward direction entirely. Indeed, note that the trivial equality 
\[
	\Pr_{\vz \sim \dis(\vy^\star)}[T_1(\vz) = 1] = \Pr_{\vz \sim \dis(\vy'^\star)}[T_1(\vz) = 1]
\]
implies that 
\[
	\Pr_{\vz \sim \dis(\vy^\star)}[T'(\vz) = 1]  \ \leq \ \Pr_{w \sim \dis(\vy'^\star)} [T'(\vz) = 1],
\]
as we already have proved that $\Pr_{\vz \sim \dis(\vy'^\star)} \left[(T' \cap T_1)(\vz) = 1 \right] \geq \delta$ by \eqref{eq:ominus-delta} and the fact that $\vy^\ominus = \vy'^\star$, and we have that $\delta = \Pr_{\vz \sim \dis(\vy^\star)}[(T' \cap T_1)(\vz) = 1]$. 
%\[
%\delta = \Pr_w[(T' \cap T_1)(w) = 1 \mid w \in \Comp(z^\star)]  \leq \Pr_w[(T' \cap T_1)(w) = 1 \mid w \in \Comp(z'^\star%)].
%\]
We can use Claims~\ref{claim:2}, \ref{claim:5} and \ref{claim:4} to conclude that
\begin{align*}
	\Pr_{\vz \sim \dis(\vy)}[T(\vz) = 1] &=  \Pr_{\vz \sim \dis(\vy^\star)} [T'_{\downarrow \vz} \in \mathcal{N}_t  \mid T'_{\downarrow \vz} \in \mathcal{N} ]\\
	&= \frac{{\displaystyle \Pr_{\vz \sim \dis(\vy^\star)}[T'_{\downarrow \vz} \in \mathcal{N}_t]}}{{\displaystyle \Pr_{\vz \sim \dis(\vy^\star)}[T'_{\downarrow \vz} \in \mathcal{N}]}}\\
	&= \frac{{\displaystyle \Pr_{\vz \sim \dis(\vy^\star)}[T'_{\downarrow \vz} \in \mathcal{N}_t ]}}{{\displaystyle \Pr_{\vz \sim \dis(\vy'^\star)}[T'_{\downarrow \vz} \in \mathcal{N}]}} \\
	&= \frac{{\displaystyle \Pr_{\vz \sim \dis(\vy^\star)}[T'(\vz) = 1 ] - \Pr_{\vz \sim \dis(\vy^\star)}[T'_{\downarrow \vz} \in \mathcal{A}_t]}}{{\displaystyle \Pr_{\vz \sim \dis(\vy'^\star)}[T'_{\downarrow \vz} \in \mathcal{N}]}}\\
	&  =  \frac{{\displaystyle \Pr_{\vz \sim \dis(\vy^\star)}[T'(\vz) = 1  ] - \Pr_{\vz \sim \dis(\vy^\star)}[T'_{\downarrow \vz} \in \mathcal{A}_t  \mid T'_{\downarrow \vz} \in \mathcal{A}] \cdot \Pr_{\vz \sim \dis(\vy^\star)}[T'_{\downarrow \vz} \in \mathcal{A}] } }{{\displaystyle \Pr_{\vz \sim \dis(\vy'^\star)}[T'_{\downarrow \vz} \in \mathcal{N}]}}\\
	&= \frac{{\displaystyle \Pr_{\vz \sim \dis(\vy^\star)}[T'(\vz) = 1  ] - \Pr_{\vz \sim \dis(\vy'^\star)}[T'_{\downarrow \vz} \in \mathcal{A}_t  \mid T'_{\downarrow \vz} \in \mathcal{A}] \cdot \Pr_{\vz \sim \dis(\vy'^\star)}[T'_{\downarrow \vz} \in \mathcal{A}]}  }{{\displaystyle \Pr_{\vz \sim \dis(\vy'^\star)}[T'_{\downarrow \vz} \in \mathcal{N}]}}\\
	&\leq  \frac{{\displaystyle \Pr_{\vz \sim \dis(\vy'^\star)}[T'(\vz) = 1  ] - \Pr_{\vz \sim \dis(\vy'^\star)}[T'_{\downarrow \vz} \in \mathcal{A}_t  \mid T'_{\downarrow \vz} \in \mathcal{A}] \cdot \Pr_{\vz \sim \dis(\vy'^\star)}[T'_{\downarrow \vz} \in \mathcal{A}]}  }{{\displaystyle \Pr_{\vz \sim \dis(\vy'^\star)}[T'_{\downarrow \vz} \in \mathcal{N}]}}\\
	&=  \frac{{\displaystyle \Pr_{\vz \sim \dis(\vy'^\star)}[T'(\vz) = 1  ] - \Pr_{\vz \sim \dis(\vy'^\star)}[T'_{\downarrow \vz} \in \mathcal{A}_t]}  }{{\displaystyle \Pr_{\vz \sim \dis(\vy'^\star)}[T'_{\downarrow \vz} \in \mathcal{N}]}}\\
	&=  \frac{{\displaystyle \Pr_{\vz \sim \dis(\vy'^\star)}[T'_{\downarrow \vz} \in \mathcal{N}_t]}  }{{\displaystyle \Pr_{\vz \sim \dis(\vy'^\star)}[T'_{\downarrow \vz} \in \mathcal{N}]}}\\
	&= \Pr_{\vz \sim \dis(\vy'^\star)}[T'_{\downarrow \vz} \in \mathcal{N}_t \mid T'_{\downarrow \vz} \in \mathcal{N}]\\
	&= \Pr_{\vz \sim \dis(\vy')}[T(\vz) = 1],
\end{align*}
where $\vy'$ is the partial instance of dimension $n$ such that $\vy'[i] = \vy'^\star[i]$ for every $i$ such that $\vy'^\star[i] \neq \bot$, and $\vy'$ is undefined in all other features. By this definition, $\vy' \subsetneq \vy$ as we had $\vy'^\star \subsetneq \vy^\star$ (because by assumption $\vy'^\star \neq \vy^\star$), and thus we have effectively proved that the instance $(T,z)$ is a positive instance of $\sbchecksub$. This concludes the proof of the forward direction.

\paragraph{Backward direction.} Assume the instance $(T,\vy)$ is a positive instance of $\sbchecksub$ and, thus, there exists some $\vy' \subsetneq \vy$ such that 
\[
	\Pr_{\vz \sim \dis(\vy')}\left[T(\vz) = 1\right] \geq \Pr_{\vz \sim \dis(\vy)}\left[T(\vz) = 1 \right]. 
\]
Define $\vz'^\star$ of the dimension of $T^\star$ based on $\vy'$ by setting $\vy'^\star[i] = \vy'[i]$ for every $i$ such that $\vy'[i] \neq \bot$, and leave the rest of $\vy'^\star$ undefined. Note that this definition immediately implies $\vy'^\star \subsetneq \vy^\star$.
By Claim~\ref{claim:5} the previous equation  implies that
\[
	\Pr_{\vz \sim \dis(\vy'^\star)}\left[T'_{\downarrow \vz} \in \mathcal{N}_t  \mid T'_{\downarrow \vz} \in \mathcal{N} \right] \geq \Pr_{\vz \sim \dis(\vy^\star)}\left[T'_{\downarrow \vz} \in \mathcal{N}_t  \mid T'_{\downarrow \vz} \in \mathcal{N} \right], 
\]
which implies in turn that 
\[
\frac{{\displaystyle \Pr_{\vz \sim \dis(\vy'^\star)}[T'_{\downarrow \vz} \in \mathcal{N}_t]}}{{\displaystyle \Pr_{\vz \sim \dis(\vy'^\star)}[T'_{\downarrow \vz} \in \mathcal{N}]}} \geq \frac{{\displaystyle \Pr_{\vz \sim \dis(\vy^\star)}[T'_{\downarrow \vz} \in \mathcal{N}_t]}}{{\displaystyle \Pr_{\vz \sim \dis(\vy^\star)}[T'_{\downarrow \vz} \in \mathcal{N}]}}.
\]
By Claim~\ref{claim:4} the denominators of the previous inequality are equal and, thus, 
\[
\Pr_{\vz \sim \dis(\vy'^\star)}\left[T'_{\downarrow \vz} \in \mathcal{N}_t\right] \geq \Pr_{\vz \sim \dis(\vy^\star)}\left[T'_{\downarrow \vz} \in \mathcal{N}_t \right], 
\]
from where 
\begin{align*}
  \hspace{-40pt}\Pr_{\vz \sim \dis(\vy'^\star)}&\left[T'_{\downarrow \vz} \in \mathcal{N}_t\right] +\ \Pr_{\vz \sim \dis(\vy^\star)}\left[T'_{\downarrow \vz} \in \mathcal{A}_t \right] \ \geq \\
  &\Pr_{\vz \sim \dis(\vy^\star)}\left[T'_{\downarrow \vz} \in \mathcal{N}_t \right]+ \Pr_{\vz \sim \dis(\vy^\star)}\left[T'_{\downarrow \vz} \in \mathcal{A}_t \right] \ = \
\Pr_{\vz \sim \dis(\vy^\star)}\left[T'(\vz) = 1\right].
\end{align*}
But combining Claims~\ref{claim:2} and~\ref{claim:4} we have that
\[
\Pr_{\vz \sim \dis(\vy^\star)}\left[T'_{\downarrow \vz} \in \mathcal{A}_t\right] = \Pr_{\vz \sim \dis(\vy'^\star)}\left[T'_{\downarrow \vz} \in \mathcal{A}_t  \right],
\]
which when combined with the previous equation gives us
\[
\Pr_{\vz \sim \dis(\vy'^\star)}\left[T'(\vz) = 1  \right] \geq \Pr_{\vz \sim \dis(\vy^\star)}\left[T'(\vz) = 1 \right],
\]
and using again that 
\[
\Pr_{\vz \sim \dis(\vy'^\star)}\left[T_1(\vz) = 1\right] =\Pr_{\vz \sim \dis(\vy^\star)}\left[T_1(\vz) = 1\right],
\]
we obtain that 
\[
\Pr_{\vz \sim \dis(\vy'^\star)}\left[(T' \cap T_1)(\vz) = 1 \right] \geq \Pr_{\vz \sim \dis(\vy^\star)}\left[(T' \cap T_1)(\vz) = 1 \right] = \delta.
\]
Finally, by observing that 
\[
\Pr_{\vz \sim \dis(\vy'^\star)}\left[T^\star(\vz) = 1 \right]
\ \geq \ \Pr_{\vz \sim \dis(\vy'^\star)}\left[(T' \cap T_1)(\vz) = 1 \right] \ \geq \ \delta,
\]
we have that $\vy'^\star$ is a valid output for the computational problem, and give it is a strict subset of $\vy^\star$, the result of $\minimal(T^\star, \vx, \delta)$ cannot be equal to $\vy^\star$. This concludes the backward direction, and with it the entire proof is complete.
\end{proof}

\section{Proof of Theorem \ref{thm:split-number-poly}}
\label{app:thm:split-number-poly}
\begin{thmbis}{thm:split-number-poly}
Let $c\geq 1$ be a fixed integer. Both $\minimum$ and $\minimal$ can be solved in polynomial time for decision trees 
with split number at most $c$. 
\end{thmbis}

\begin{proof}
 It suffices to provide a polynomial time algorithm for $\minimum$. (The same algorithm works for $\minimal$  
as a minimum $\delta$-SR is in particular minimal.) In turn, using standard arguments, 
it is enough to provide a polynomial time algorithm for the following decision problem
$\checkmin$: Given a tuple $(T,\vy,\delta, k)$, where $T$ is a decision tree of dimension $n$, $\vy\in\{0,1,\bot\}^n$ is a partial instance, 
$\delta\in (0,1]$, and $k\geq 0$ is an integer, decide whether there is a partial instance $\vy'\subseteq \vy$ such that $n-|\vy'|_{\bot}\leq k$ (i.e., $\vy$ has at most $k$ defined components) and 
$\Pr{}_{\!{\vz}}[T(\vz) = 1\mid \vz \in \Comp(\vy')]\geq \delta$.

In order to solve $\checkmin$ over an instance $(T,\vy,\delta, k)$, where $T$ has split number at most $c$, we apply dynamic programming 
over $T$ in a bottom-up manner. Let $Z\subseteq\{1,\dots, n\}$ be the set of features defined in $\vy$, that is, features $i$ with $\vy[i]\neq \bot$. 
Those are the features we could eventually remove when looking for $\vy'$. 
For each node $u$ in $T$, we solve a polynomial number of subproblems over the subtree $T_u$. We define
\[
{\sf Int}(u):= \feat\left(N^{\downarrow}_u\right) \cap \feat\left(N^{\uparrow}_u\right)\cap Z \qquad \qquad {\sf New}(u):= \left(\feat\left(N^{\downarrow}_u\right)\setminus {\sf Int}(u)\right)\cap Z.
\]
In other words, ${\sf Int}(u)$ are the features appearing both inside and outside $T_u$, while ${\sf New}(u)$ are the features only inside $T_u$, 
that is, the new features introduced below $u$. Both sets are restricted to $Z$ as features not in $Z$ play no role in the process. 

Each particular subproblem is indexed by a possible size $s\in\{0,\dots, k\}$ and a possible set $J\subseteq {\sf Int}(u)$
with $|J|\leq s$ and the goal is to compute the quantity:
\[
p_{u,s,J} := \max_{\vy'\in\, \mathcal{C}_{u,s,J}} \Pr{}_{\!{\vz}}[T_u(\vz) = 1\mid \vz \in \Comp(\vy')],
\]
where $\mathcal{C}_{u,s,J}$ is the space of partial instances $\vy'\subseteq \vy$  with $n-|\vy'|_{\bot}\leq s$ and such that $\vy'[i]=\vy[i]$ for 
$i\in J$ and $\vy'[i]=\bot$ for $i\in {\sf Int}(u)\setminus J$. In other words, the set $J$ fixes the behavior on ${\sf Int}(u)$ (keep features in $J$, remove features in ${\sf Int}(u)\setminus J$) and hence the maximization occurs over choices on the set ${\sf New}(u)$ (which features are kept and which features are removed). The key idea is that $p_{u,s,J}$ can be computed inductively using the information already computed for the children $u_1$ and $u_2$ of $u$. 
Intuitively, this holds since the common features between $T_{u_1}$ and $T_{u_2}$ are at most $c$, which is a fixed constant, 
and hence we can efficiently synchronize the information stored for 
$u_1$ and $u_2$. Finally, to solve the instance $(T,\vy,\delta, k)$ we simply check whether $p_{r, k,\emptyset}\geq \delta$, for the root $r$ of $T$.

Formally, let us define for a set $H\subseteq Z$, the partial instance $\vy_{H}\in\{0,1,\bot\}^n$ such that $\vy_{H}[i]=\vy[i]$ for every $i\in H$,
and $\vy_{H}[i]=\bot$ for every $i\not\in H$. In particular, $\vy_{H}\subseteq \vy$. 
Then we can write $p_{u,s,J}$ as 
\[
p_{u,s,J} = \max_{\substack{K\subseteq {\sf New}(u)\\ |K|\leq\, s-|J|}} \Pr{}_{\!{\vz}}[T_u(\vz) = 1\mid \vz \in \Comp(\vy_{J\cup K})].
\]

Let $u_1$ and $u_2$ be the children of $u$. We have that ${\sf New}(u)$ is the disjoint union of:
\[
{\sf New}(u) = {\sf New}(u_1) \cup {\sf New}(u_2) \cup {\sf Sync}(u),
\]
where ${\sf Sync}(u):={\sf New}(u) \cap \left(\feat\left(N^{\downarrow}_{u_1}\right)\cap \feat\left(N^{\downarrow}_{u_2}\right)\right)$. 
In other words, the features in ${\sf Sync}(u)$ are the features that are in both $T_{u_1}$ and $T_{u_2}$ but
not outside $T_{u}$. We conclude by explaining the computation of $p_{u,s,J}$. We consider the following cases:
\begin{enumerate}
\item The feature $i$ labeling $u$ is in $J$. This means we have to keep feature $i$. If $\vy[i]=0$, then to compute 
$p_{u,s,J}$ we can simply look at $u_1$ (the left child). Note that ${\sf Int}(u_1)$ is the disjoint union of 
${\sf Int}(u_1)\cap {\sf Int}(u)$ and ${\sf Sync}(u)$. Then 
\[
p_{u,s,J} = \max_{\substack{J'\subseteq {\sf Sync}(u)\\ |J'|\leq s - |{\sf Int}(u_1)\cap J|}} p_{u_1,s, ({\sf Int}(u_1)\cap J) \cup J'}.
\]
This computation can be done in polynomial time as ${\sf Sync}(u)\leq c$ and then there are a constant number
of possible $J'\subseteq {\sf Sync}(u)$. The case when $\vy[i]=1$ is analogous, taking $u_2$ instead of $u_1$. 
\item The feature $i$ labeling $u$ is either outside $Z$ or belongs to ${\sf Int}(u)\setminus J$. This means feature $i$ is undefined. Again, we have that ${\sf Int}(u_1)$ is the disjoint union of 
${\sf Int}(u_1)\cap {\sf Int}(u)$ and ${\sf Sync}(u)$. Similarly, ${\sf Int}(u_2)$ is the disjoint union of 
${\sf Int}(u_2)\cap {\sf Int}(u)$ and ${\sf Sync}(u)$. Then
\[
p_{u,s,J} = \max_{\substack{J'\subseteq {\sf Sync}(u)\\|J'|\leq s - |J|}} 
\max_{\substack{0\leq s_1,s_2\leq s\\ s_1+s_2\leq s- |J|-|J'|}}
\frac{1}{2}\cdot p_{u_1,s_1, ({\sf Int}(u_1)\cap J) \cup J'} + \frac{1}{2}\cdot p_{u_2,s_2, ({\sf Int}(u_2)\cap J) \cup J'}.
\] 
Again, this can be done in polynomial time as ${\sf Sync}(u)\leq c$.
\item Finally, the remaining case is that the feature $i$ labeling $u$ is in ${\sf New}(u)$. In that case we have the two possibilities: either we keep feature $i$ or we remove it. If $s-|J|=0$, then the only possible choice is to remove the feature $i$, and hence $p_{u,s,J}$ is computed exactly as in case (2). If $s-|J| > 0$. Then we take the maximum between the cases when we keep feature $i$ and the case when we remove feature $i$. For the latter,  
$p_{u,s,J}$ is computed exactly as in case (2). For the former, we compute $p_{u,s,J}$ in a similar way as in case (1). More precisely, if $\vy[i]=0$, then:
\[
p_{u,s,J} = \max_{\substack{J'\subseteq {\sf Sync}(u)\\ |J'|\leq s-1 - |{\sf Int}(u_1)\cap J|}} p_{u_1,s-1, ({\sf Int}(u_1)\cap J) \cup J'}.
\]
The case $\vy[i]=1$ is analogous.
\end{enumerate}
\end{proof}

\section{Proof of Lemma \ref{lemma:monotone}}
\label{app:lemma:monotone}

%\begin{lemma}
%Let $\mathfrak{C}$ be a class of monotone models, ${\cal M} \in \mathfrak{C}$ a model of dimension $n$, and $\es \in \{0, 1\}^n$ an instance. Consider any $\delta \in (0,1]$. Then if $\es' \subseteq \es$ is a $\delta$-SR for $\es$ under $\cal M$ which is not minimal, then there is a partial instance $\es^\star = \es' \setminus \{i\}$, for some $i \in \{1,\dots,n\}$, such that $\es^\star$ is a $\delta$-SR for $\es$ under $\cal M$.
%Let $\mathcal{C}$ be a class of monotone models, $M \in \mathcal{C}$ be a model of dimension $d$, and $x \in \{0, 1\}^d$ be an instance. Consider any $\delta \in [0,1]$. Then if $y \subseteq x$ is a $\delta$-SR for $x$ under $M$, but is not minimal, there exists a partial instance $z = y \setminus \{i\}$ for some $i$ that is also a $\delta$-SR for $x$ under $M$.
%\label{lemma:monotone-close-exp}
%\end{lemma}

{\bf Lemma 1.} {\em 
Let $\mathfrak{C}$ be a class of monotone models, ${\cal M} \in \mathfrak{C}$ a model of dimension $n$, and $\vx \in \{0, 1\}^n$ an instance. Consider any $\delta \in (0,1]$. Then if $\vy \subseteq \vx$ is a $\delta$-SR for $\vx$ under $\cal M$ which is not minimal, then there is a partial instance $\vy' \coloneqq \vy \setminus \{i\}$, for some $i \in \{1,\dots,n\}$, such that $\vy'$ is a $\delta$-SR for $\vx$ under $\cal M$.}

\begin{proof}

Note that, if ${\cal M}(\vx) = 1$ then we can safely assume that for every $i$ where $\vy[i] \neq \bot$ it holds that $\vy[i] = 1$, as otherwise if $\vy[i^\star] = 0$ for some $i^\star$, then the lemma trivially holds by setting $\vy' = \vy \setminus \{i^\star\}$ because of monotonicity. Similarly, if ${\cal M}(\vx) = 0$ then we can safely assume that for every $i$ where $\vy[i] \neq \bot$ it holds that $\vy[i] = 0$.

As by hypothesis $\vy$ is not minimal, there exists a $\delta$-SR  $\vy^\star \subsetneq \vy$ that minimizes $|\vy^\star|_\bot$. We will prove that $|\vy^\star|_\bot = |\vy|_\bot + 1$, from where the lemma immediately follows.

Assume for the sake of a contradiction that $|\vy^\star|_\bot > |\vy|_\bot + 1$. Then, there must exist a feature
	$i^\star$ that $\vy^\star[i^\star] = \bot \neq \vy[i^\star]$, and such that $\vy^\star \cup \{i^\star\} \neq \vy$, where $\vy^\star \cup \{i^\star\}$ is defined as
	\[
		(\vy^\star \cup \{i^\star\})[i] = \begin{cases}
 \vy[i^\star] & \text{if } i = i^\star\\
 \vy^\star[i^\star] & \text{otherwise.}	
 \end{cases}
	\]
	Similarly we denote $\vy^\star \cup (i^\star \to \alpha)$, with $\alpha \in \{0, 1\}$, the partial instance defined as 
	\[
	(\vy^\star \cup (i^\star \to \alpha))[i] = \begin{cases}
\alpha & \text{if } i = i^\star\\
 \vy^\star[i^\star] & \text{otherwise.}	
 \end{cases}
	\]
We now claim that $\vy^\star \cup \{i^\star\}$ is also a $\delta$-SR for $\vx$ under $\mathcal{M}$, which will contradict the minimality of $\vy^\star$, as $|\vy^\star \cup \{i^\star\}|_\bot < |\vy^\star|$. 
Let us denote by $C(\mathcal{M}, \vy)$ the number of completions $\vz \in \Comp(\vy)$ such that $\mathcal{M}(\vz) = 1$. Now there are two cases, if $\mathcal{M}(\vx) = 1$ then 
\begin{align*}
C(\mathcal{M}, \vy^\star) &= C(\mathcal{M}, \vy^\star \cup (i^\star \to 0)) + C(\mathcal{M}, \vy^\star \cup (i^\star \to 1)) \\
&\leq 2 C(\mathcal{M}, \vy^\star \cup (i^\star \to 1)), \tag{Because of monotonicty}
\end{align*}
from where 
\[
\frac{C(\mathcal{M}, \vy^\star \cup (i^\star \to 1))}{2^{|\vy^\star \cup (i^\star \to 1)|_\bot}} = \frac{C(\mathcal{M}, \vy^\star \cup (i^\star \to 1))}{2^{|\vy^\star|_\bot -1}} \geq \frac{C(\mathcal{M}, \vy^\star)}{2 \cdot 2^{|\vy^\star|_\bot -1}} \geq \delta,
\]
which implies that $ \vy^\star \cup (i^\star \to 1)$ is also a $\delta$-SR (note that $ \vy^\star \cup (i^\star \to 1) = \vy^\star \cup \{ i^\star\}$ because of the initial observation), contradicting the minimality of $\vy^\star$. Similarly, if $\mathcal{M}(\vx) = 0$, then
\begin{align*}
C(\mathcal{M}, \vy^\star) &= C(\mathcal{M}, \vy^\star \cup (i^\star \to 0)) + C(\mathcal{M}, \vy^\star \cup (i^\star \to 1)) \\
&\geq 2 C(\mathcal{M}, \vy^\star \cup (i^\star \to 0)), \tag{Because of monotonicty}
\end{align*}

from where 
\begin{align*}
\frac{2^{|\vy^\star \cup (i^\star \to 0)|_\bot} - C(\mathcal{M}, \vy^\star \cup (i^\star \to 0))}{2^{|\vy^\star \cup (i^\star \to 0)|_\bot}} &= \frac{2^{|\vy^\star|_\bot -1} - C(\mathcal{M}, \vy^\star \cup (i^\star \to 0))}{2^{|\vy^\star|_\bot -1}}\\
 &\geq  1 - \frac{C(\mathcal{M}, \vy^\star)}{2 \cdot 2^{|\vy^\star|_\bot -1}} \geq \delta,
\end{align*}
thus implying that $\vy^\star \cup (i^\star \to 0)$ is also a $\delta$-SR for $\vx$ under $\mathcal{M}$, which again contradicts the minimality of $\vy^\star$.

\end{proof}

\section{Experiments}
\label{sec:experiments}

This section presents some experimental results both for the deterministic encoding $(\delta = 1)$ and for the general probabilistic encoding $(\delta < 1)$. 

\paragraph{Datasets} For testing the deterministic encoding we use the classical MNIST dataset~\citep{deng2012mnist}, binarizing features by simply setting to black all pixels of value less than 128. For testing the general probabilistic encoding we build a dataset of 5x5 images that are either  \emph{tall} or \emph{wide}, and the task is to predict the kind of a given rectangle. This idea is based on the dataset built by~\citet{choi2017compiling} for illustrating sufficient reasons.\footnote{Under the name of PI-explanations. To the best of our knowledge their dataset is not published and thus we recreated it.} 

\paragraph{Training decision trees} We use \texttt{scikit-learn}~\citep{scikit-learn} to train decision trees. In order to accelerate the training, the \texttt{splitter} parameter is set to \texttt{random}. Also, due to the natural class unbalance on the MNIST dataset, we set the parameter \texttt{class\_weight} to \texttt{balanced}.

\paragraph{Hardware} All our experiments have been run on a personal computer with the following specifications: MacBook Pro (13-inch, M1, 2020), Apple M1 processor, 16 GB of RAM.

\paragraph{Solver}
 We use \emph{CaDiCaL}~\citep{BiereFazekasFleuryHeisinger-SAT-Competition-2020-solvers}, a standard CDCL based solver. 
 In order to find $k^\star$, the minimum $k$ for which an explanation of size $k$ exists one can either proceed by using a MaxSAT solver, directly to minimize the number of features used in the explanation, or use a standard SAT solver and do a search over $k$ to find the minimum size for which an explanation exists. After testing both approaches we use the latter as it showed to be more efficient in most cases. Instead of using binary search to find the $k^\star$, we use doubling search. This is because a single instance with $k = \frac{n}{2}$ at the start of a binary search can dominate the complexity, and often $k^\star \ll n$.

\paragraph{Deterministic results} Given the compactness of the deterministic encoding, with ~$O(nk + |T|)$~clauses, it is feasible to use it for MNIST instances, for which $n = 28 \times 28 = 784$. Tables~\ref{table:det-neg} and~\ref{table:det-pos}	exhibit results obtained for this dataset when recognizing digit $1$. Figures~\ref{fig:experiments-positive-1} and~\ref{fig:experiments-negative-1} exhibit minimum sufficient reasons for positive and negative instance (respectively) on a decision tree for recognizing the digit $1$. Figures~\ref{fig:experiments-positive-3} and~\ref{fig:experiments-negative-3} show examples when recognizing the digit $3$, and finally Figures~\ref{fig:experiments-positive-9} and~\ref{fig:experiments-negative-9} exhibit examples when recognizing the digit $9$. Figure~\ref{fig:plot-linearity-1} shows empirically how time scales linearly with $k^\star$, the size of the minimum sufficient reason found.

% Note how sufficient reasons for positive instances and negative instances differ substantially; are much larger than those for negative instances in the case of 1; this suggest that the decision tree requires more information to certify that an input image contains a $1$ than to falsify it. As a concrete example, in Figure~\ref{fig:experiments-negative-1} (a) only 4 pixels are enough to decide the input image is not the digit $1$ as 	they imply; the presence of two white pixels along a diagonal that is cut in the middle by a black pixel seems intuitively enough to falsify a given input being a $1$. The purpose of the bottom left black pixel that is part of the illustrated sufficient reason is less obvious. On the other hand, most examples in Figure~\ref{fig:experiments-positive-1} show that $1$ is certified by confirming some white pixels along a fairly vertical line that is surrounded by mostly black pixels. An example of how this might be problematic is illustrated in Figure~\ref{fig:experiments-negative-1} (k), where very thin horizontal line is not being detected by such checks, fooling the classifier to believe the input is still a $1$.
%Note how in $1$s that are leaning right (e.g., Figure~\ref{fig:experiments-positive-1} (e)) it is necessary to confirm black pixels on the left of the upper-half of the image, probably to discard the input being a $7$.

\paragraph{Probabilistic results} Because of the complexity of the encoding, we test over the synthetic dataset described above in which the dimension is only $5 \times 5 = 25$. Table~\ref{table:rand} summarizes the results obtained for this dataset. We emphasize the following observations:
\begin{enumerate}
	\item Computing probabilistic sufficient reasons through the general probabilistic encoding (i.e., $\delta < 1$) is less efficient than computing deterministic ones, even by several orders of magnitude.
	\item As the value of $\delta$ approaches one, the size of the minimum $\delta$-SR approaches $k^\star$, the size of the minimum sufficient reason for the given instance. On the other hand, as $\delta$ decreases the size of the minimum $\delta$-SR goes to 0. This trade-off implies that $\delta$ can be used to control the size of the obtained explanation.
	\item The time per explanation increases significantly as $\delta$ approaches $1$, even though the encoding itself does not get any larger, implying that the resulting $\cnf$ is more challenging. Interestingly enough, for $\delta = 1$ the deterministic encoding is very efficient, thus suggesting a discontinuity. It remains a challenging problem to compute $\delta$-SRs for $\delta < 1$ in a way that at least matches the efficiency of the case $\delta = 1$.
\end{enumerate}

\begin{table}
	\caption{Experimental results for the probabilistic encoding over positive instances of tall rectangles. Each datapoint is the average of 3 instances.}
	\begin{center}
	\begin{tabular}{ccccc}\toprule
	$\delta$ & Size of smallest explanation & Time & Number of leaves & Accuracy of the tree \\  \midrule 
0.6 & 0.0 & 0.105s & 20 & 0.854\\ 
0.7 & 1.0 & 0.362s & 20 & 0.854\\ 
0.8 & 2.0 & 0.669s & 20 & 0.854\\ 
0.9 & 3.0 & 1.551s & 20 & 0.854\\ 
0.95 & 3.0 & 1.409s & 20 & 0.854\\
1.0 & 3.0 & 0.032s & 20 & 0.854\\ 
\midrule
0.6 & 0.0 & 0.183s & 30 & 0.965\\ 
0.7 & 1.0 & 0.567s & 30 & 0.965\\ 
0.8 & 1.0 & 0.578s & 30 & 0.965\\ 
0.9 & 3.67 & 5.377s & 30 & 0.965\\ 
0.95 & 5.67 & 12.27s & 30 & 0.965\\
1.0 & 6.67 & 0.037s & 30 & 0.965\\
\midrule
0.6 & 2.0 & 1.003s & 40 & 0.986\\ 
0.7 & 3.0 & 2.259s & 40 & 0.986\\ 
0.8 & 4.0 & 3.507s & 40 & 0.986\\ 
0.9 & 6.0 & 10.318s & 40 & 0.986\\ 
0.95 & 7.0 & 16.387s & 40 & 0.986\\
1.0 & 10.0 & 0.046s & 40 & 0.986\\
\midrule
0.6 & 2.67 & 1.992s & 50 & 1.0\\ 
0.7 & 4.33 & 6.111s & 50 & 1.0\\ 
0.8 & 5.0 & 7.216s & 50 & 1.0\\ 
0.9 & 7.0 & 26.58s & 50 & 1.0\\ 
0.95 & 7.67 & 33.129s & 50 & 1.0\\ 
1.0 & 8.33 & 0.044s & 50 & 1.0\\ 
\bottomrule
	\end{tabular}
	\end{center}
	
	\label{table:prob-results}
\end{table}

\begin{table}
\caption{Experimental results for the deterministic encoding over negative instances of digit $1$ in MNIST. Each datapoint corresponds to the average of 10 instances.}
	\begin{center}
	\begin{tabular}{cccc}\toprule
	Number of leaves & Size of smallest explanation & Time & Accuracy of the tree \\  \midrule
	100 & 3.6 & 0.17s & 0.988\\
125 & 3.1 & 0.141s & 0.989\\
150 & 3.7 & 0.196s & 0.989\\
175 & 4.1 & 0.216s & 0.99\\
200 & 4.3 & 0.255s & 0.991\\
225 & 3.9 & 0.269s & 0.991\\
250 & 4.1 & 0.304s & 0.992\\
275 & 4.1 & 0.334s & 0.992\\
300 & 4.4 & 0.329s & 0.993\\
325 & 4.0 & 0.292s & 0.993\\
350 & 4.3 & 0.327s & 0.993\\
375 & 4.0 & 0.283s & 0.993\\
400 & 5.2 & 0.337s & 0.993\\
425 & 6.0 & 0.346s & 0.993\\
450 & 6.8 & 0.495s & 0.993\\
475 & 6.4 & 0.401s & 0.993\\
500 & 5.8 & 0.497s & 0.993\\
\bottomrule
	\end{tabular}
	\end{center}
	\label{table:det-neg}
\end{table}

\begin{table}
\caption{Experimental results for the deterministic encoding over positive instances of digit $1$ in MNIST. Each datapoint corresponds to the average of 10 instances.}
	\begin{center}
	\begin{tabular}{cccc}\toprule
	Number of leaves & Size of smallest explanation & Time & Accuracy of the tree \\  \midrule
100 & 17.0 & 1.238 & 0.988\\
125 & 17.5 & 1.077 & 0.989\\
150 & 18.4 & 1.058 & 0.989\\
175 & 17.9 & 1.082 & 0.99\\
200 & 19.0 & 1.001 & 0.991\\
225 & 21.4 & 1.188 & 0.991\\
250 & 25.0 & 1.405 & 0.992\\
275 & 23.7 & 1.183 & 0.992\\
300 & 31.2 & 1.603 & 0.993\\
325 & 31.3 & 1.504 & 0.993\\
350 & 28.5 & 1.365 & 0.993\\
375 & 30.9 & 1.547 & 0.993\\
400 & 32.1 & 2.429 & 0.993\\
425 & 34.3 & 2.188 & 0.993\\
450 & 34.3 & 2.115 & 0.993\\
475 & 42.5 & 2.533 & 0.993\\
500 & 43.6 & 2.614 & 0.993\\
\bottomrule
	\end{tabular}
	\end{center}
		\label{table:det-pos}
\end{table}

\begin{table}
\caption{Experimental results for the randomized encoding, for $\delta = \frac{3}{4}$.}
	\begin{center}
	\begin{tabular}{cccc}\toprule
	number of leaves & size of smallest explanation & time & accuracy of the tree \\ \midrule
		15 & 4 & 30.44s & 0.86\\
	16 & 1 & 2.86s & 0.84\\
	17 & 2 & 6.91s & 0.90\\
	18 & 4 & 14.59s & 0.90\\
	19 & 3 & 14.11s & 0.88\\
	20 & 2 & 6.92s & 0.90\\ \bottomrule
	\end{tabular}
	\end{center}
				\label{table:rand}
\end{table}

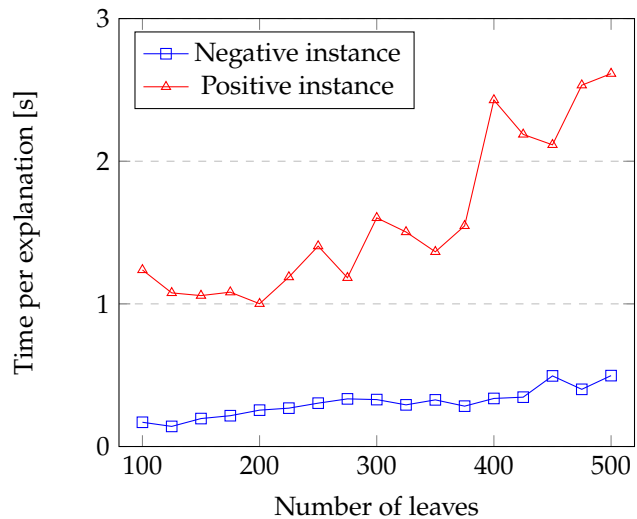
\begin{figure}
	\centering
	\begin{tikzpicture}
		\begin{axis}[
    xlabel={Number of leaves},
    ylabel={Time per explanation [s]},
    xmin=80, xmax=520,
    ymin=0, ymax=3,
    legend pos=north west,
    ymajorgrids=true,
    grid style=dashed,
]
\addplot[
    color=blue,
    mark=square,
    ]
    coordinates {
		(100, 0.17)
		(125, 0.141)
		(150, 0.196)
		(175, 0.216)
		(200, 0.255)
		(225, 0.269)
		(250, 0.304)
		(275, 0.334)
		(300, 0.329)
		(325, 0.292)
		(350, 0.327)
		(375, 0.283)
		(400, 0.337)
		(425, 0.346)
		(450, 0.495)
		(475, 0.401)
		(500, 0.497)
	};
	\legend{Negative instance, Positive instance}

	\addplot[
    color=red,
    mark=triangle,
    legend=d,
    ]
    coordinates {
		(100, 	1.238)
		(125, 1.077)
		(150, 1.058)
		(175, 1.082)
		(200, 1.001)
		(225, 1.188)
		(250, 1.405)
		(275, 1.183)
		(300, 1.603)
		(325, 1.504)
		(350, 1.365)
		(375, 1.547)
		(400, 2.429)
		(425, 2.188)
		(450, 2.115)
		(475, 2.533)
		(500, 2.614)
	};

	\end{axis}
	\end{tikzpicture}
	\caption{Time for computing a minimum sufficient reason $(\delta = 1)$ as a function of decision tree size. All datapoints correspond to an average of 10 different instances for decision trees trained to recognize the digit $1$ in the MNIST dataset.}
	\label{fig:plot-time-1}
\end{figure}

\begin{figure}
	\centering
	\begin{tikzpicture}
		\begin{axis}[
    xlabel={Number of leaves},
    ylabel={Size of MSR},
    xmin=80, xmax=520,
    ymin=0, ymax=50,
    legend pos=north west,
    ymajorgrids=true,
    grid style=dashed,
]

\addplot[
    color=blue,
    mark=square,
    ]
    coordinates {
		(100, 3.6)
		(125, 3.1)
		(150, 3.7)
		(175, 4.1)
		(200, 4.3)
		(225, 3.9)
		(250, 4.1)
		(275, 4.1)
		(300, 4.4)
		(325, 4.0)
		(350, 4.3)
		(375, 4.0)
		(400, 5.2)
		(425, 6.0)
		(450, 6.8)
		(475, 6.4)
		(500, 5.8)
	};
	\legend{Negative instance, Positive instance}

	\addplot[
    color=red,
    mark=triangle,
    legend=d,
    ]
    coordinates {
		(100, 	17.0)
		(125, 17.5)
		(150, 18.4)
		(175, 17.9)
		(200, 19.0)
		(225, 21.4)
		(250, 25.0)
		(275, 23.7)
		(300, 31.2)
		(325, 31.3)
		(350, 28.5)
		(375, 30.9)
		(400, 32.1)
		(425, 34.3)
		(450, 34.3)
		(475, 42.5)
		(500, 43.6)
	};

	\end{axis}
	\end{tikzpicture}
	\caption{Size of minimum sufficient reasons $(\delta = 1)$ as a function of decision tree size. All datapoints correspond to an average of 10 different instances for decision trees trained to recognize the digit $1$ in the MNIST dataset.}
	\label{fig:plot-explanation-size-1}
\end{figure}
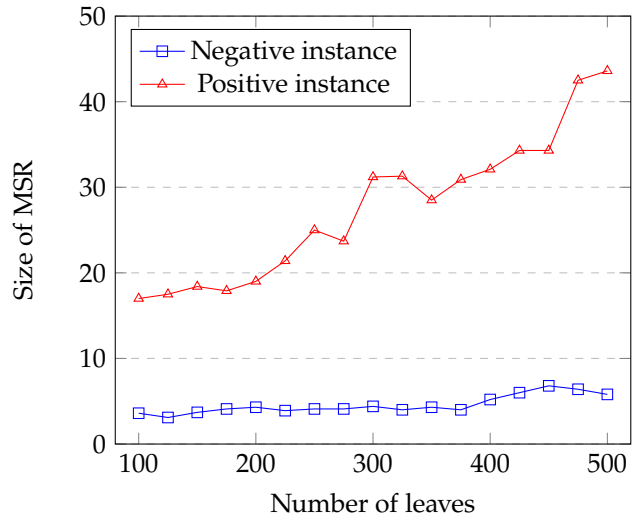

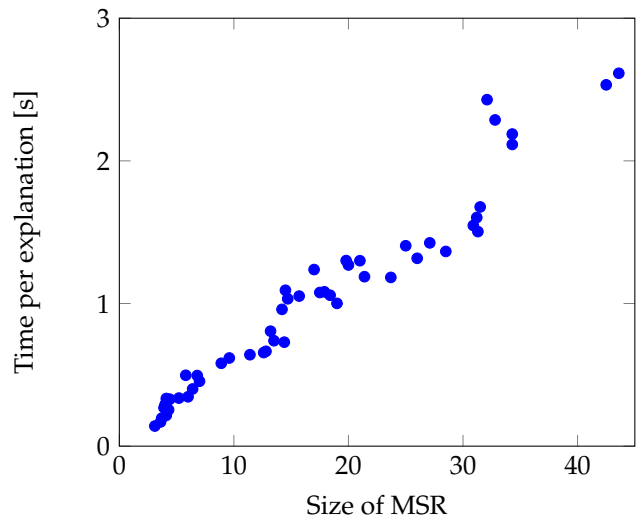
\begin{figure}
	\centering
	\begin{tikzpicture}
		\begin{axis}[
    xlabel={Size of MSR},
    ylabel={Time per explanation [s]},
    xmin=0, xmax=45,
    ymin=0, ymax=3,
    legend pos=north west,
    grid style=dashed,
]
\addplot[
    color=blue,
    only marks,
    ]
    coordinates {
    	(3.1, 0.141)
		(3.6, 0.17)
		(3.7, 0.196)
		(3.9, 0.269)
		(4.0, 0.292)
		(4.0, 0.283)
		(4.1, 0.216)
		(4.1, 0.304)
		(4.1, 0.334)
		(4.3, 0.255)
		(4.3, 0.327)
		(4.4, 0.329)
		(5.2, 0.337)
		(5.8, 0.497)
		(6.0, 0.346)
		(6.4, 0.401)		
		(6.8, 0.495)
		(7.0, 0.455)
		(8.9, 0.581)
		(9.6, 0.618)
		(11.4, 0.641)
		(12.6, 0.656)
		(12.8, 0.665)
		(13.2, 0.806)
		(13.5, 0.739)
		(14.2, 0.959)
		(14.4, 0.729)
		(14.5, 1.093)
		(14.7, 1.033)
		(15.7, 1.052)
		(17.0, 1.238)
		(17.5, 1.077)
		(17.9, 1.082)
		(18.4, 1.058)
		(19.0, 1.001)
		(19.8, 1.301)
		(20, 1.27)
		(21, 1.30)
		(21.4, 1.188)
		(23.7, 1.183)
		(25.0, 1.405)
		(26.0, 1.317)
		(27.1, 1.425)
		(28.5, 1.365)
		(30.9, 1.547)
		(31.2, 1.603)
		(31.3, 1.504)
		(31.5, 1.677)
		(32.1, 2.429)
		(32.8, 2.287)
		(34.3, 2.188)
		(34.3, 2.115)
		(42.5, 2.533)
		(43.6, 2.614)
	};
	
	\end{axis}
	\end{tikzpicture}
	\caption{Relationship between the size of the Minimum Sufficient Reason and the time it takes to obtain it.}
	\label{fig:plot-linearity-1}
\end{figure}

\begin{figure}
\centering
		\begin{subfigure}{.3\textwidth}
			\centering 
			\includegraphics[scale=0.1]{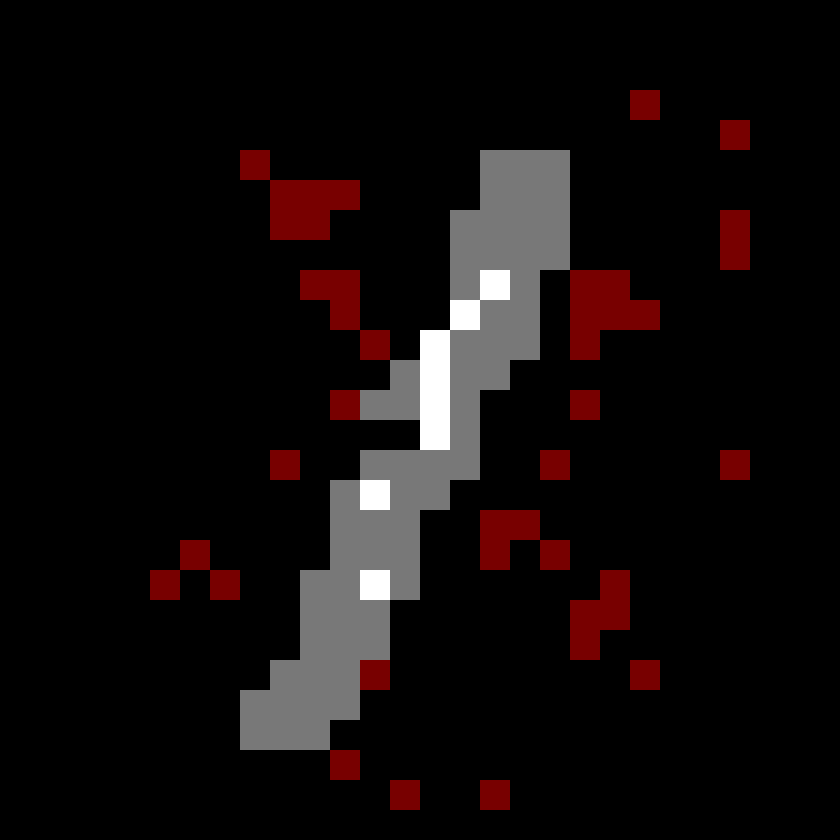}
			\caption{$k^\star = 49$}
		\end{subfigure}  
		\begin{subfigure}{.3\textwidth}
			\centering
			\includegraphics[scale=0.1]{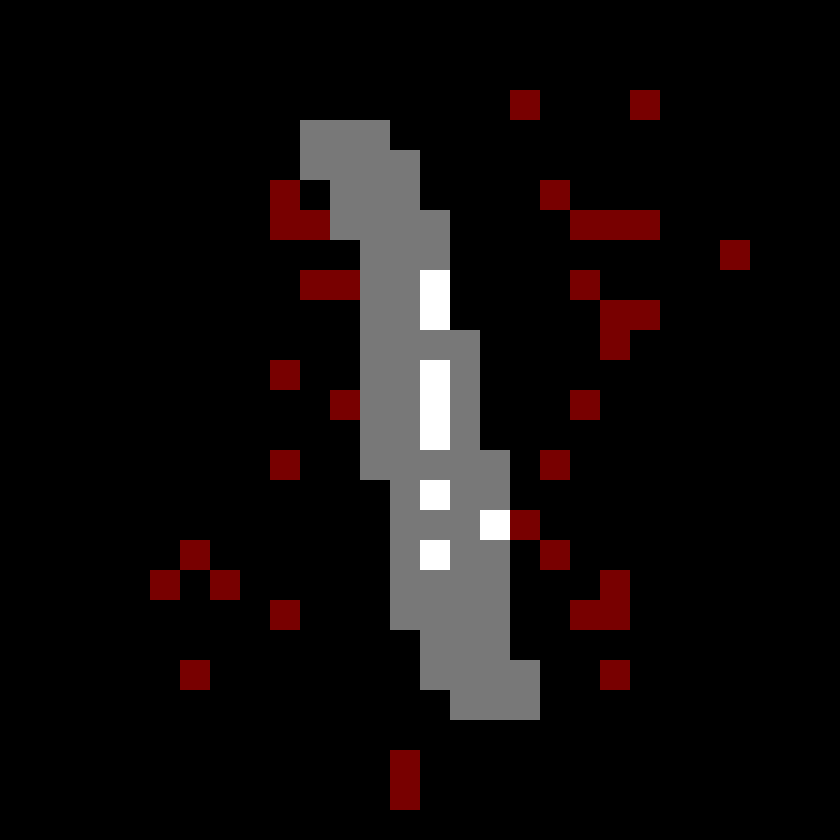}
			\caption{$k^\star = 42$}
		\end{subfigure}
		\begin{subfigure}{.3\textwidth}
			\centering
			\includegraphics[scale=0.1]{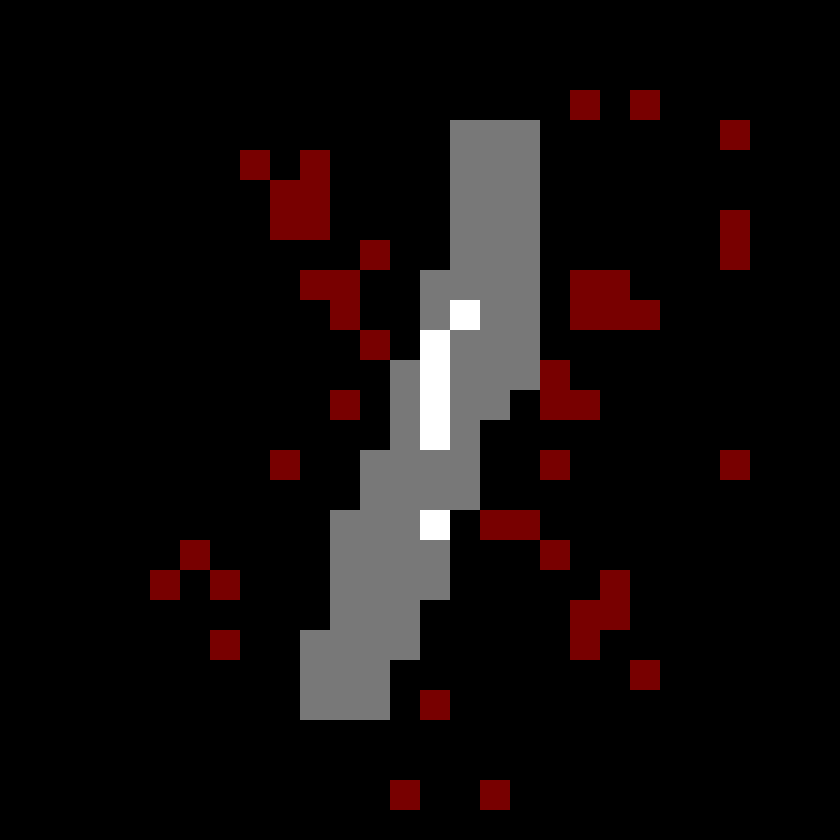}
			\caption{$k^\star = 49$}
		\end{subfigure}

		\begin{subfigure}{.3\textwidth}
			\centering
			\includegraphics[scale=0.1]{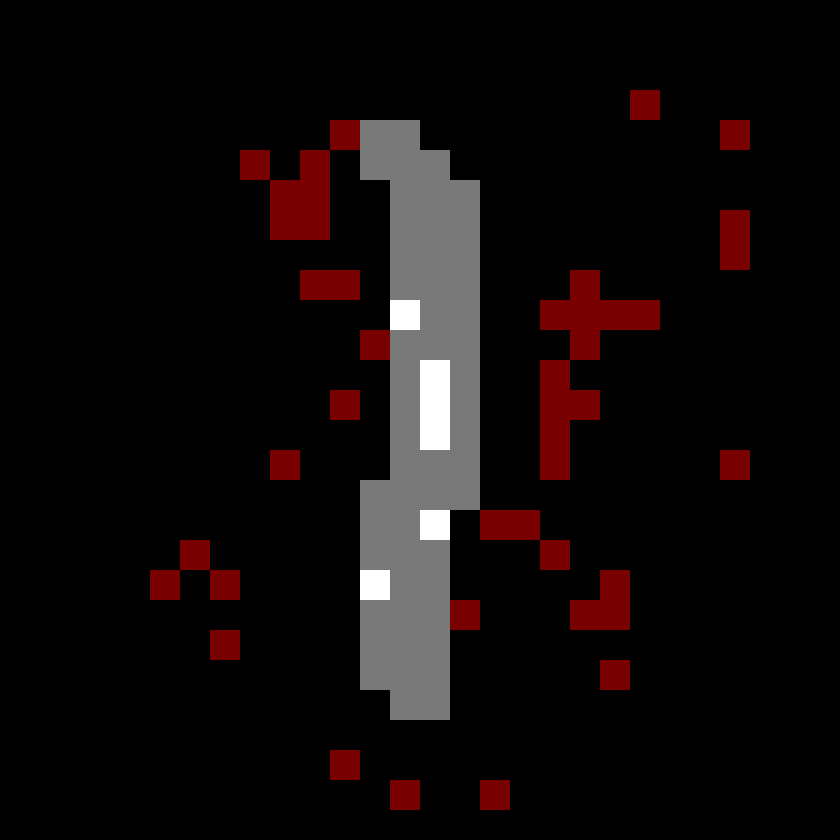}
			\caption{$k^\star = 49$}
		\end{subfigure}
		\begin{subfigure}{.3\textwidth}
			\centering
			\includegraphics[scale=0.1]{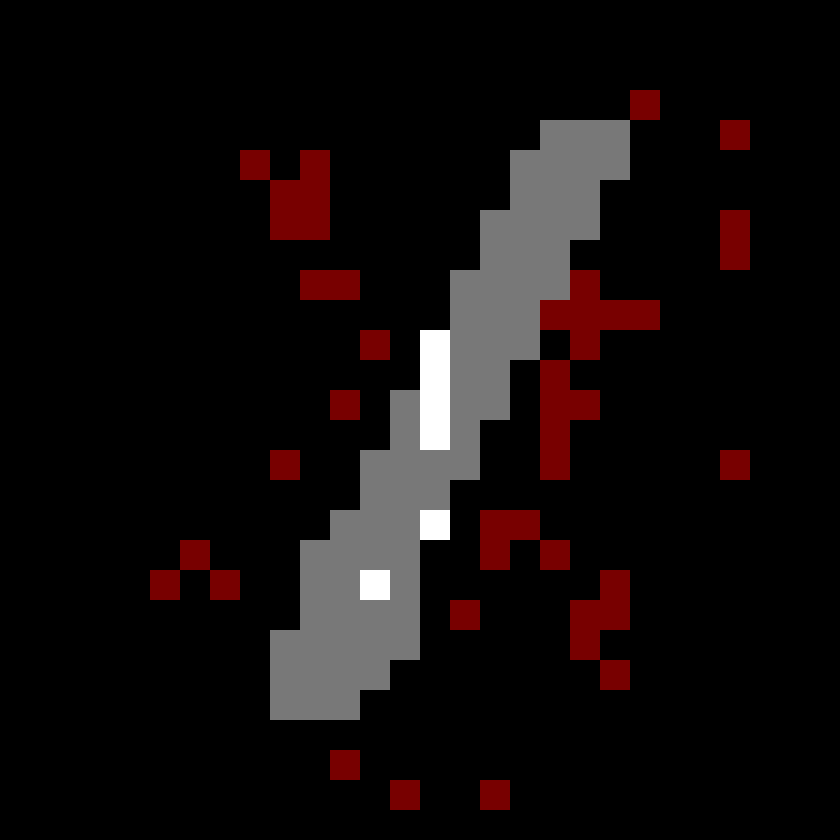}
			\caption{$k^\star = 49$}
		\end{subfigure}
		\begin{subfigure}{.3\textwidth}
			\centering
			\includegraphics[scale=0.1]{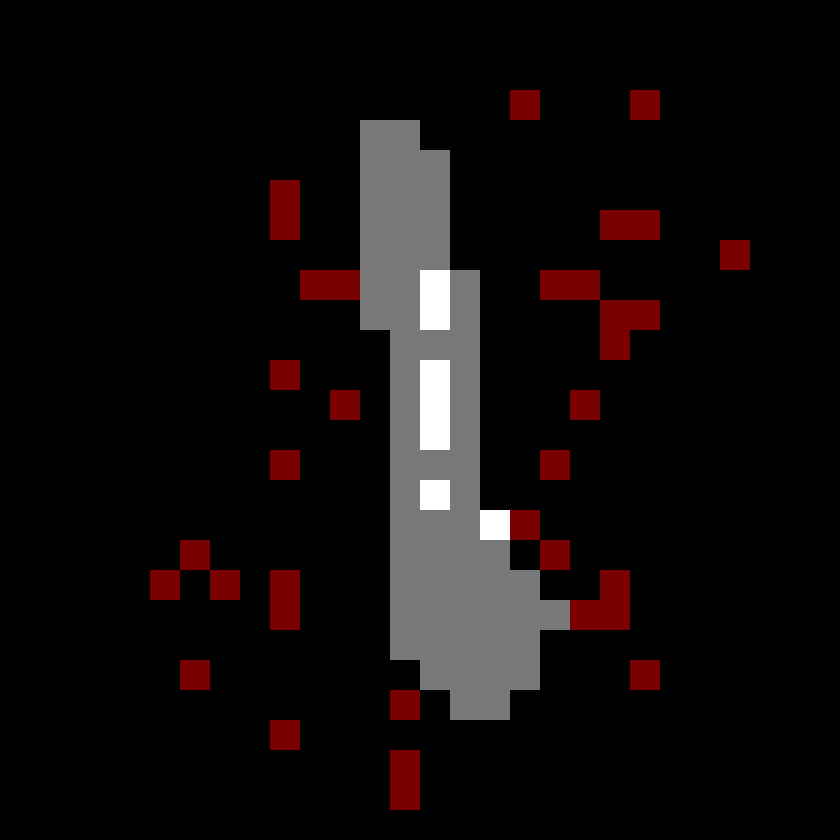}
			\caption{$k^\star = 42$}
		\end{subfigure}
		
		\begin{subfigure}{.3\textwidth}
			\centering
			\includegraphics[scale=0.1]{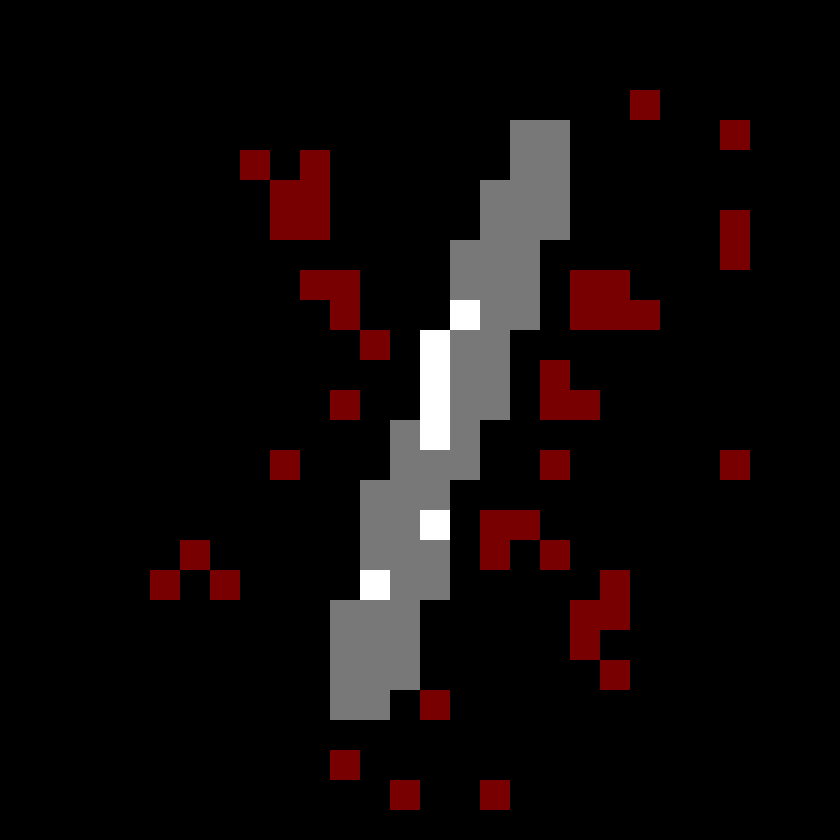}
			\caption{$k^\star = 49$}
		\end{subfigure}
		\begin{subfigure}{.3\textwidth}
			\centering
			\includegraphics[scale=0.1]{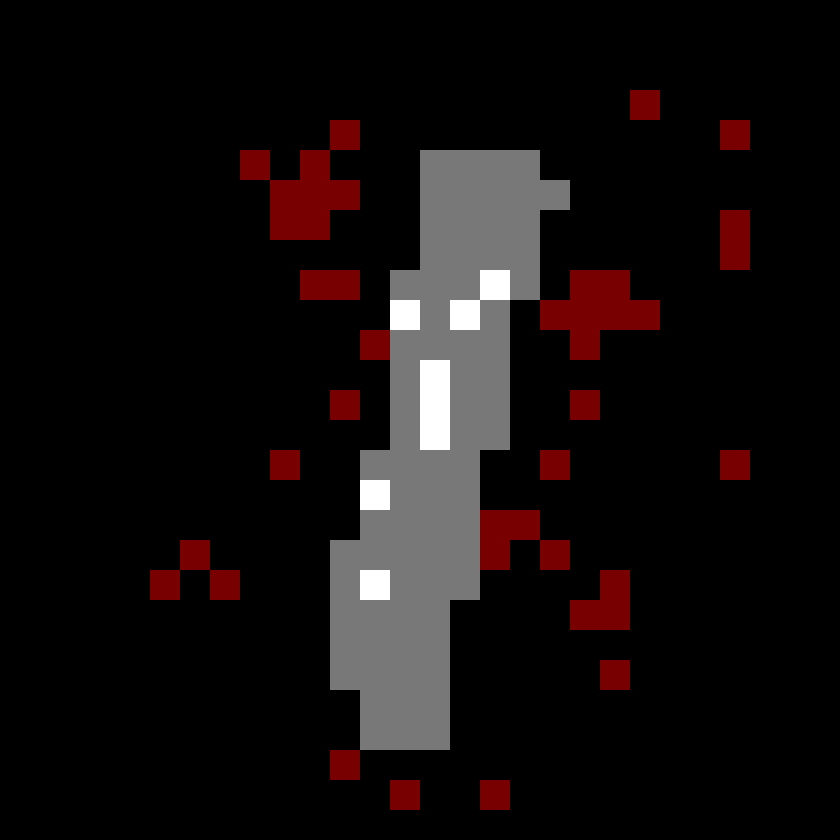}
			\caption{$k^\star = 49$}
		\end{subfigure}
		\begin{subfigure}{.3\textwidth}
			\centering
			\includegraphics[scale=0.1]{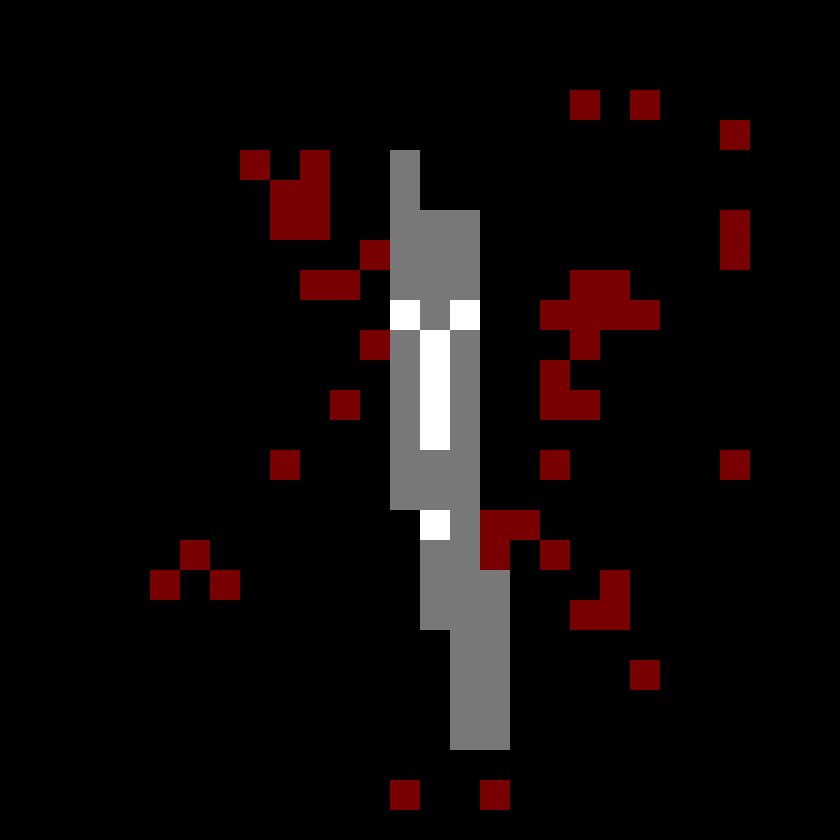}
			\caption{$k^\star = 49$}
		\end{subfigure}
		
		\begin{subfigure}{.3\textwidth}
			\centering
			\includegraphics[scale=0.1]{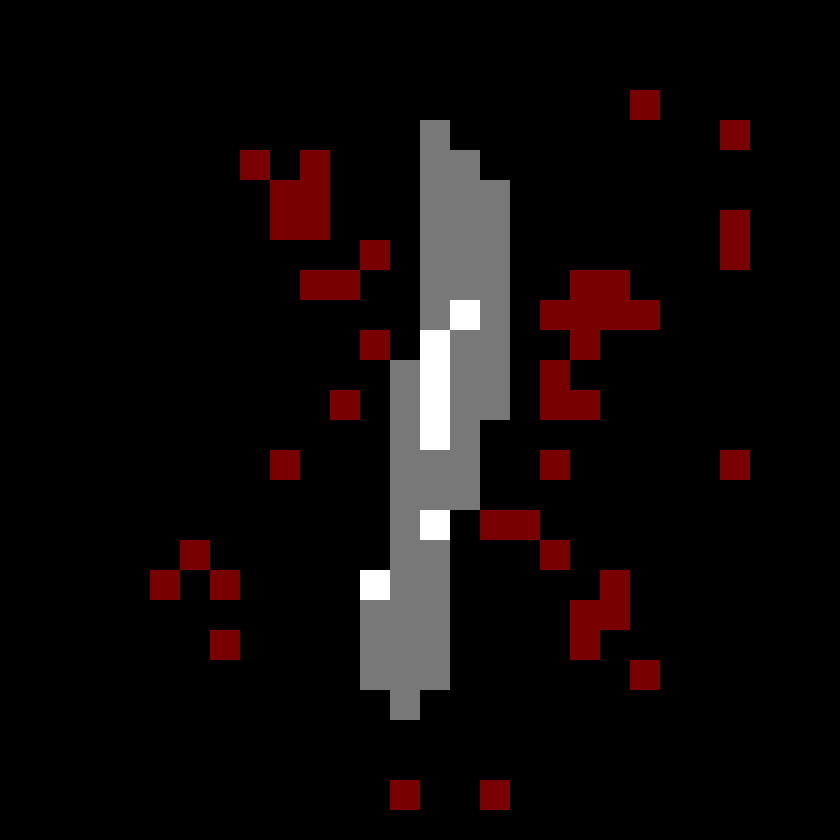}
			\caption{$k^\star = 49$}
		\end{subfigure}
		\begin{subfigure}{.3\textwidth}
			\centering
			\includegraphics[scale=0.1]{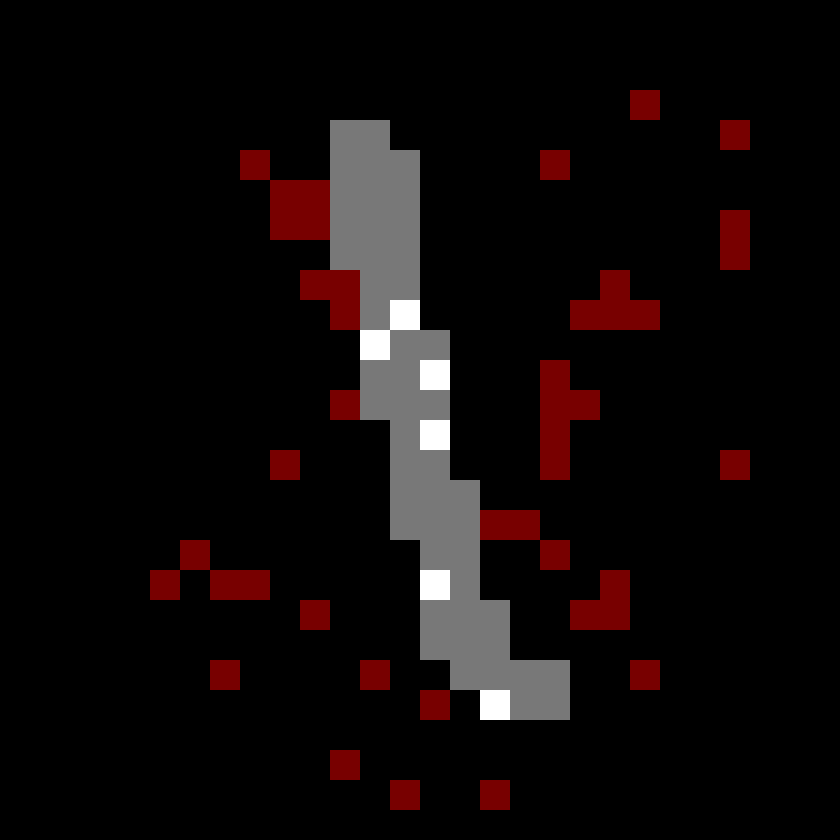}
			\caption{$k^\star = 49$}
		\end{subfigure}
		\begin{subfigure}{.3\textwidth}
			\centering
			\includegraphics[scale=0.1]{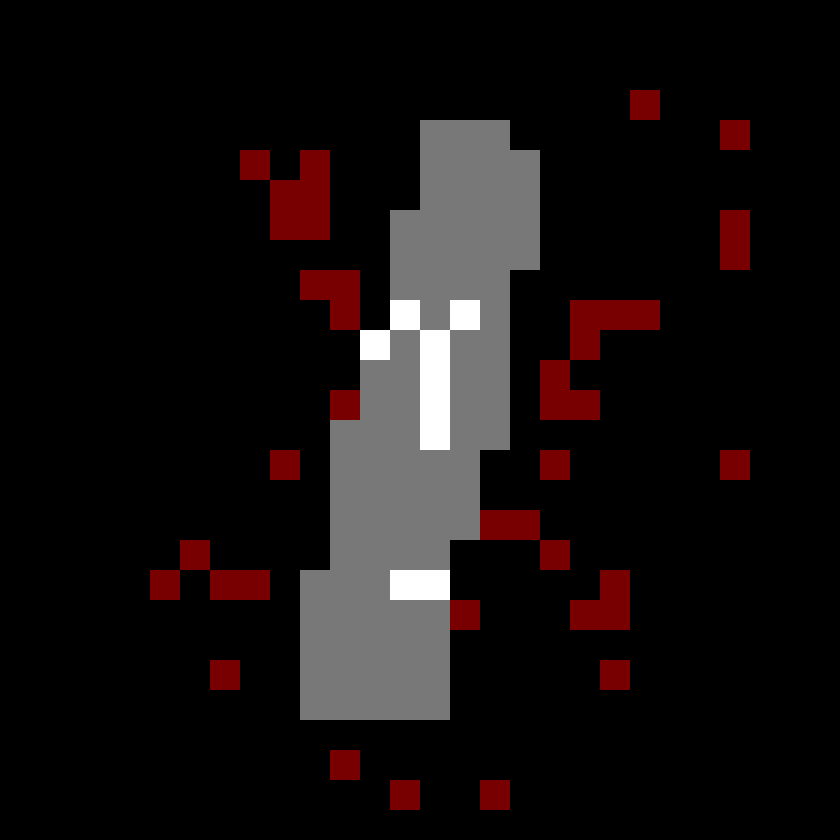}
			\caption{$k^\star = 49$}
		\end{subfigure}

		\caption{Examples of \emph{Minimum Sufficient Reasons} over the MNIST dataset. All instances are (correctly predicted) positive instances for a decision tree of $591$ leaves that detects the digit $1$. Light pixels of the original image are depicted in grey, and the light pixels of the original image that are part of the minimum sufficient reason are colored white. Dark pixels that are part of the minimum sufficient reason are colored with red. Individual captions denote the size of the minimum sufficient reasons with $k^\star$.}
		\label{fig:experiments-positive-1}
\end{figure}

\begin{figure}
\centering
		\begin{subfigure}{.3\textwidth}
			\centering 
			\includegraphics[scale=0.1]{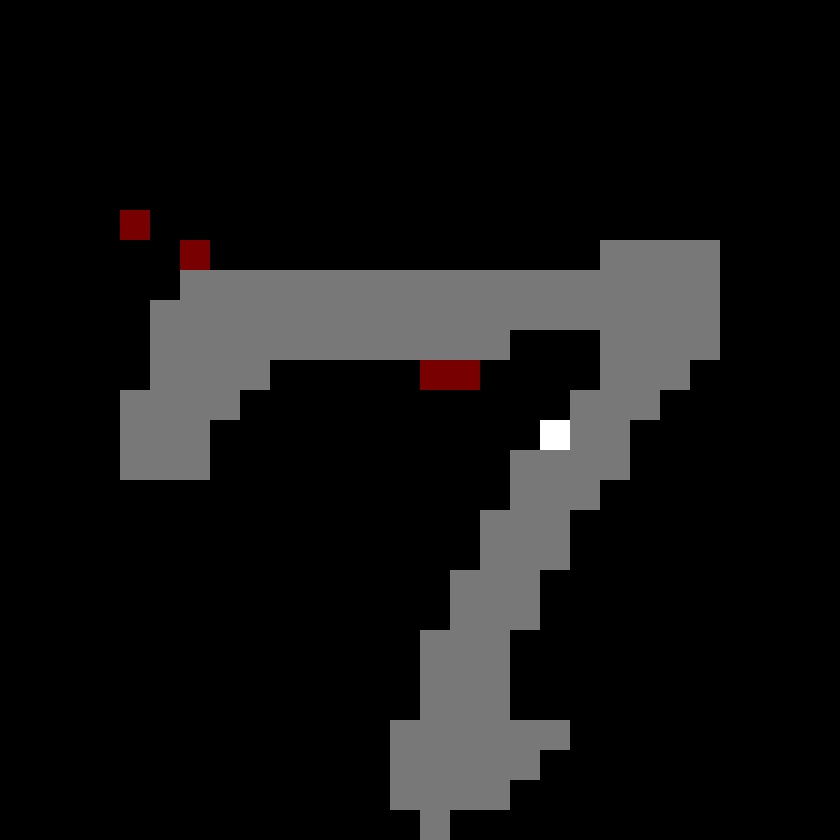}
			\caption{$k^\star = 5$}
		\end{subfigure}  
		\begin{subfigure}{.3\textwidth}
			\centering
			\includegraphics[scale=0.1]{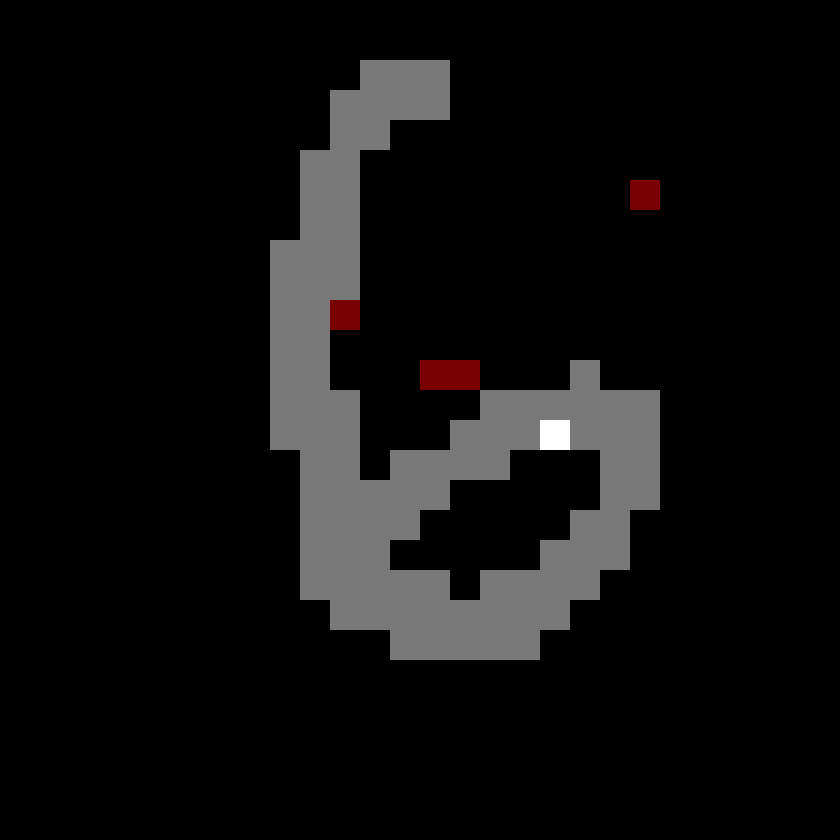}
			\caption{$k^\star = 5$}
		\end{subfigure}
		\begin{subfigure}{.3\textwidth}
			\centering
			\includegraphics[scale=0.1]{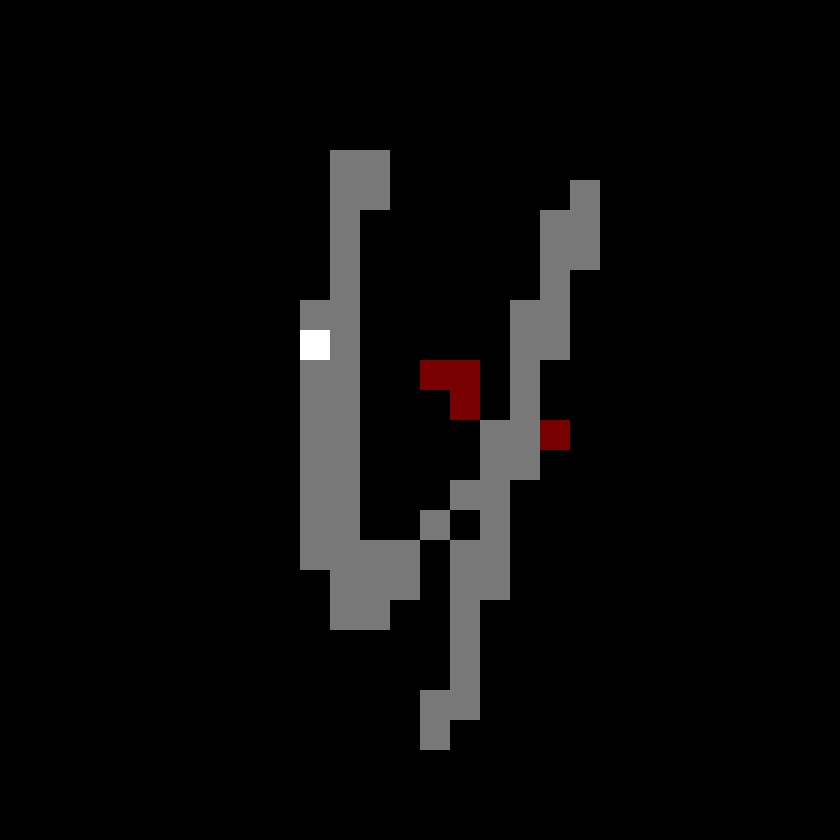}
			\caption{$k^\star = 5$}
		\end{subfigure}

		\begin{subfigure}{.3\textwidth}
			\centering
			\includegraphics[scale=0.1]{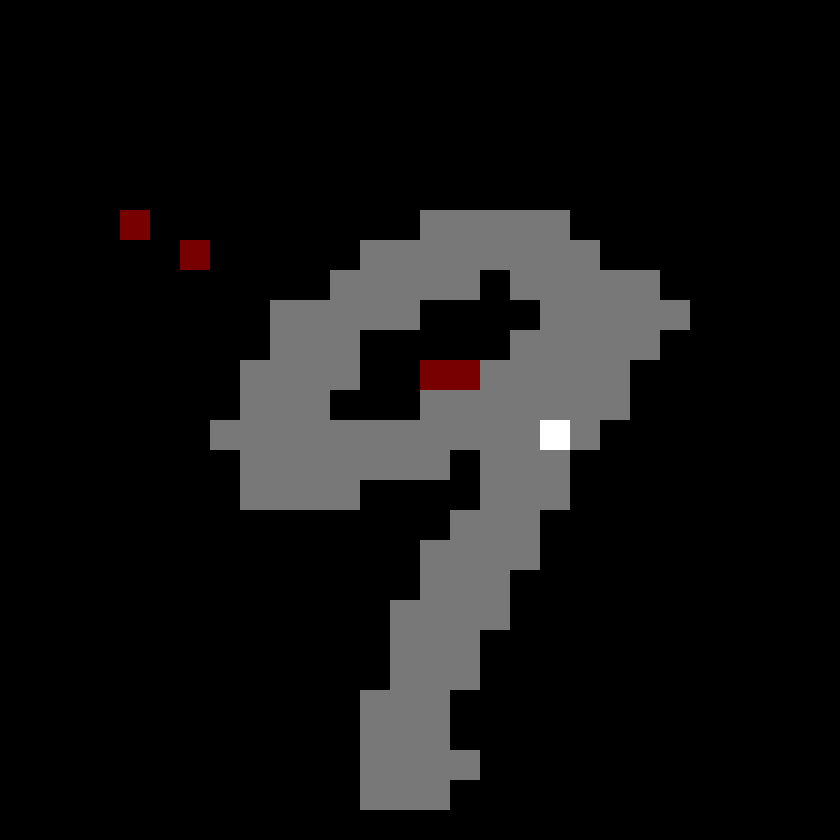}
			\caption{$k^\star = 5$}
		\end{subfigure}
		\begin{subfigure}{.3\textwidth}
			\centering
			\includegraphics[scale=0.1]{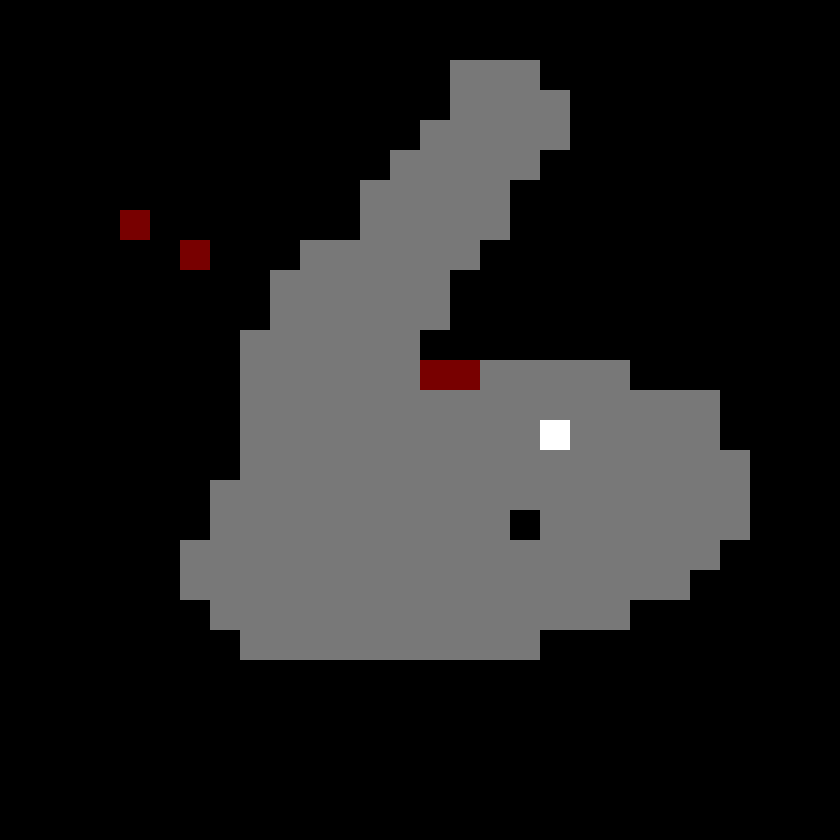}
			\caption{$k^\star = 5$}
		\end{subfigure}
		\begin{subfigure}{.3\textwidth}
			\centering
			\includegraphics[scale=0.1]{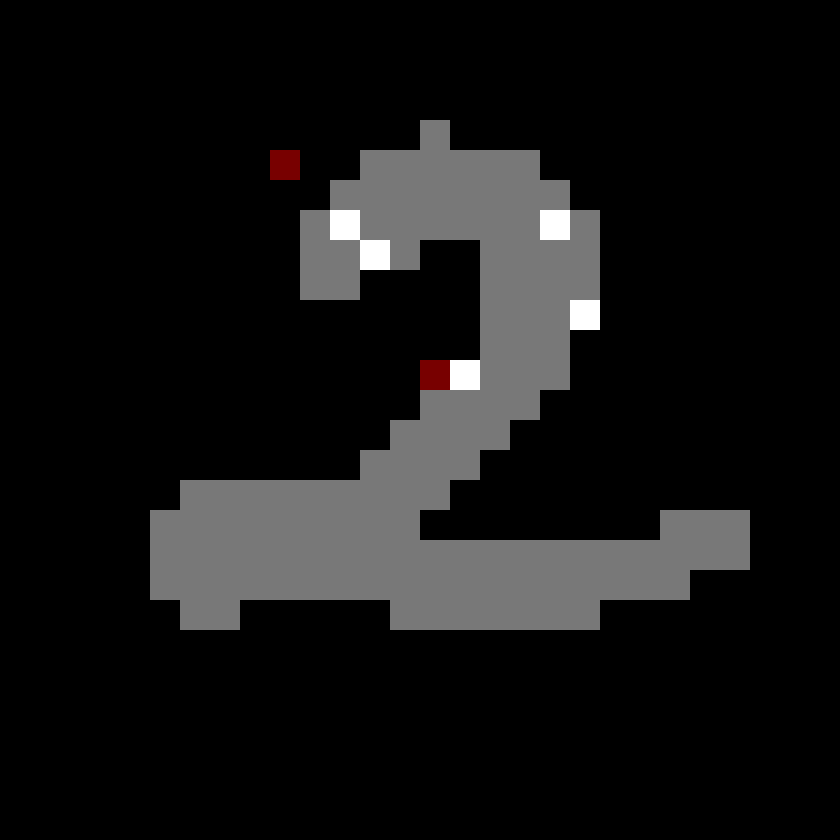}
			\caption{$k^\star = 7$}
		\end{subfigure}
		
		\begin{subfigure}{.3\textwidth}
			\centering
			\includegraphics[scale=0.1]{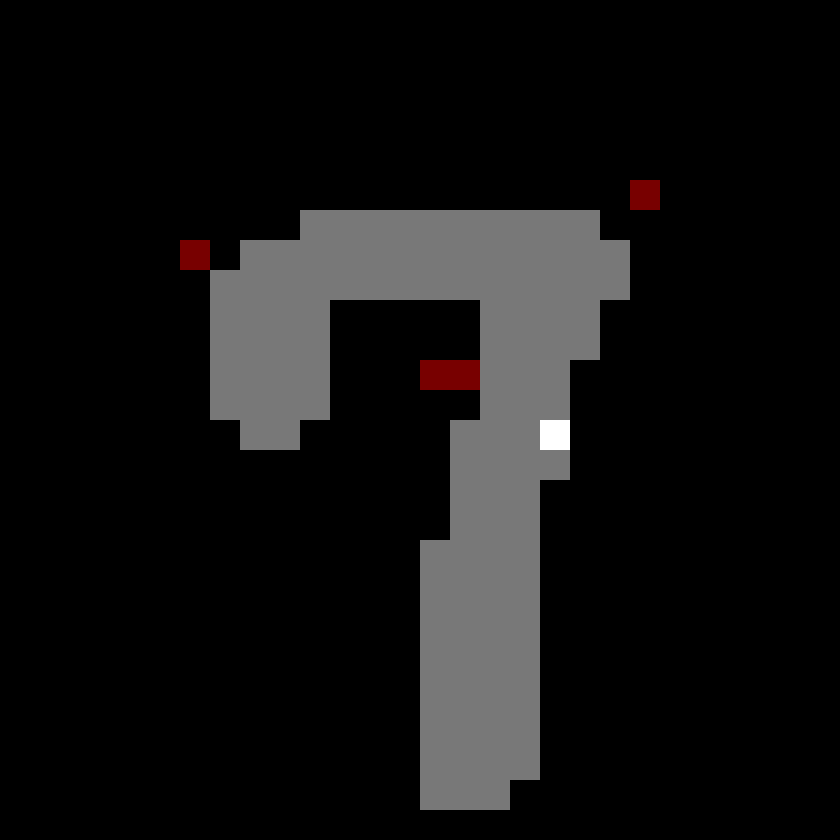}
			\caption{$k^\star = 5$}
		\end{subfigure}
		\begin{subfigure}{.3\textwidth}
			\centering
			\includegraphics[scale=0.1]{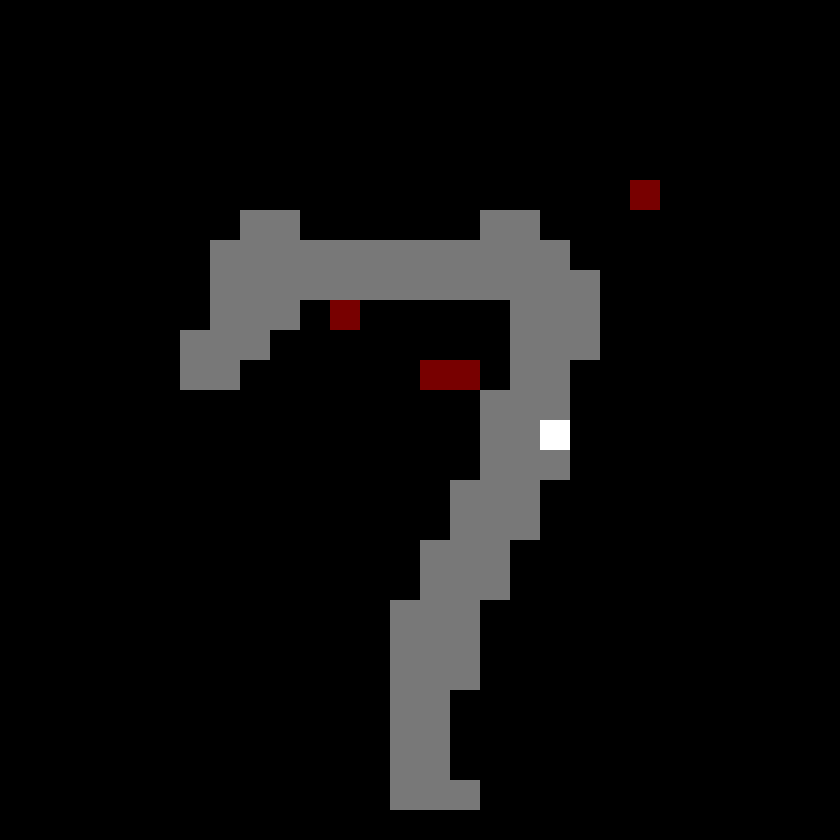}
			\caption{$k^\star = 5$}
		\end{subfigure}
		\begin{subfigure}{.3\textwidth}
			\centering
			\includegraphics[scale=0.1]{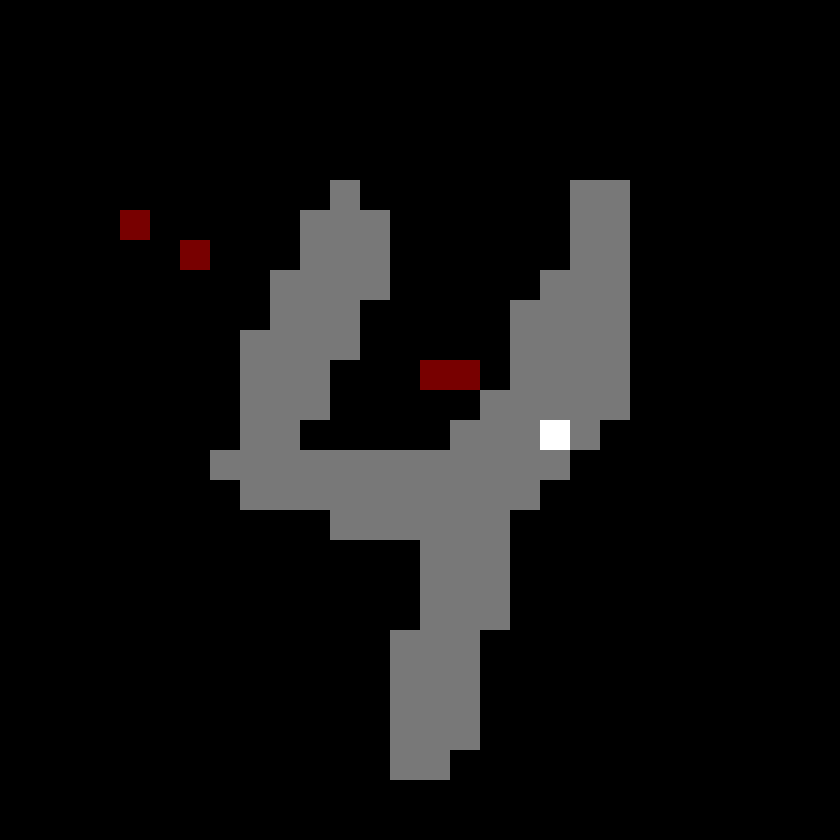}
			\caption{$k^\star = 5$}
		\end{subfigure}
		
		\begin{subfigure}{.3\textwidth}
			\centering
			\includegraphics[scale=0.1]{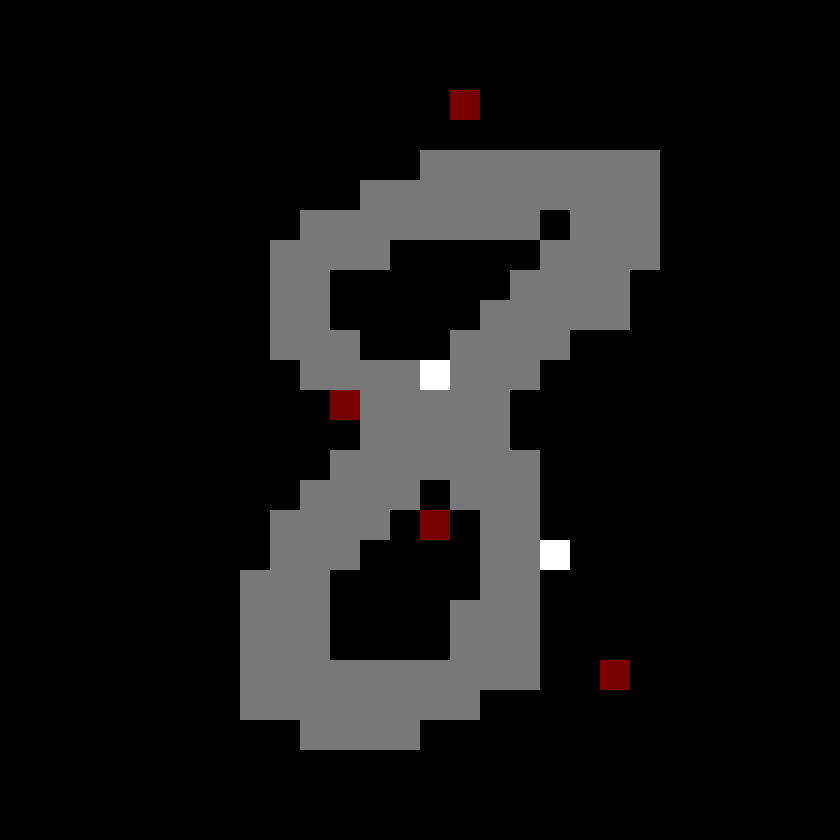}
			\caption{$k^\star = 6$}
		\end{subfigure}
		\begin{subfigure}{.3\textwidth}
			\centering
			\includegraphics[scale=0.1]{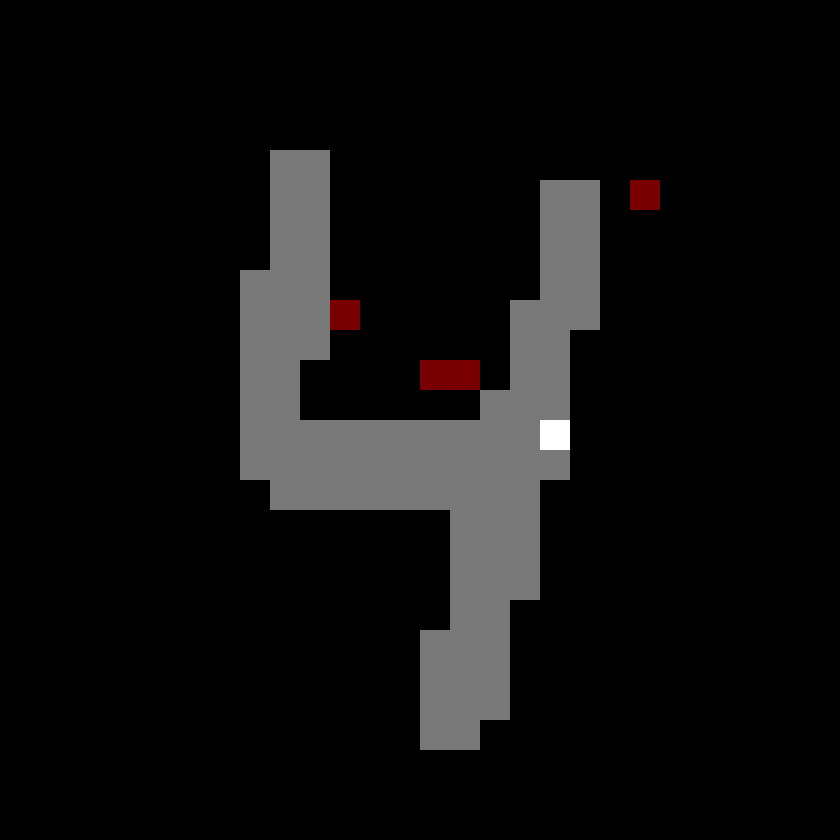}
			\caption{$k^\star = 5$}
		\end{subfigure}
		\begin{subfigure}{.3\textwidth}
			\centering
			\includegraphics[scale=0.1]{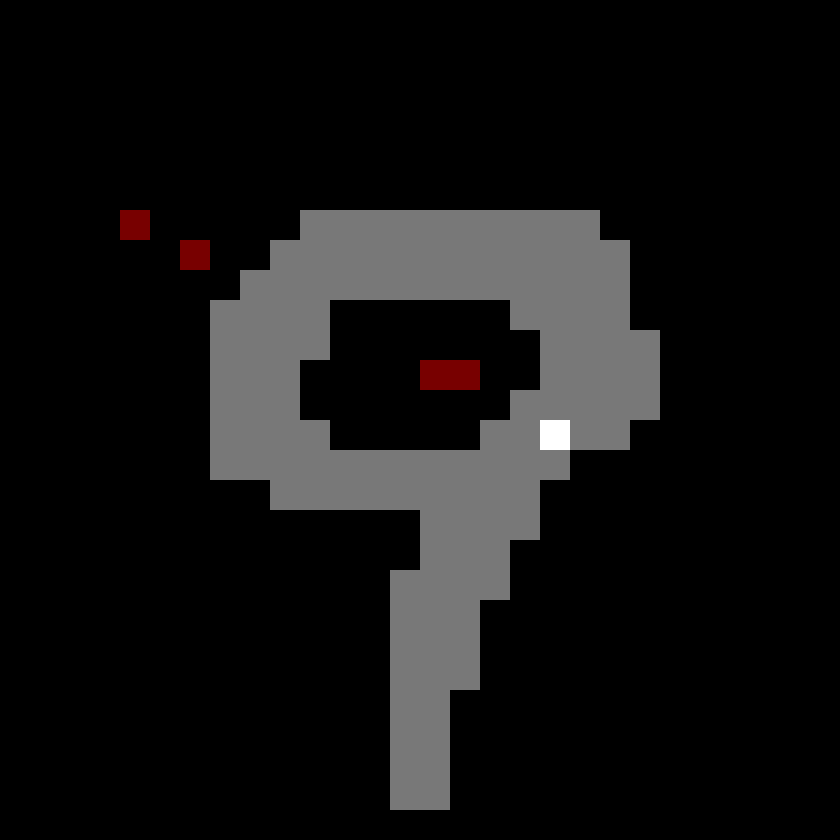}
			\caption{$k^\star = 5$}
		\end{subfigure}

		\caption{Examples of \emph{Minimum Sufficient Reasons} over the MNIST dataset. All instances are correctly predicted negative instances for a decision tree of $591$ leaves that detects the digit $1$.  Light pixels of the original image are depicted in grey, and the light pixels of the original image that are part of the minimum sufficient reason are colored white. Dark pixels that are part of the minimum sufficient reason are colored with red. Individual captions denote the size of the minimum sufficient reasons with $k^\star$.}
		\label{fig:experiments-negative-1}
\end{figure}

\begin{figure}
\centering
		\begin{subfigure}{.3\textwidth}
			\centering 
			\includegraphics[scale=0.1]{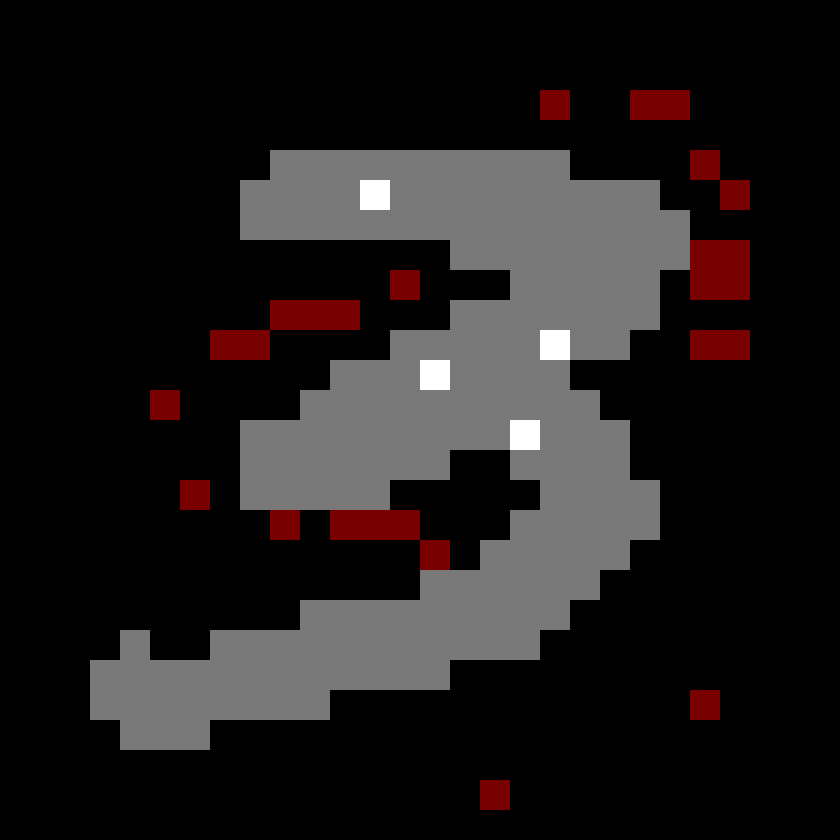}
			\caption{$k^\star = 30$}
		\end{subfigure}  
		\begin{subfigure}{.3\textwidth}
			\centering
			\includegraphics[scale=0.1]{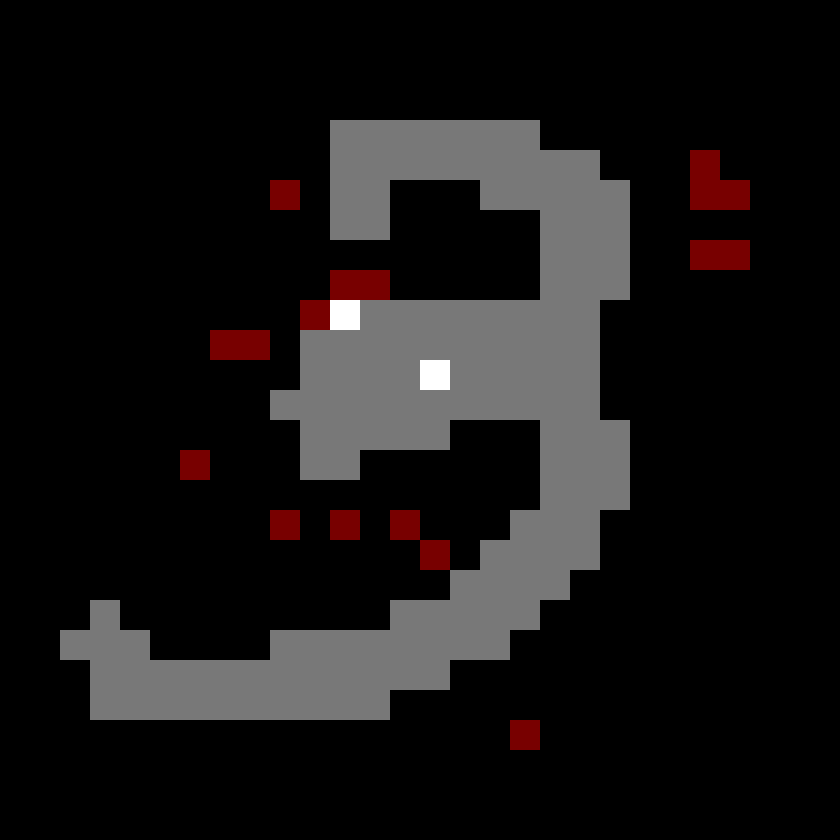}
			\caption{$k^\star = 19$}
		\end{subfigure}
		\begin{subfigure}{.3\textwidth}
			\centering
			\includegraphics[scale=0.1]{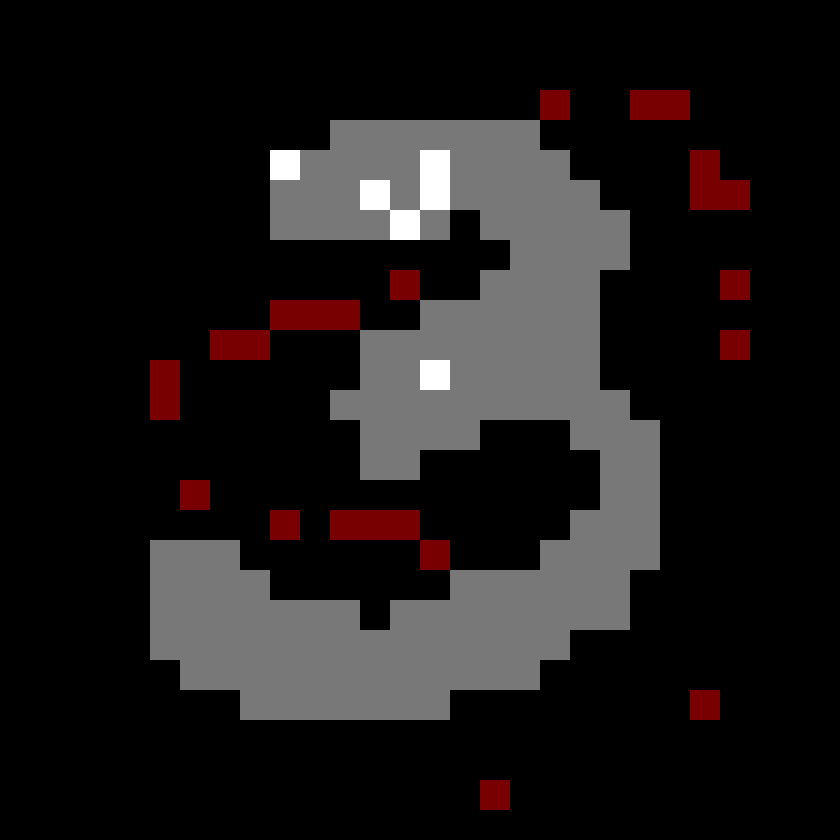}
			\caption{$k^\star = 30$}
		\end{subfigure}

		\begin{subfigure}{.3\textwidth}
			\centering
			\includegraphics[scale=0.1]{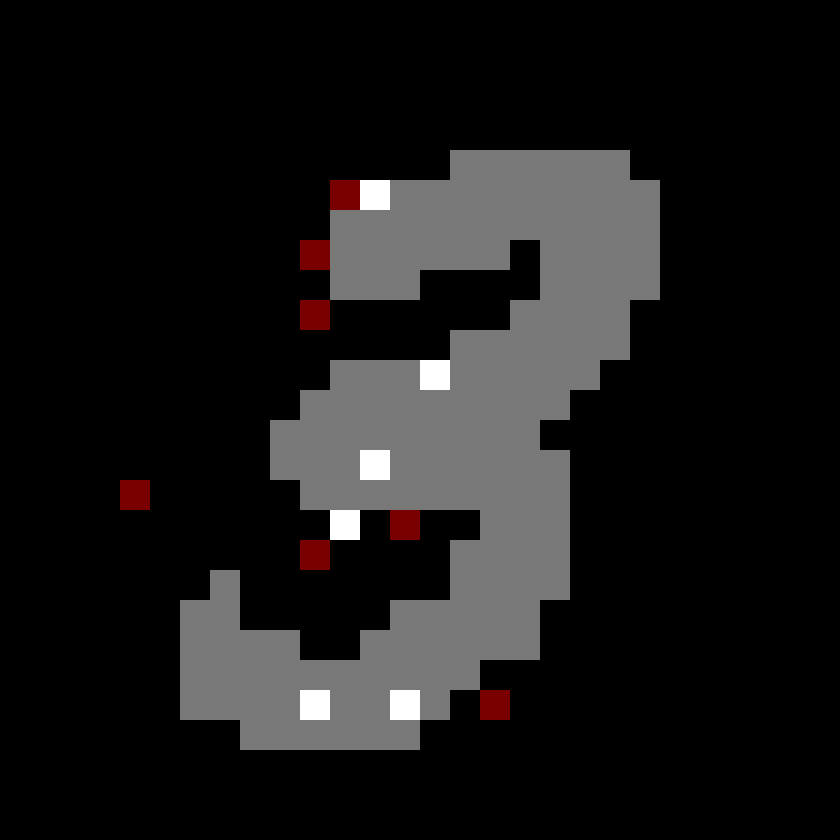}
			\caption{ $k^\star = 13$}
		\end{subfigure}
		\begin{subfigure}{.3\textwidth}
			\centering
			\includegraphics[scale=0.1]{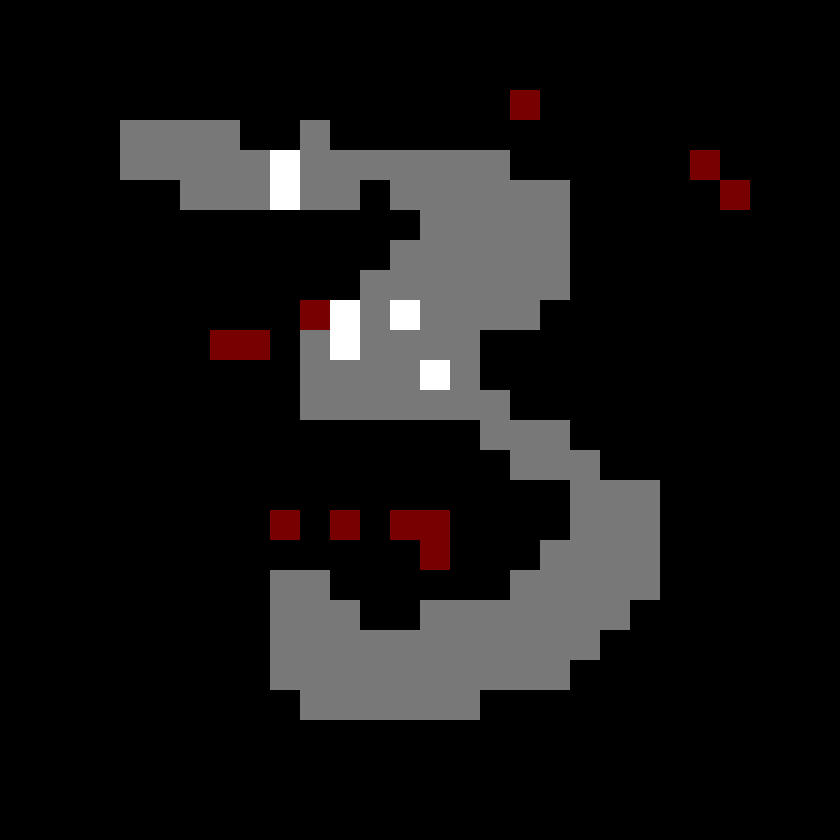}
			\caption{$k^\star = 17$}
		\end{subfigure}
		\begin{subfigure}{.3\textwidth}
			\centering
			\includegraphics[scale=0.1]{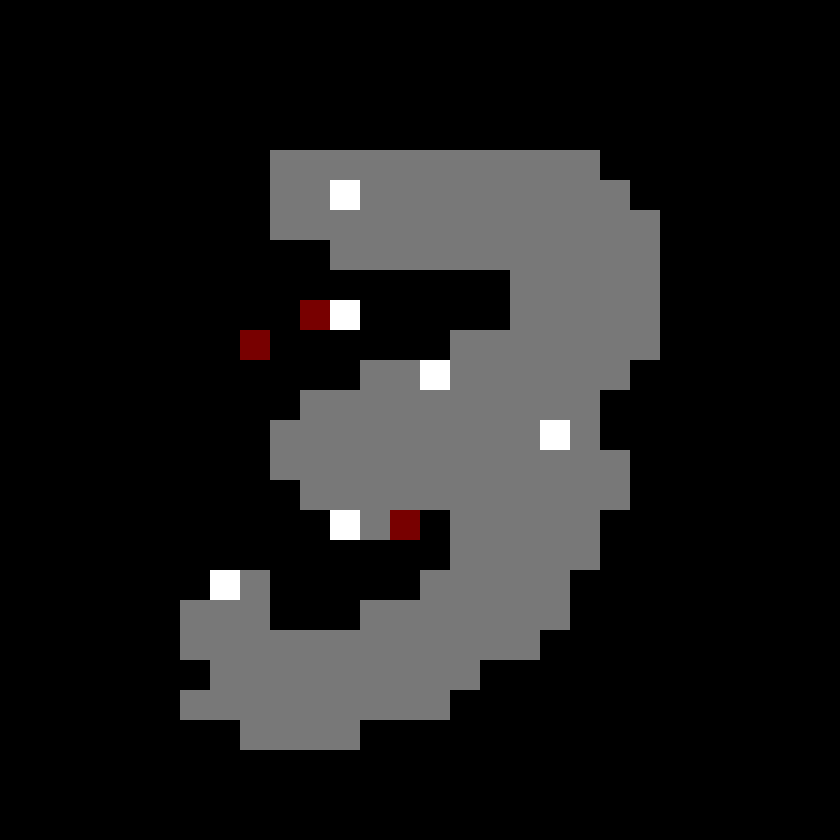}
			\caption{(Misclassified) $k^\star = 9$}
		\end{subfigure}
		
		\begin{subfigure}{.3\textwidth}
			\centering
			\includegraphics[scale=0.1]{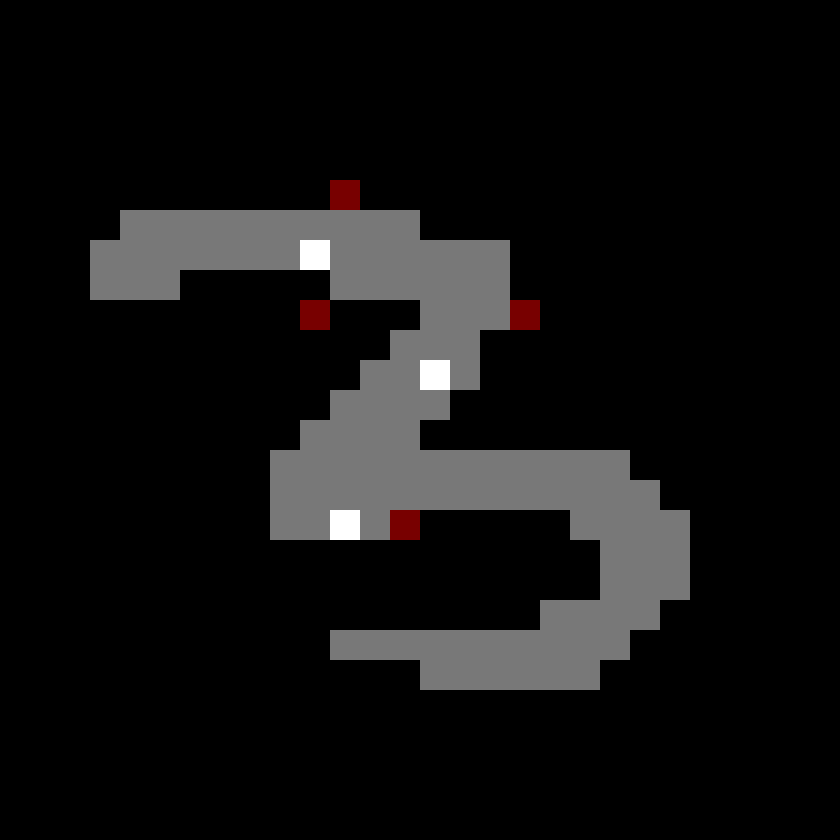}
			\caption{(Misclassified) $k^\star = 7$}
		\end{subfigure}
		\begin{subfigure}{.3\textwidth}
			\centering
			\includegraphics[scale=0.1]{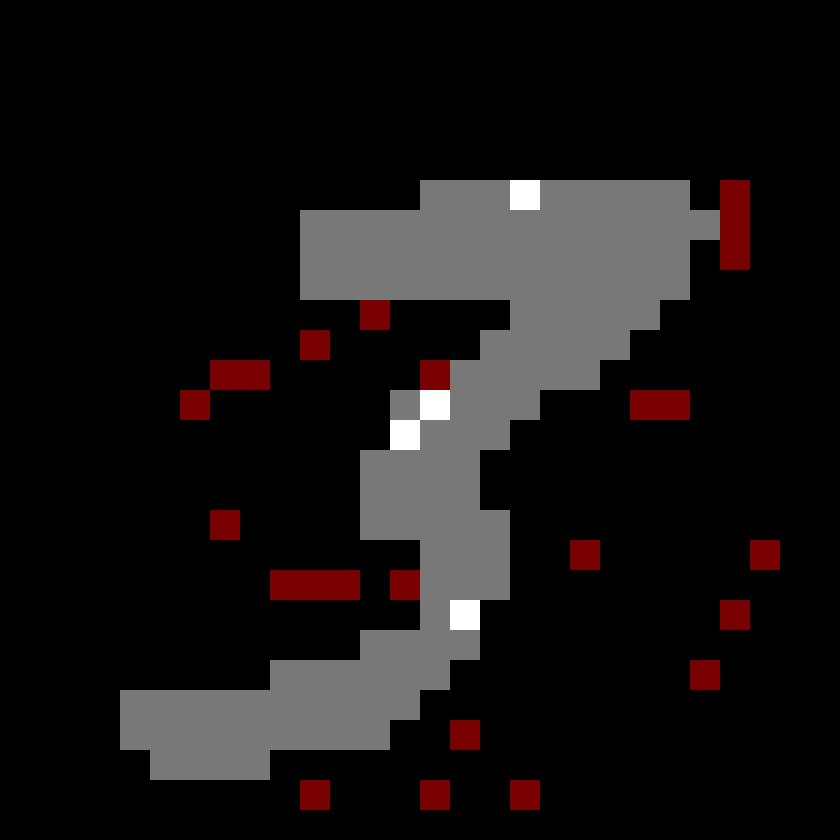}
			\caption{$k^\star = 28$}
		\end{subfigure}
		\begin{subfigure}{.3\textwidth}
			\centering
			\includegraphics[scale=0.1]{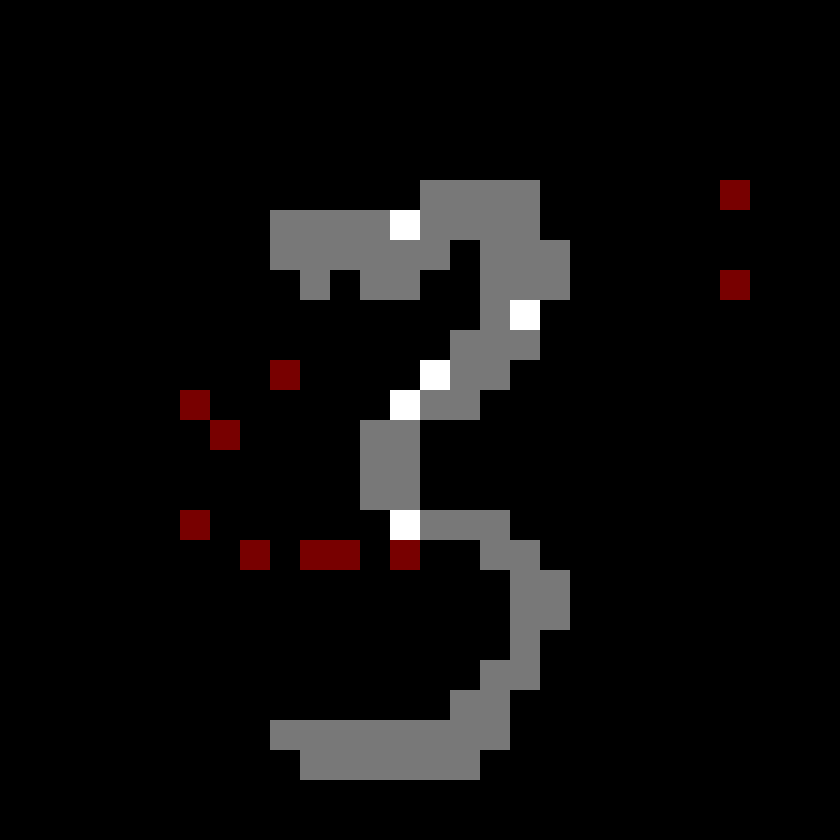}
			\caption{$k^\star = 15$}
		\end{subfigure}
		
		\begin{subfigure}{.3\textwidth}
			\centering
			\includegraphics[scale=0.1]{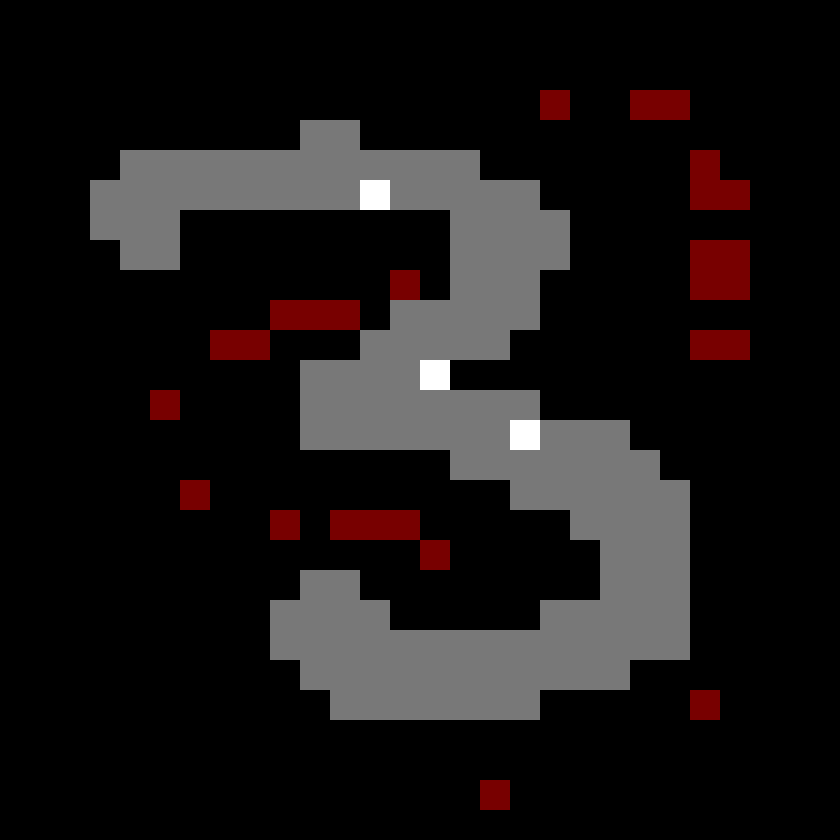}
			\caption{$k^\star = 30$}
		\end{subfigure}
		\begin{subfigure}{.3\textwidth}
			\centering
			\includegraphics[scale=0.1]{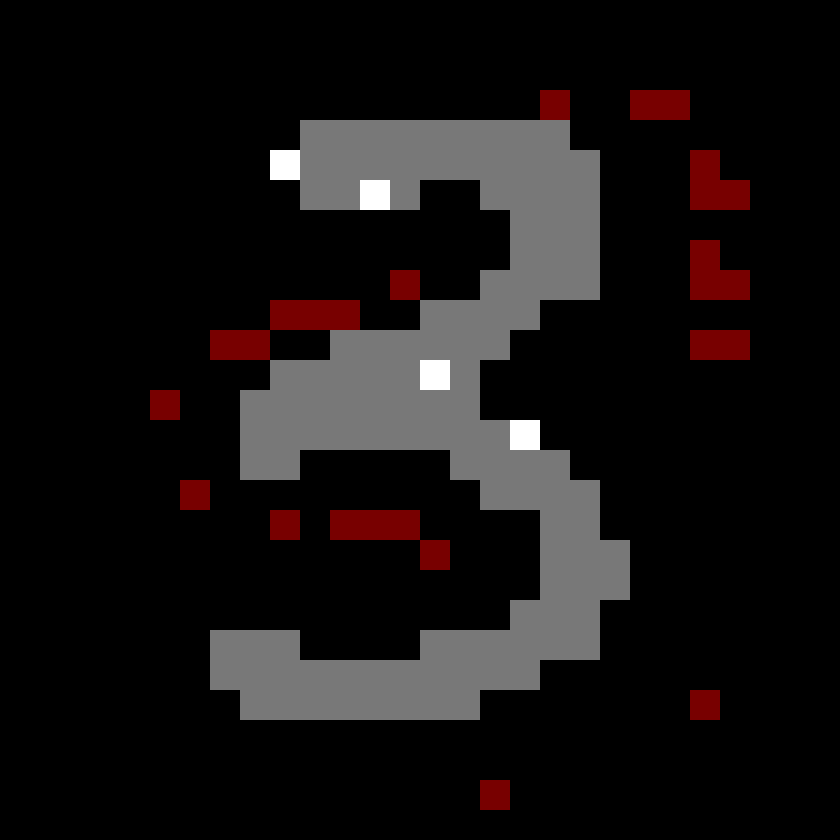}
			\caption{$k^\star = 30$}
		\end{subfigure}
		\begin{subfigure}{.3\textwidth}
			\centering
			\includegraphics[scale=0.1]{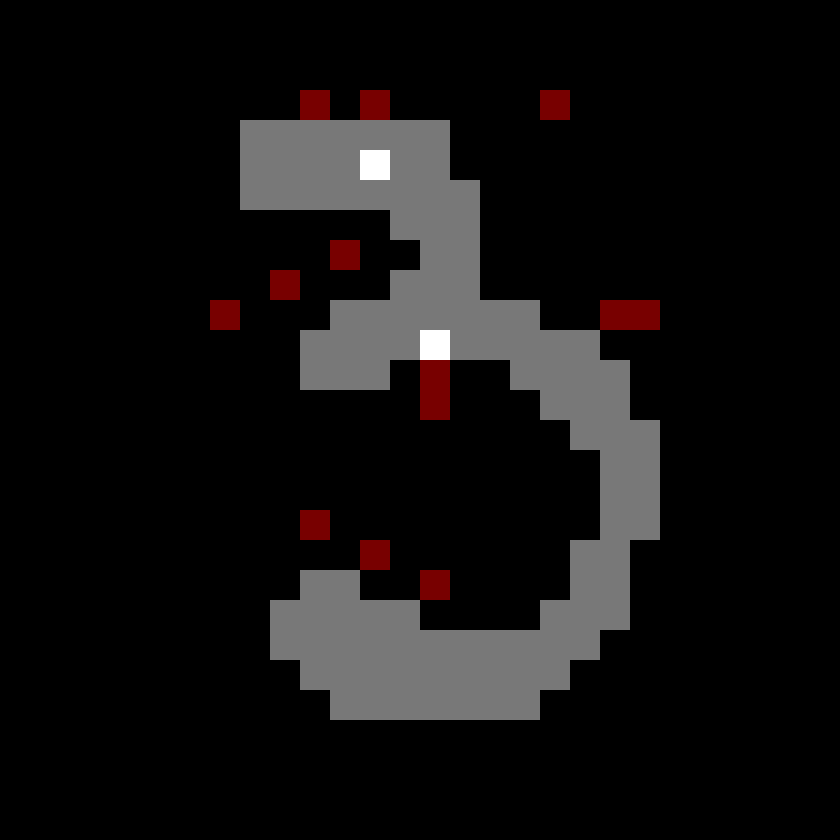}
			\caption{$k^\star = 15$}
		\end{subfigure}

		\caption{Examples of \emph{Minimum Sufficient Reasons} over the MNIST dataset. All images correspond to positive instances for a decision tree of $1486$ leaves that detects the digit $3$. Two instances are misclassified. Light pixels of the original image are depicted in grey, and the light pixels of the original image that are part of the minimum sufficient reason are colored white. Dark pixels that are part of the minimum sufficient reason are colored with red. Individual captions denote the size of the minimum sufficient reasons with $k^\star$.}
		\label{fig:experiments-positive-3}
\end{figure}

\begin{figure}
\centering
		\begin{subfigure}{.3\textwidth}
			\centering 
			\includegraphics[scale=0.1]{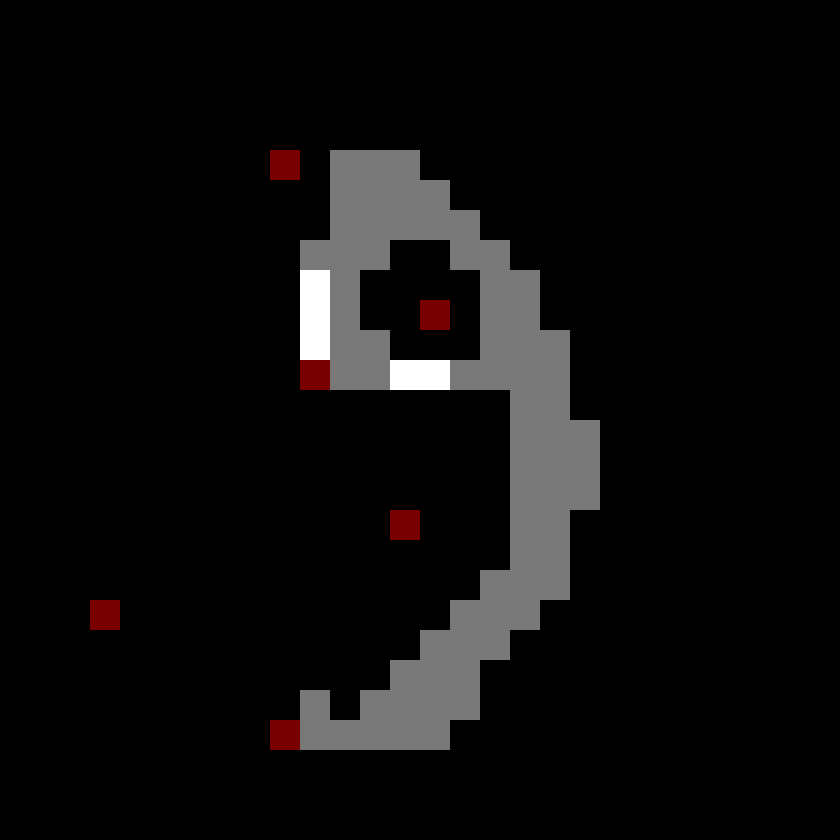}
			\caption{$k^\star = 11$}
		\end{subfigure}  
		\begin{subfigure}{.3\textwidth}
			\centering
			\includegraphics[scale=0.1]{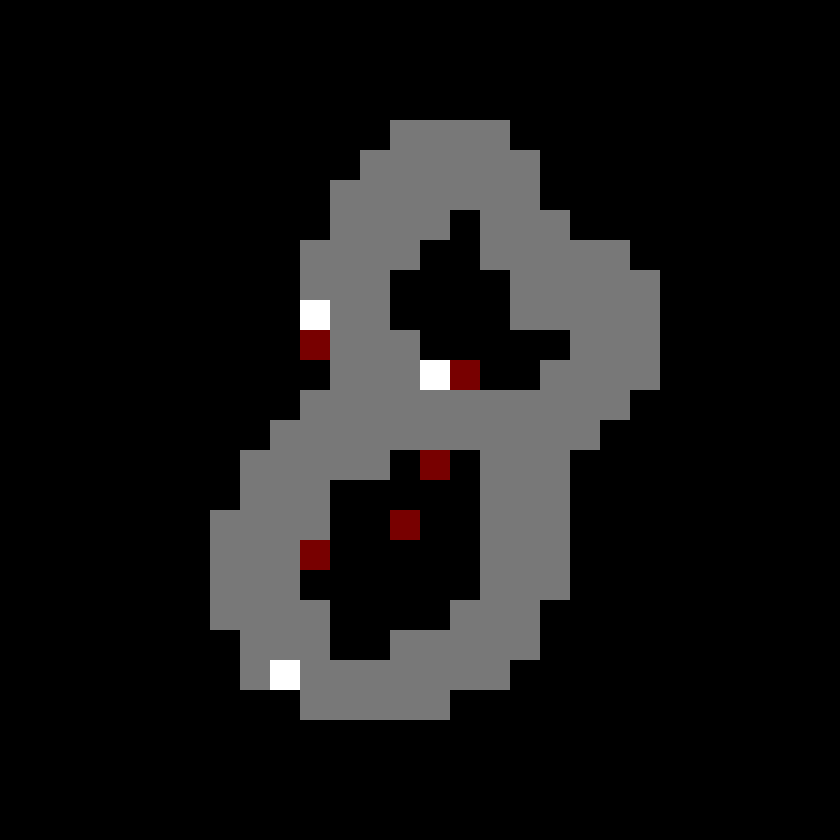}
			\caption{$k^\star = 8$}
		\end{subfigure}
		\begin{subfigure}{.3\textwidth}
			\centering
			\includegraphics[scale=0.1]{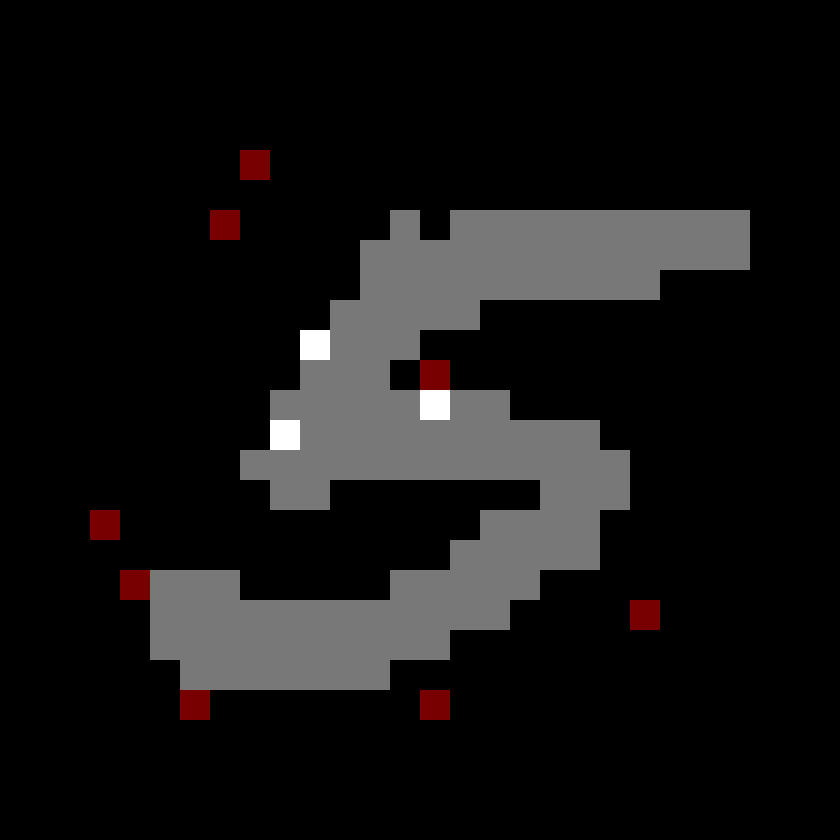}
			\caption{$k^\star = 11$}
		\end{subfigure}

		\begin{subfigure}{.3\textwidth}
			\centering
			\includegraphics[scale=0.1]{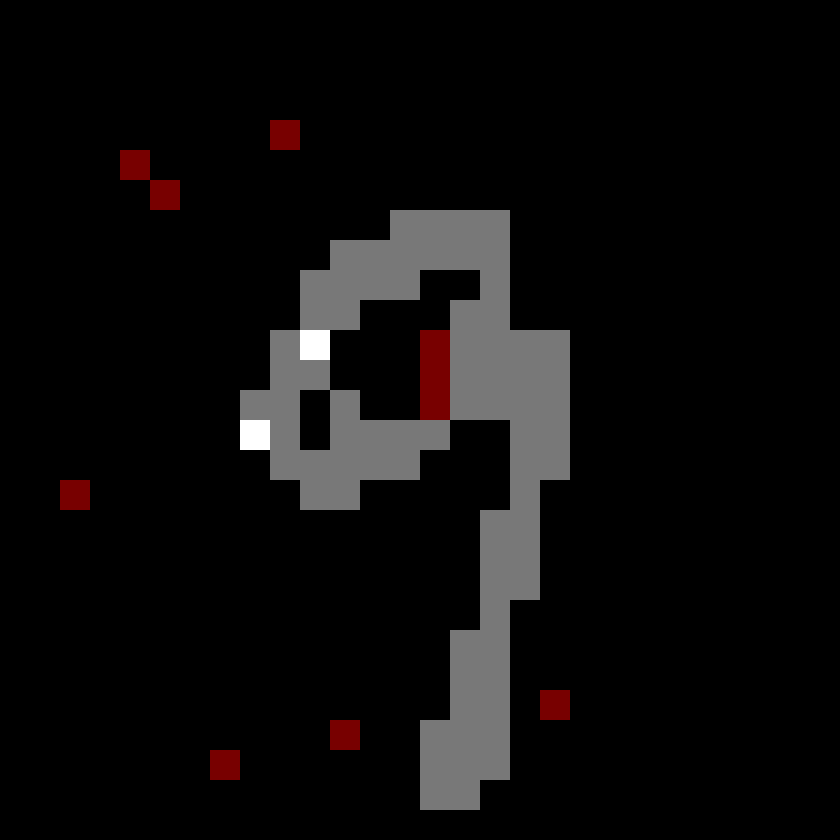}
			\caption{ $k^\star = 12$}
		\end{subfigure}
		\begin{subfigure}{.3\textwidth}
			\centering
			\includegraphics[scale=0.1]{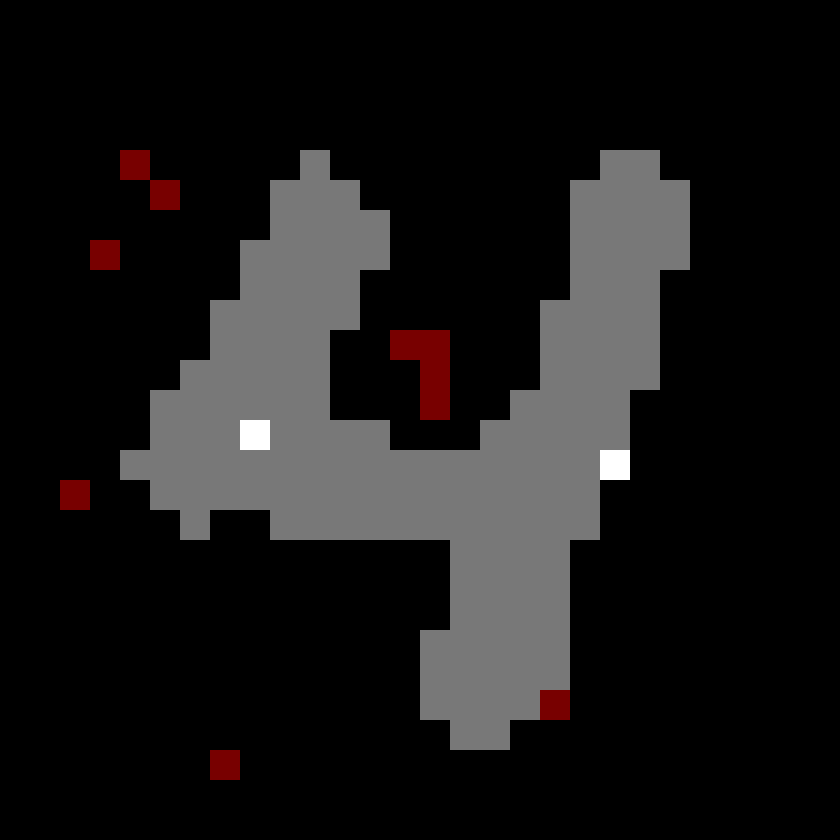}
			\caption{$k^\star = 12$}
		\end{subfigure}
		\begin{subfigure}{.3\textwidth}
			\centering
			\includegraphics[scale=0.1]{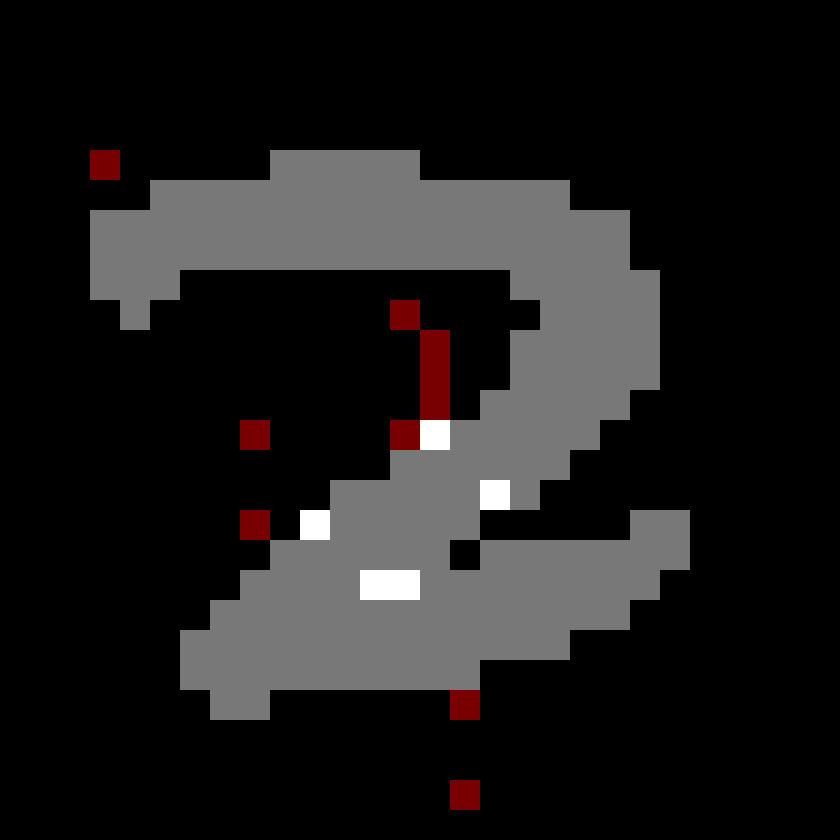}
			\caption{$k^\star = 15$}
		\end{subfigure}
		
		\begin{subfigure}{.3\textwidth}
			\centering
			\includegraphics[scale=0.1]{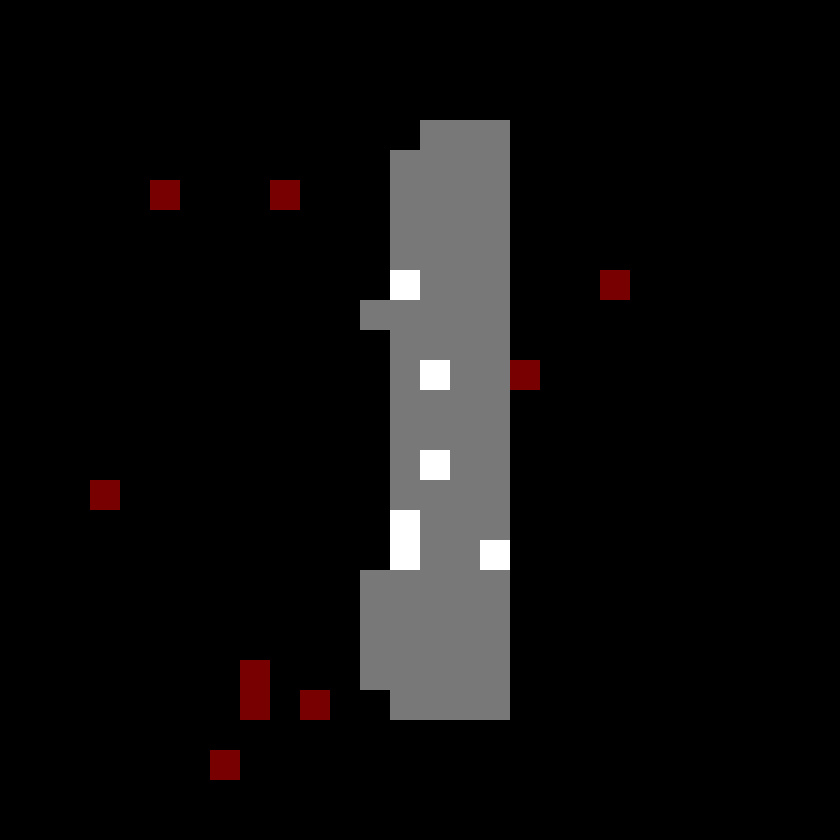}
			\caption{($k^\star = 15$}
		\end{subfigure}
		\begin{subfigure}{.3\textwidth}
			\centering
			\includegraphics[scale=0.1]{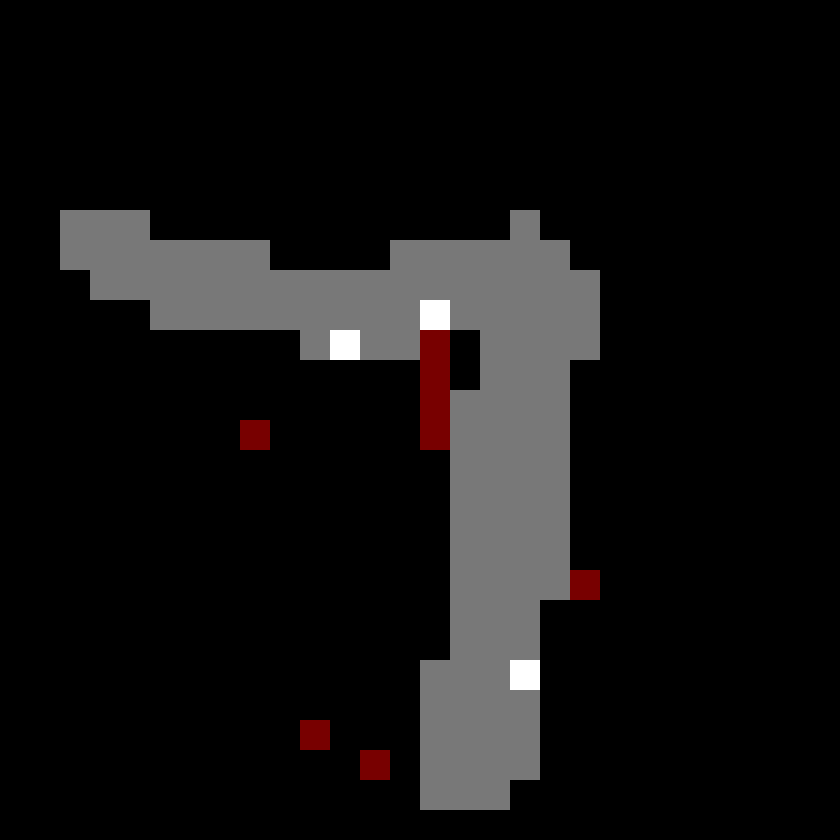}
			\caption{$k^\star = 11$}
		\end{subfigure}
		\begin{subfigure}{.3\textwidth}
			\centering
			\includegraphics[scale=0.1]{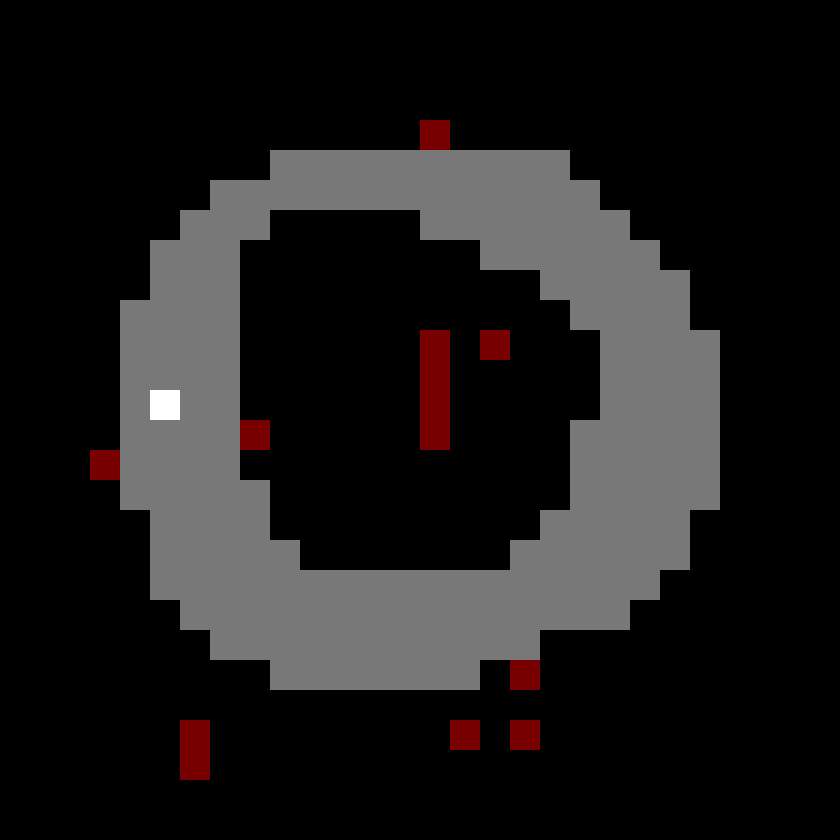}
			\caption{$k^\star = 14$}
		\end{subfigure}
		
		\begin{subfigure}{.3\textwidth}
			\centering
			\includegraphics[scale=0.1]{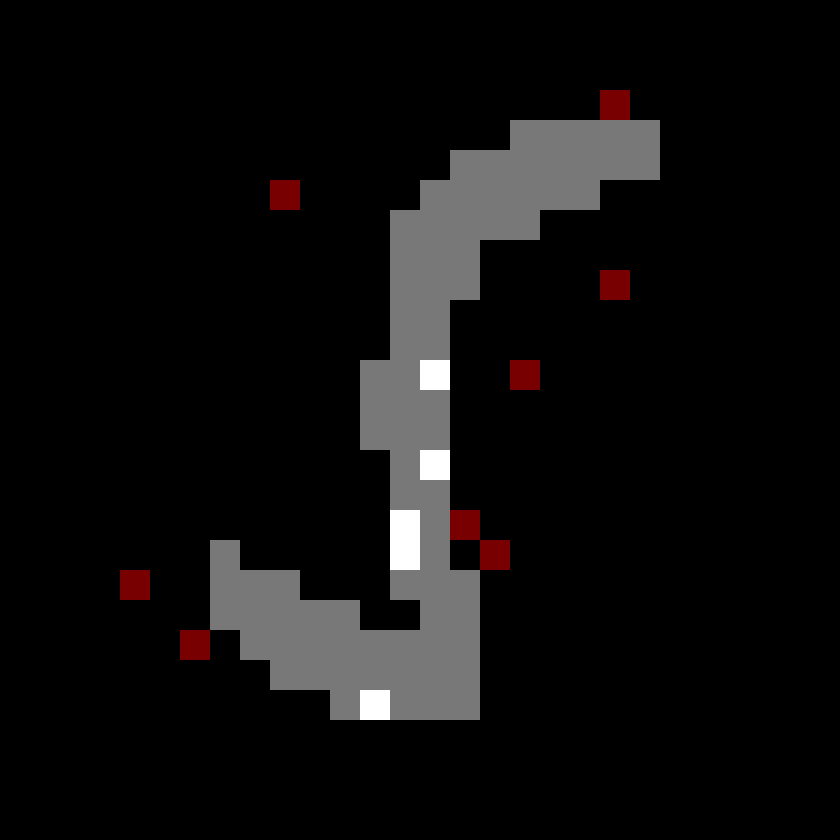}
			\caption{$k^\star = 13$}
		\end{subfigure}
		\begin{subfigure}{.3\textwidth}
			\centering
			\includegraphics[scale=0.1]{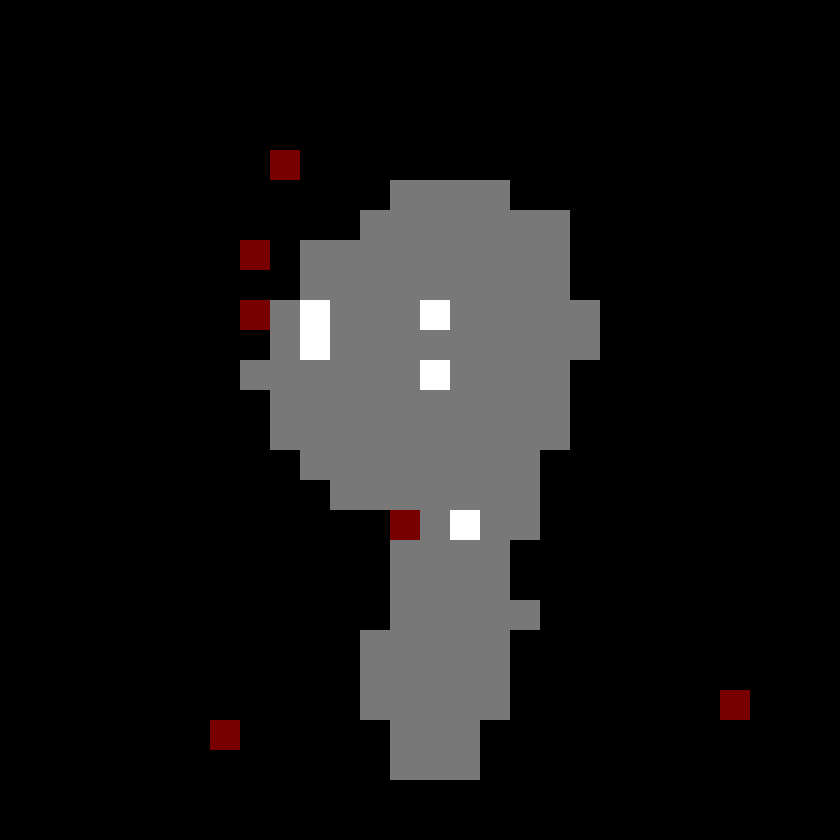}
			\caption{$k^\star = 11$}
		\end{subfigure}
		\begin{subfigure}{.3\textwidth}
			\centering
			\includegraphics[scale=0.1]{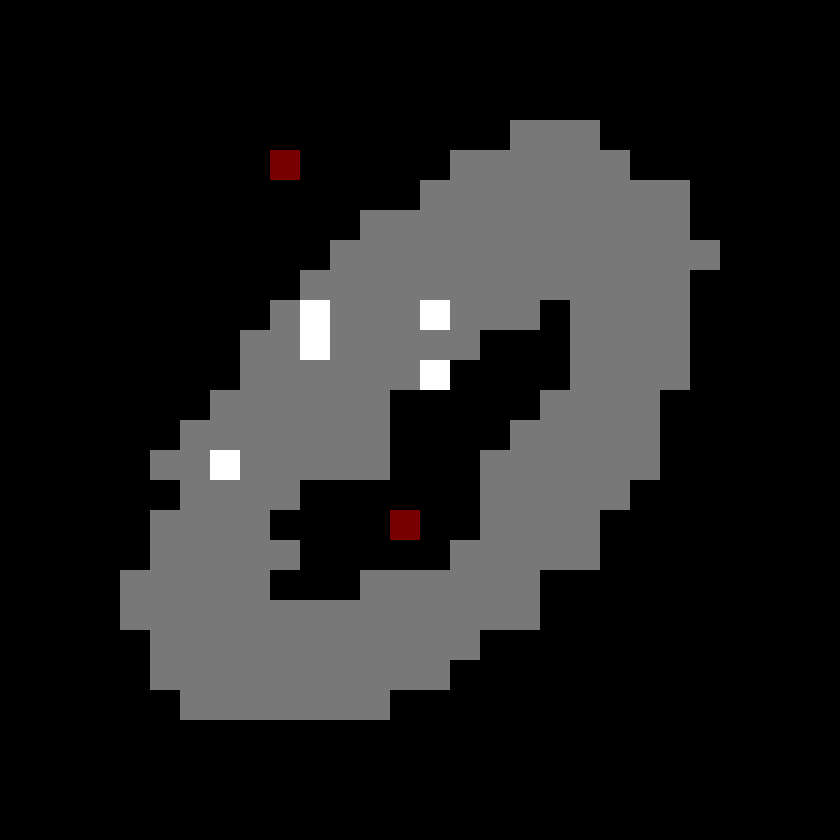}
			\caption{$k^\star = 7$}
		\end{subfigure}

		\caption{Examples of \emph{Minimum Sufficient Reasons} over the MNIST dataset. All images correspond to negative instances for a decision tree of $1486$ leaves that detects the digit $3$ and are classified correctly.  Light pixels of the original image are depicted in grey, and the light pixels of the original image that are part of the minimum sufficient reason are colored white. Dark pixels that are part of the minimum sufficient reason are colored with red. Individual captions denote the size of the minimum sufficient reasons with $k^\star$.}
		\label{fig:experiments-negative-3}
\end{figure}

\begin{figure}
\centering
		\begin{subfigure}{.3\textwidth}
			\centering 
			\includegraphics[scale=0.1]{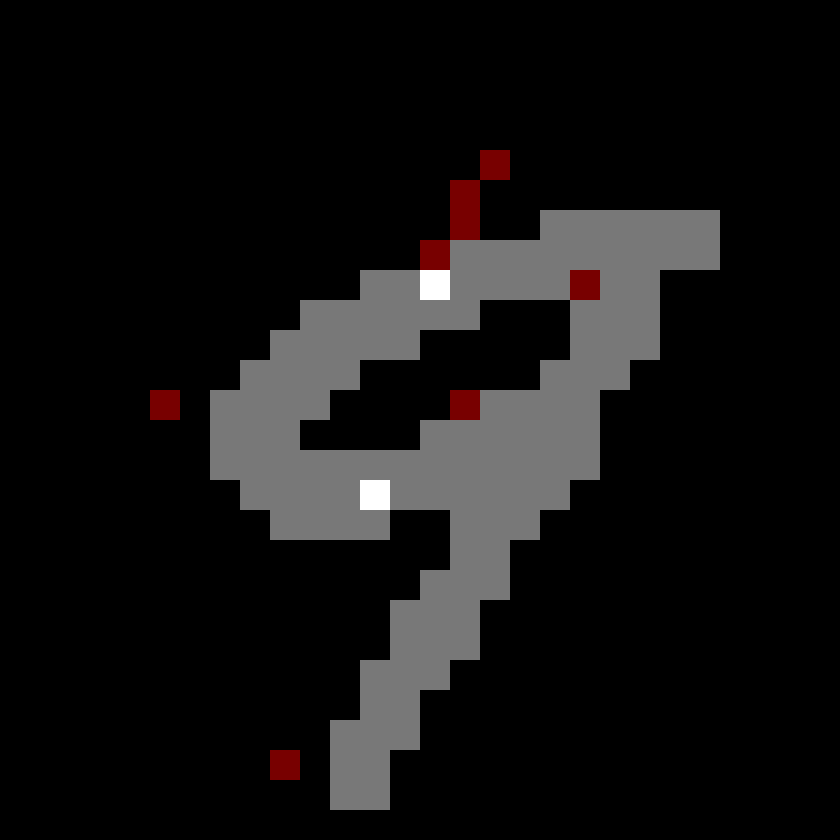}
			\caption{(Misclassified) $k^\star = 10$}
		\end{subfigure}  
		\begin{subfigure}{.3\textwidth}
			\centering
			\includegraphics[scale=0.1]{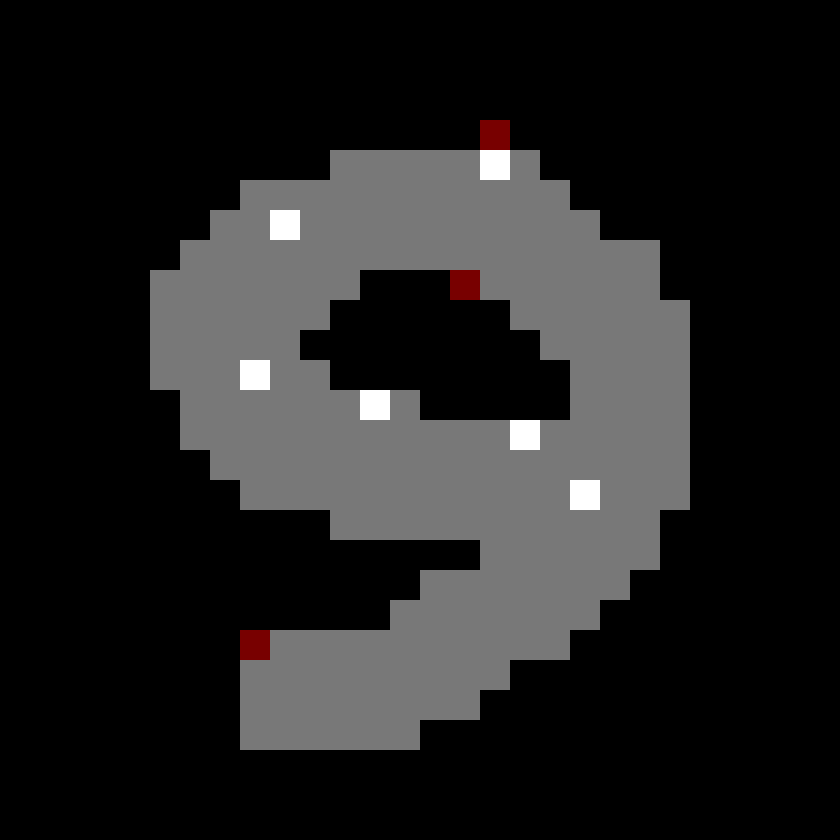}
			\caption{(Misclassified) $k^\star = 9$}
		\end{subfigure}
		\begin{subfigure}{.3\textwidth}
			\centering
			\includegraphics[scale=0.1]{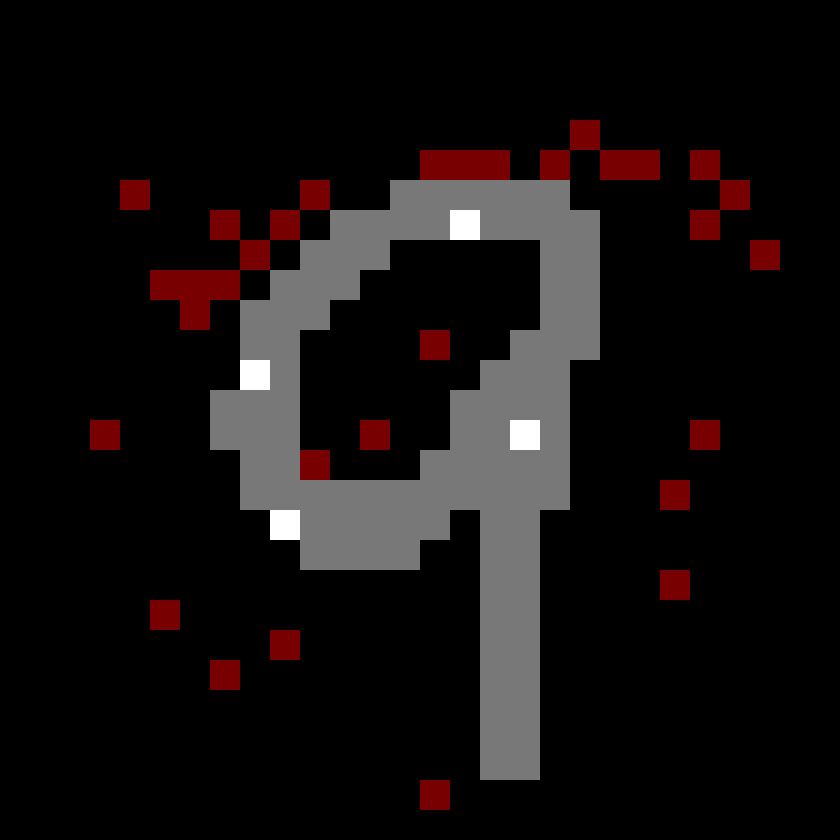}
			\caption{$k^\star = 35$}
		\end{subfigure}

		\begin{subfigure}{.3\textwidth}
			\centering
			\includegraphics[scale=0.1]{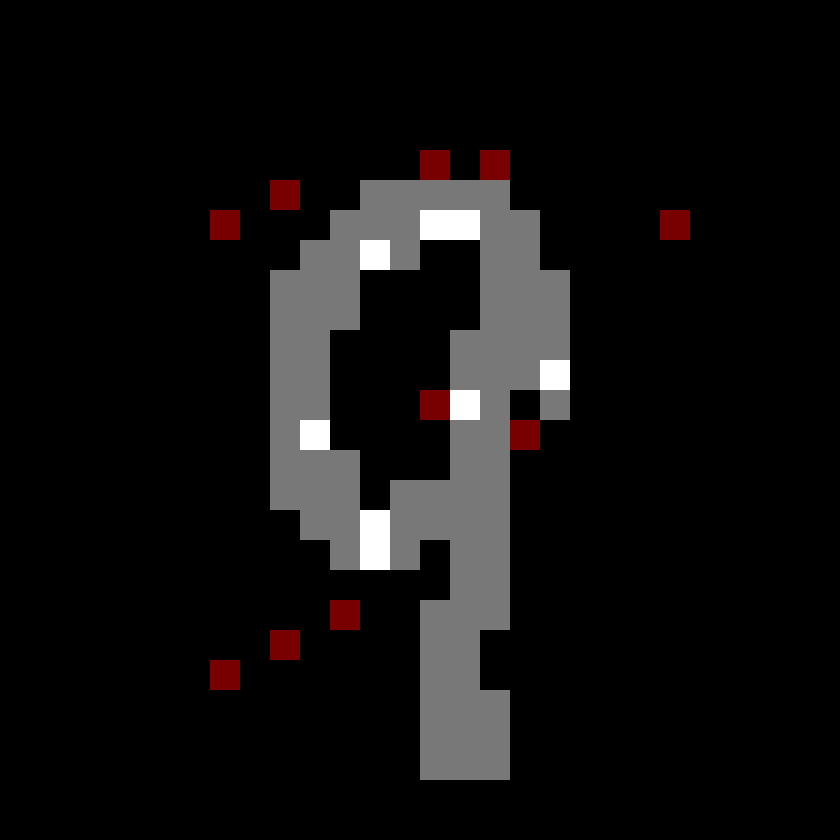}
			\caption{ $k^\star = 18$}
		\end{subfigure}
		\begin{subfigure}{.3\textwidth}
			\centering
			\includegraphics[scale=0.1]{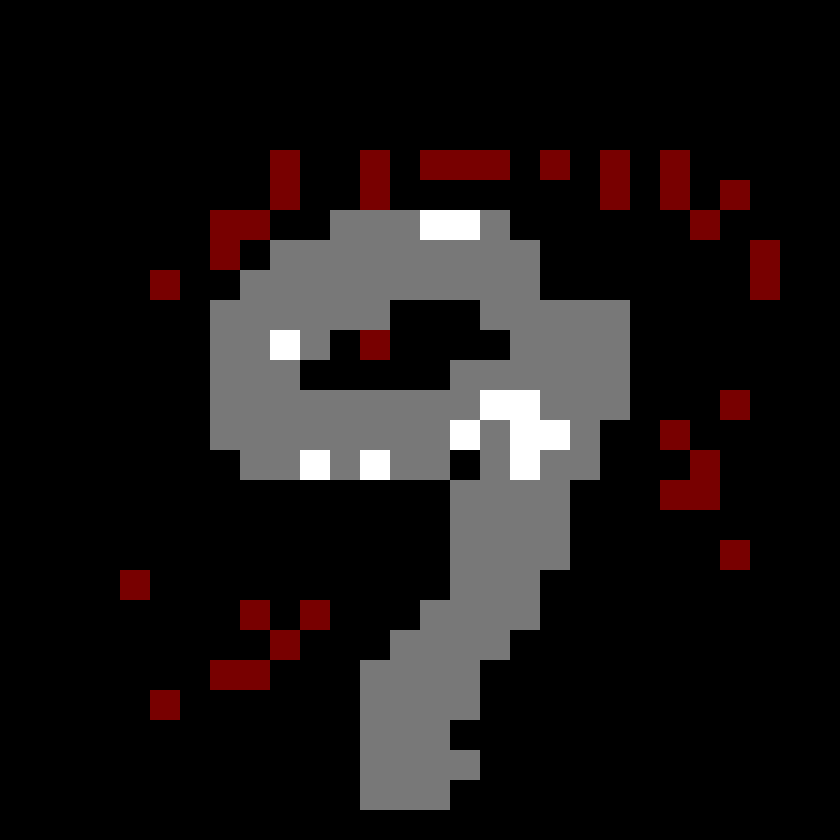}
			\caption{$k^\star = 45$}
		\end{subfigure}
		\begin{subfigure}{.3\textwidth}
			\centering
			\includegraphics[scale=0.1]{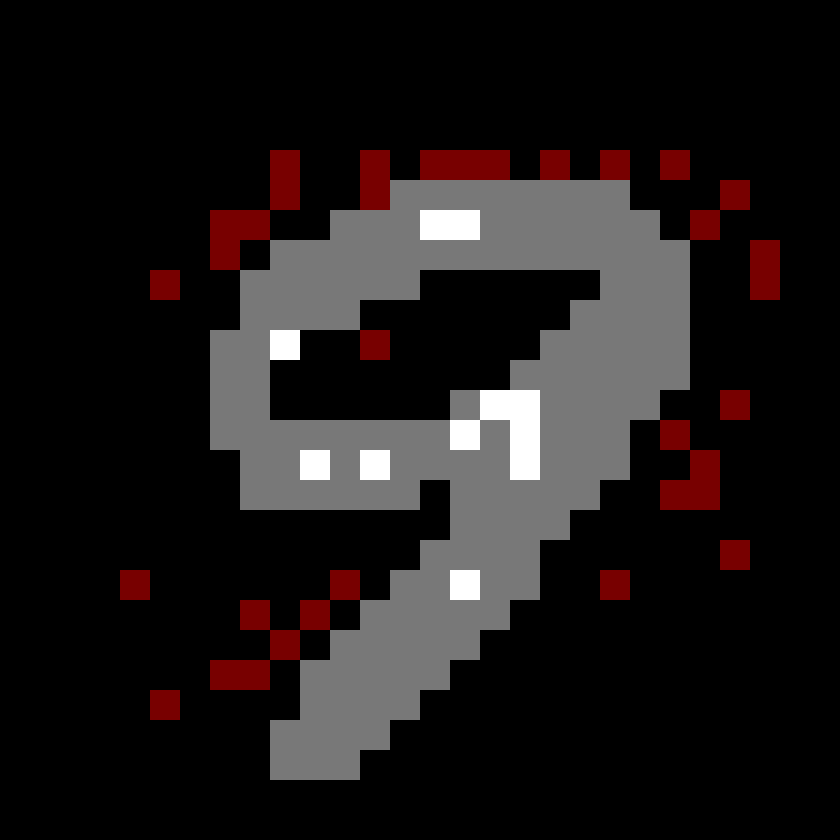}
			\caption{($k^\star = 45$}
		\end{subfigure}
		
		\begin{subfigure}{.3\textwidth}
			\centering
			\includegraphics[scale=0.1]{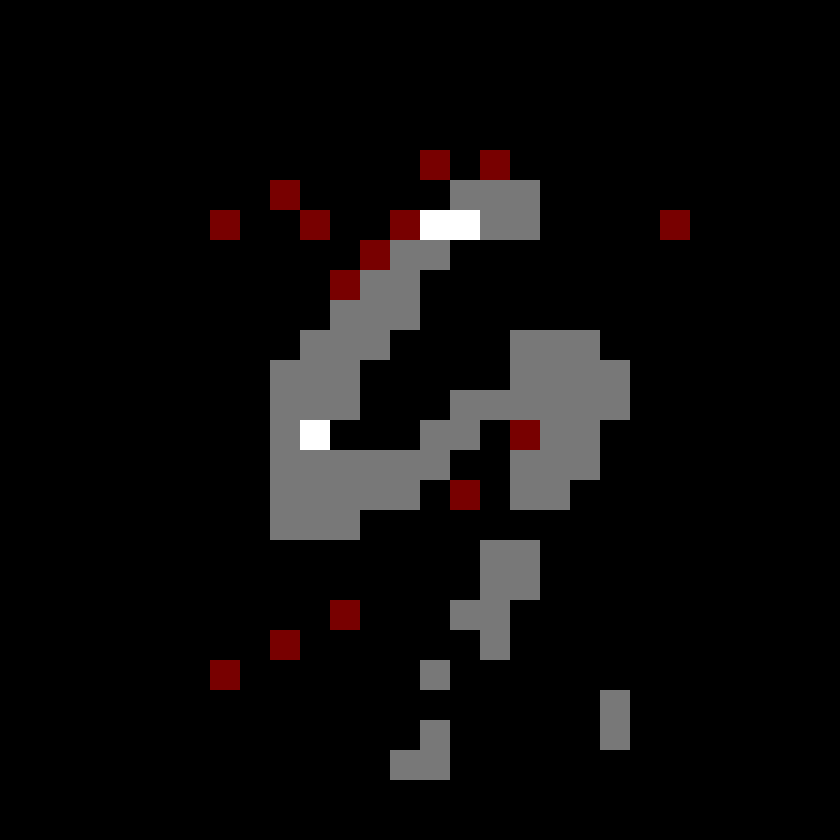}
			\caption{(Misclassified) $k^\star = 17$}
		\end{subfigure}
		\begin{subfigure}{.3\textwidth}
			\centering
			\includegraphics[scale=0.1]{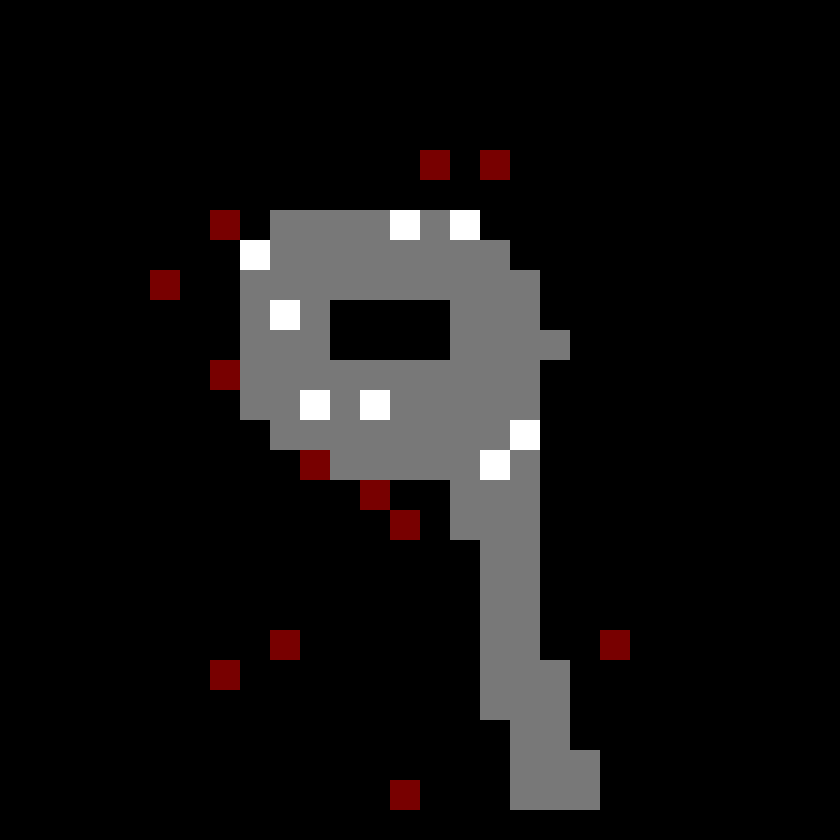}
			\caption{$k^\star = 20$}
		\end{subfigure}
		\begin{subfigure}{.3\textwidth}
			\centering
			\includegraphics[scale=0.1]{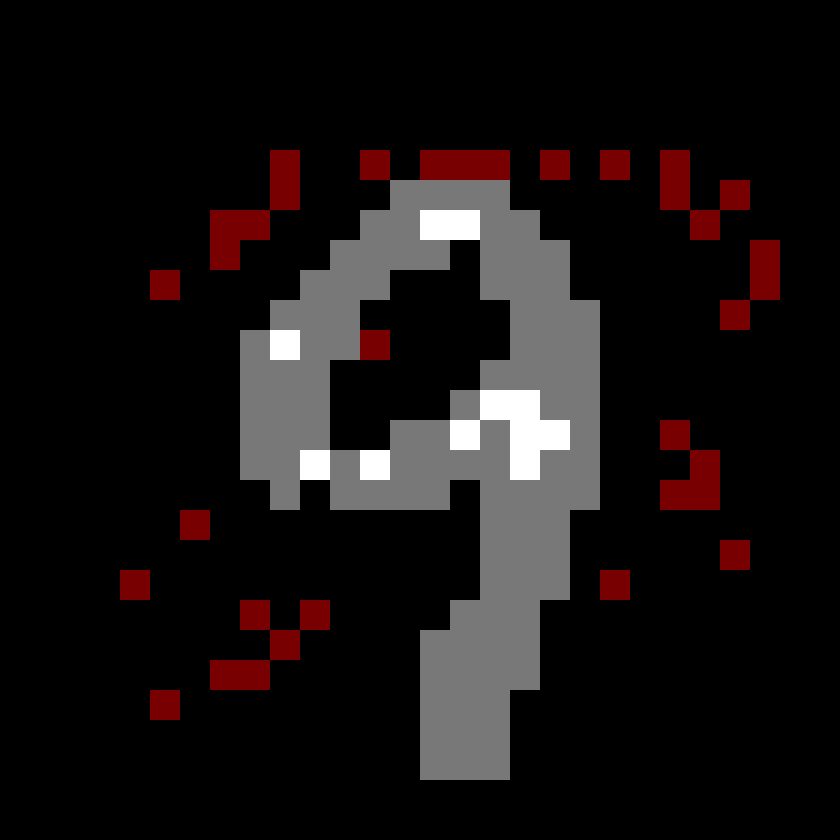}
			\caption{$k^\star = 45$}
		\end{subfigure}
		
		\begin{subfigure}{.3\textwidth}
			\centering
			\includegraphics[scale=0.1]{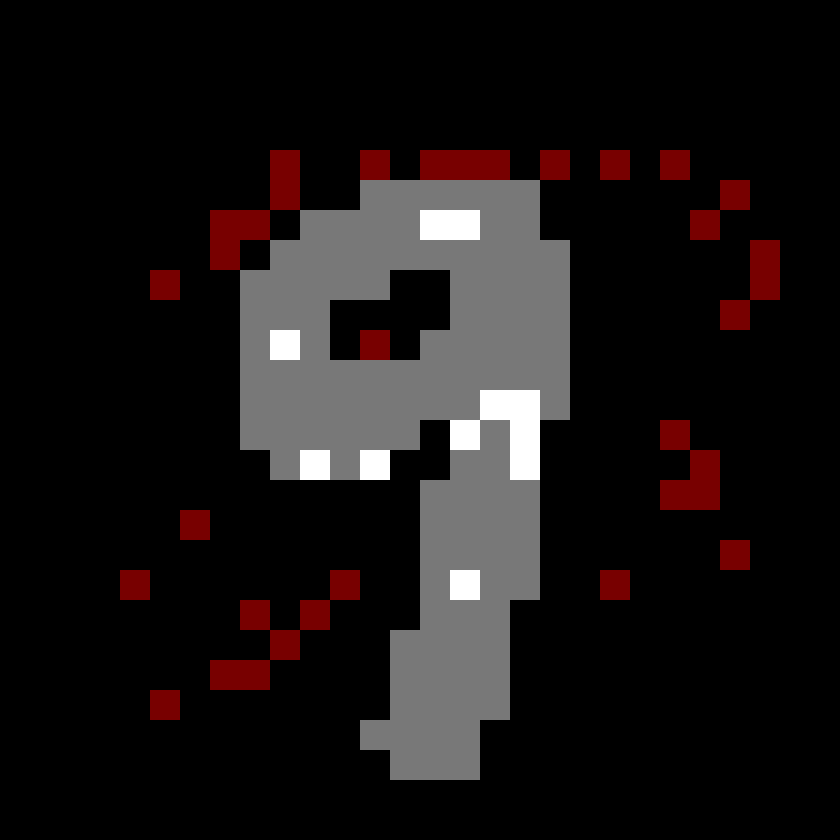}
			\caption{$k^\star = 45$}
		\end{subfigure}
		\begin{subfigure}{.3\textwidth}
			\centering
			\includegraphics[scale=0.1]{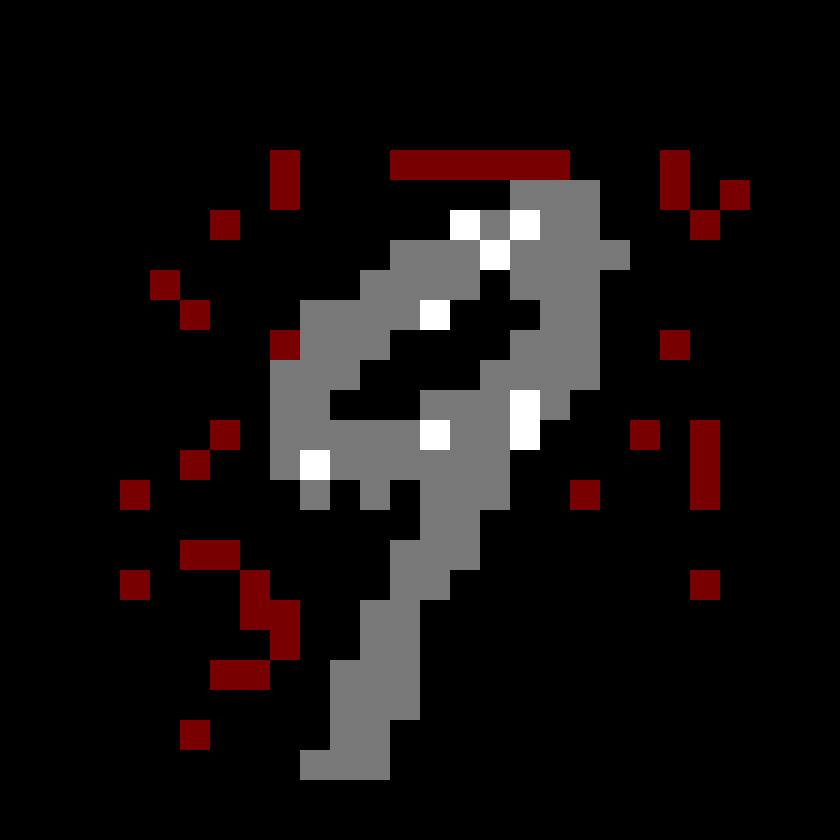}
			\caption{$k^\star = 44$}
		\end{subfigure}
		\begin{subfigure}{.3\textwidth}
			\centering
			\includegraphics[scale=0.1]{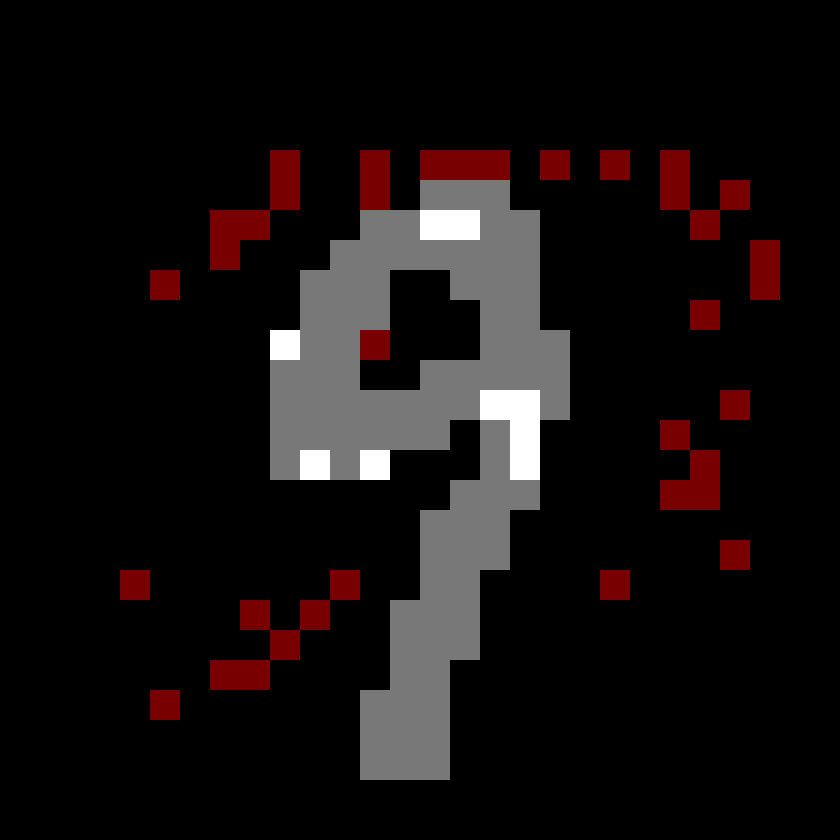}
			\caption{$k^\star = 45$}
		\end{subfigure}

		\caption{Examples of \emph{Minimum Sufficient Reasons} over the MNIST dataset. All images correspond to positive instances for a decision tree of $1652$ leaves that detects the digit $9$. Three instances are misclassified.  Light pixels of the original image are depicted in grey, and the light pixels of the original image that are part of the minimum sufficient reason are colored white. Dark pixels that are part of the minimum sufficient reason are colored with red. Individual captions denote the size of the minimum sufficient reasons with $k^\star$.}
		\label{fig:experiments-positive-9}
\end{figure}

\begin{figure}
\centering
		\begin{subfigure}{.3\textwidth}
			\centering 
			\includegraphics[scale=0.1]{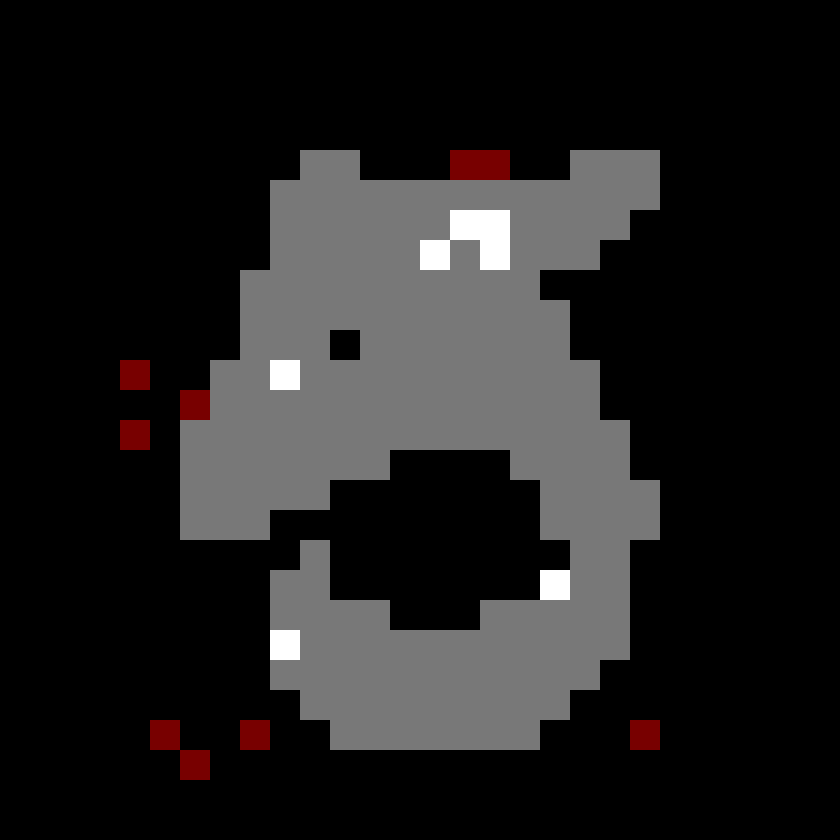}
			\caption{$k^\star = 16$}
		\end{subfigure}  
		\begin{subfigure}{.3\textwidth}
			\centering
			\includegraphics[scale=0.1]{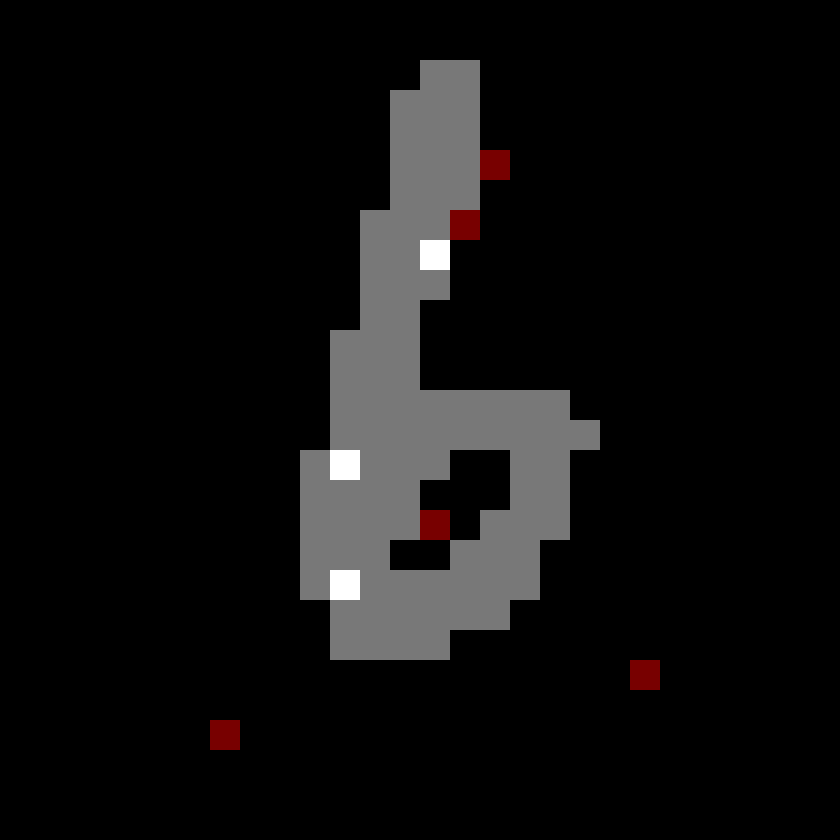}
			\caption{$k^\star = 8$}
		\end{subfigure}
		\begin{subfigure}{.3\textwidth}
			\centering
			\includegraphics[scale=0.1]{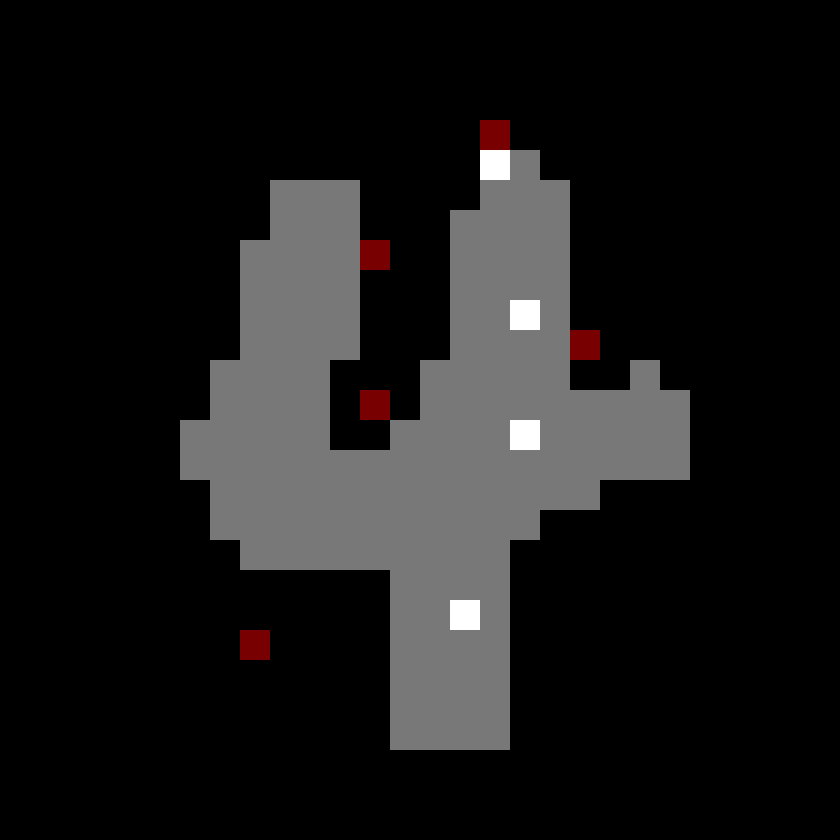}
			\caption{$k^\star = 9$}
		\end{subfigure}

		\begin{subfigure}{.3\textwidth}
			\centering
			\includegraphics[scale=0.1]{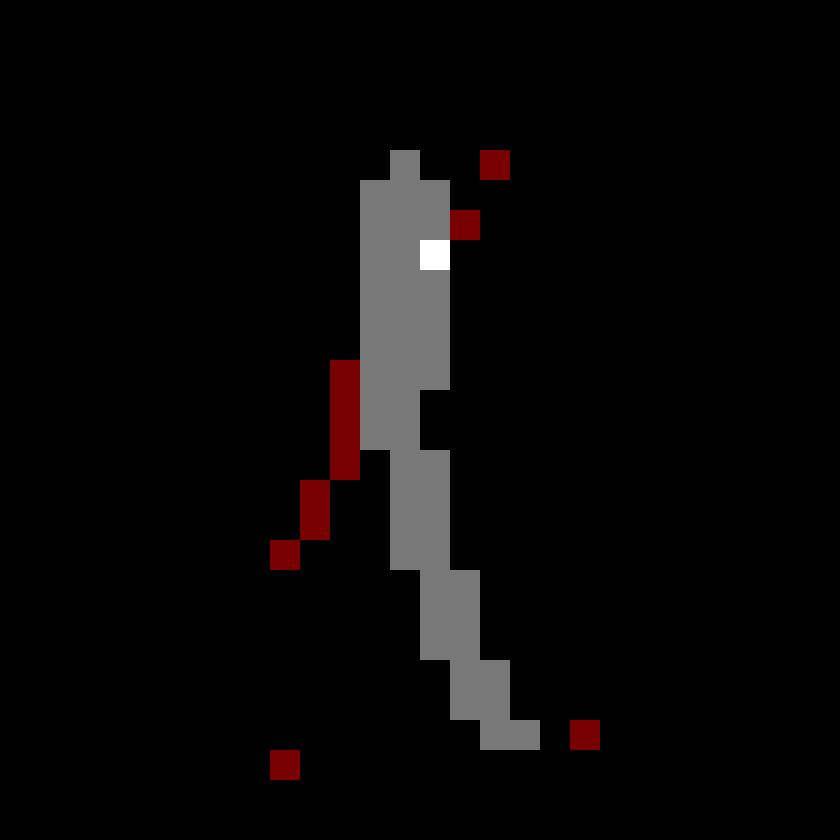}
			\caption{ $k^\star = 12$}
		\end{subfigure}
		\begin{subfigure}{.3\textwidth}
			\centering
			\includegraphics[scale=0.1]{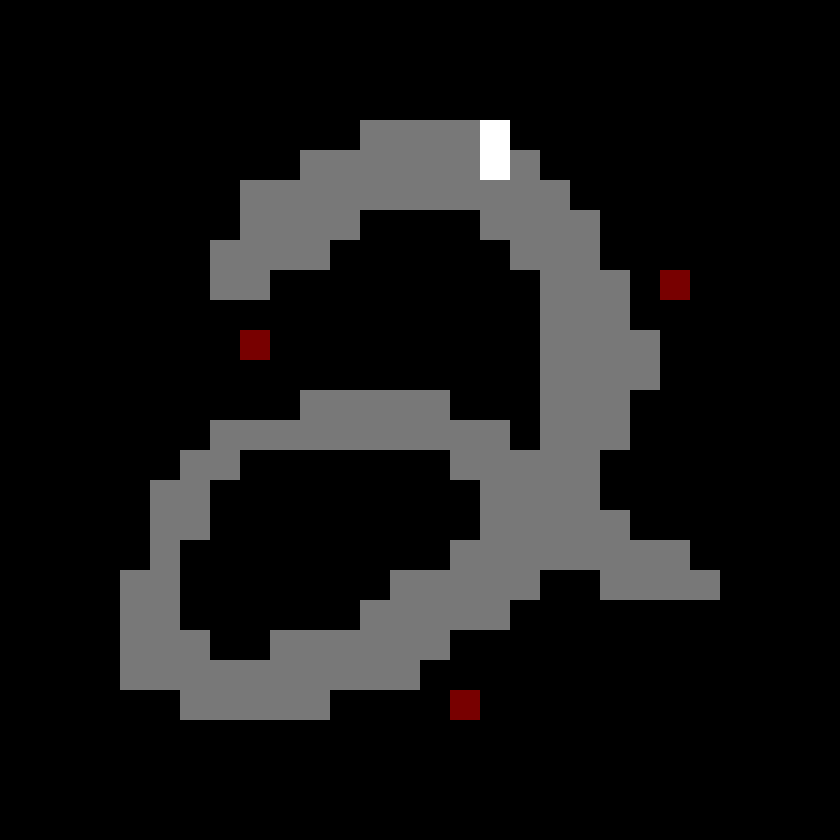}
			\caption{$k^\star = 5$}
		\end{subfigure}
		\begin{subfigure}{.3\textwidth}
			\centering
			\includegraphics[scale=0.1]{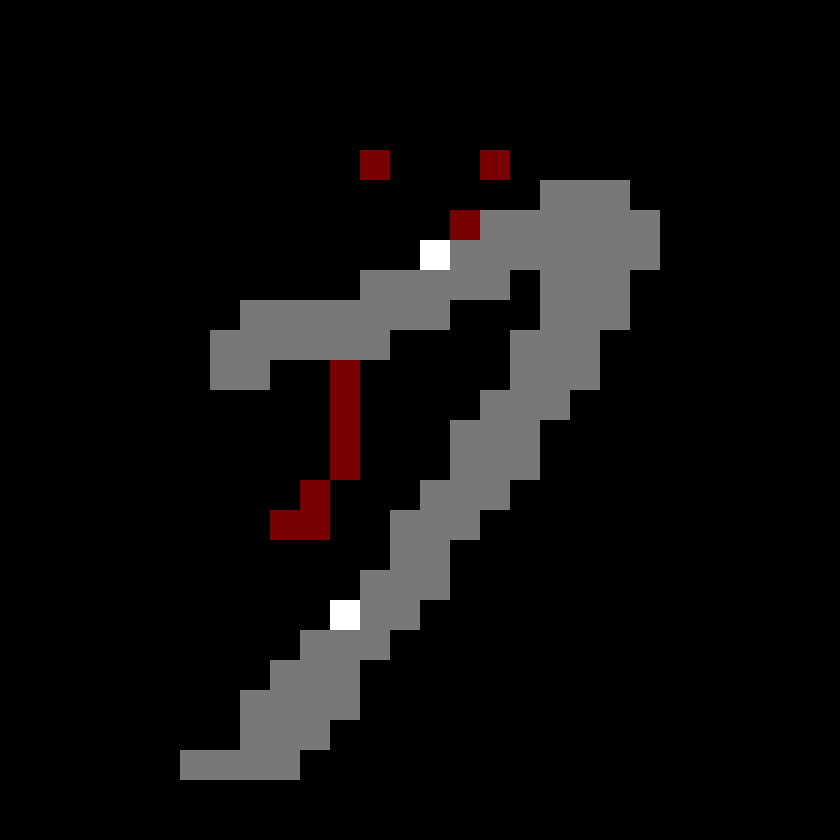}
			\caption{($k^\star = 12$}
		\end{subfigure}
		
		\begin{subfigure}{.3\textwidth}
			\centering
			\includegraphics[scale=0.1]{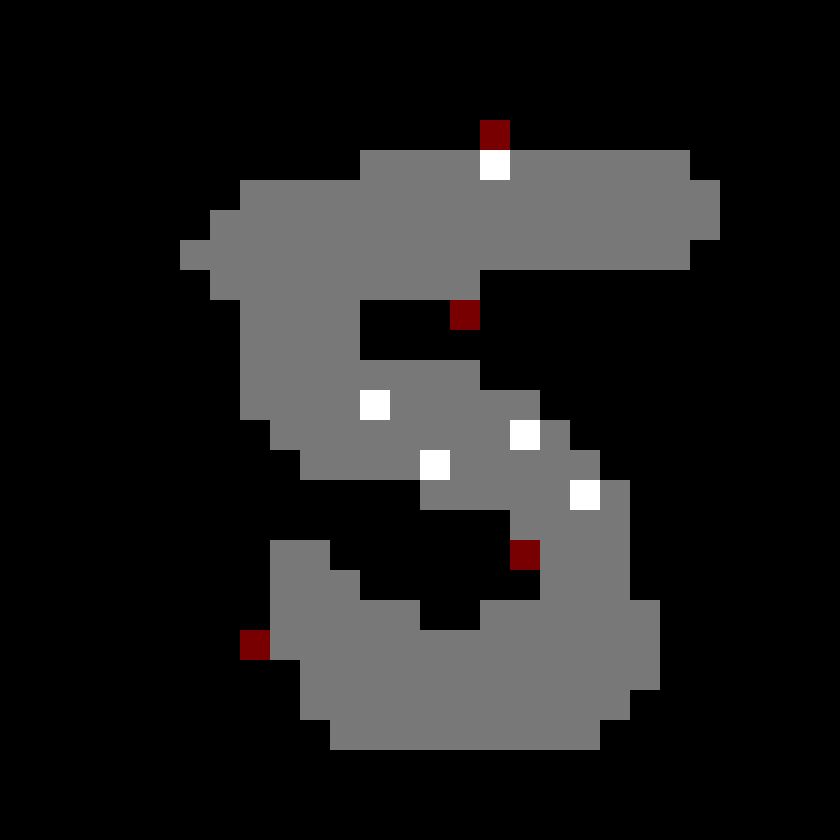}
			\caption{( $k^\star = 9$}
		\end{subfigure}
		\begin{subfigure}{.3\textwidth}
			\centering
			\includegraphics[scale=0.1]{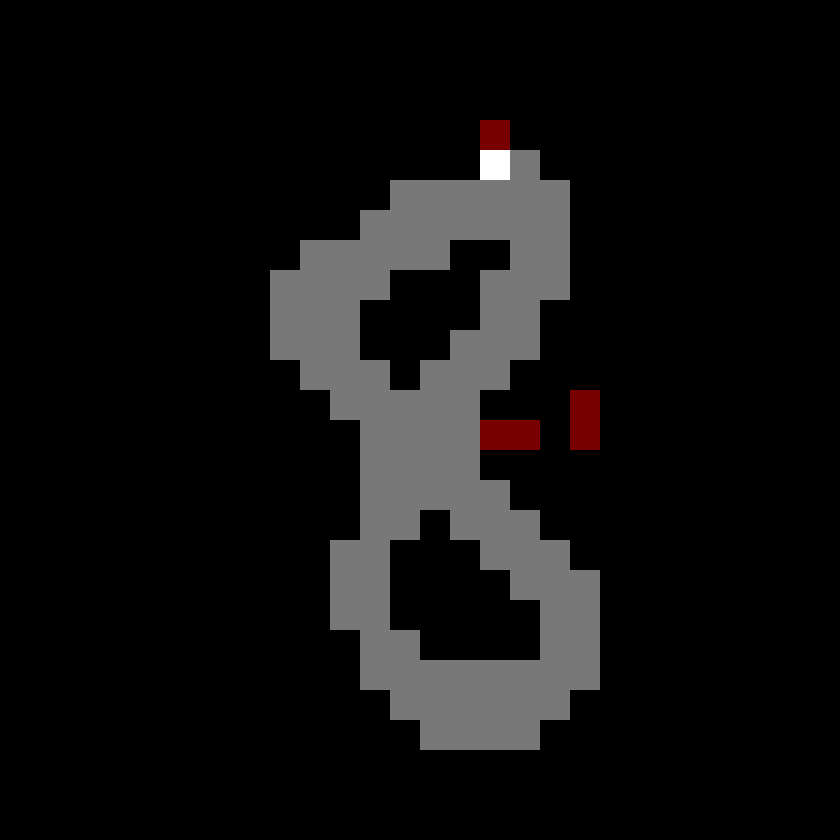}
			\caption{$k^\star = 6$}
		\end{subfigure}
		\begin{subfigure}{.3\textwidth}
			\centering
			\includegraphics[scale=0.1]{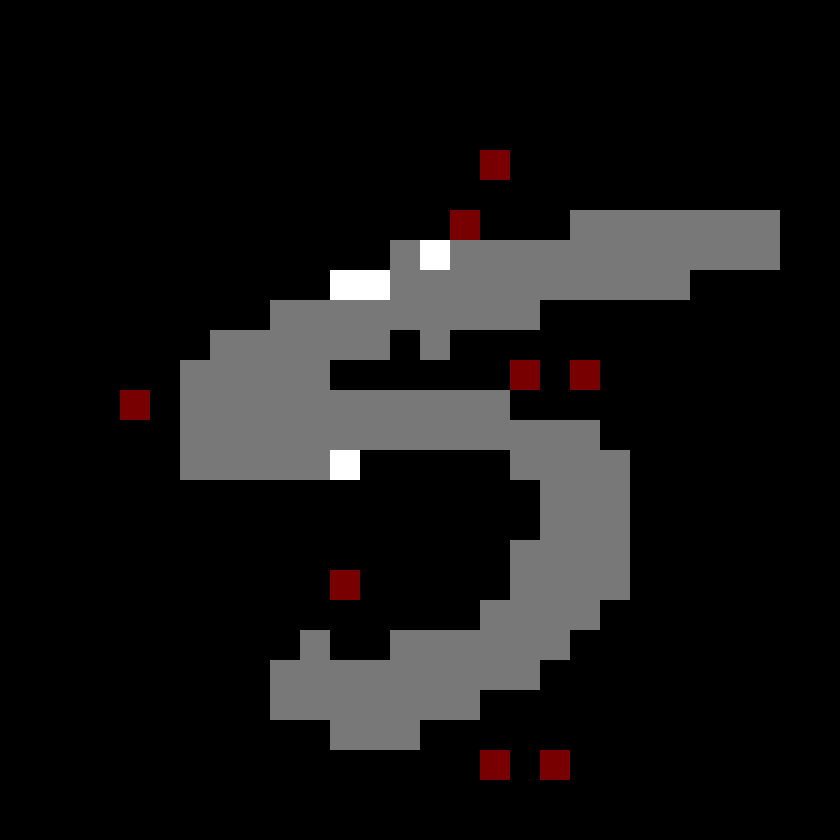}
			\caption{$k^\star = 12$}
		\end{subfigure}
		
		\begin{subfigure}{.3\textwidth}
			\centering
			\includegraphics[scale=0.1]{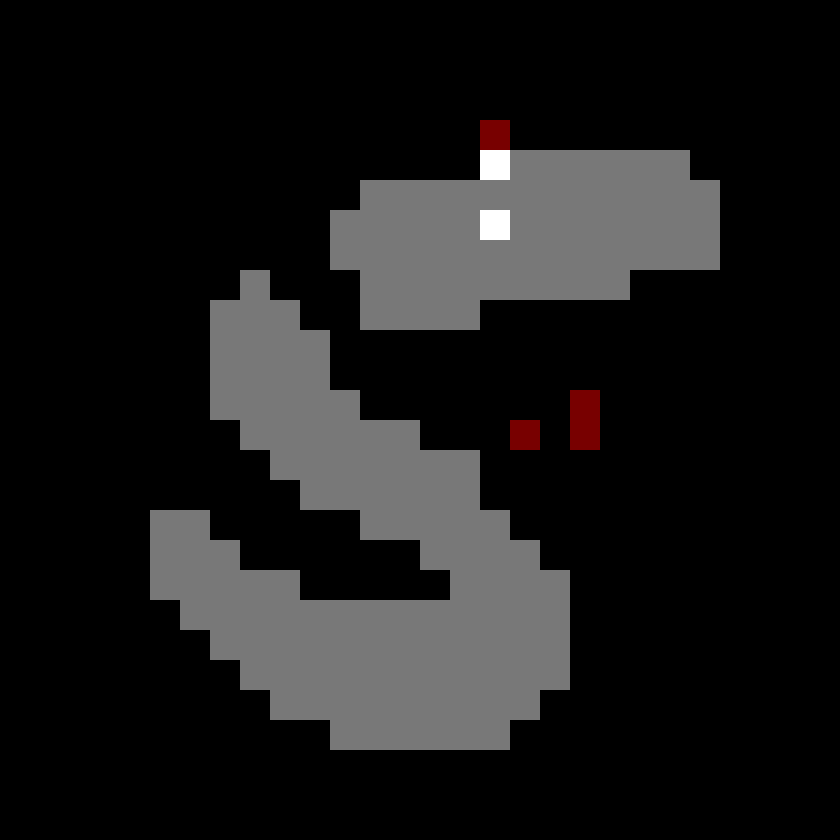}
			\caption{$k^\star = 6$}
		\end{subfigure}
		\begin{subfigure}{.3\textwidth}
			\centering
			\includegraphics[scale=0.1]{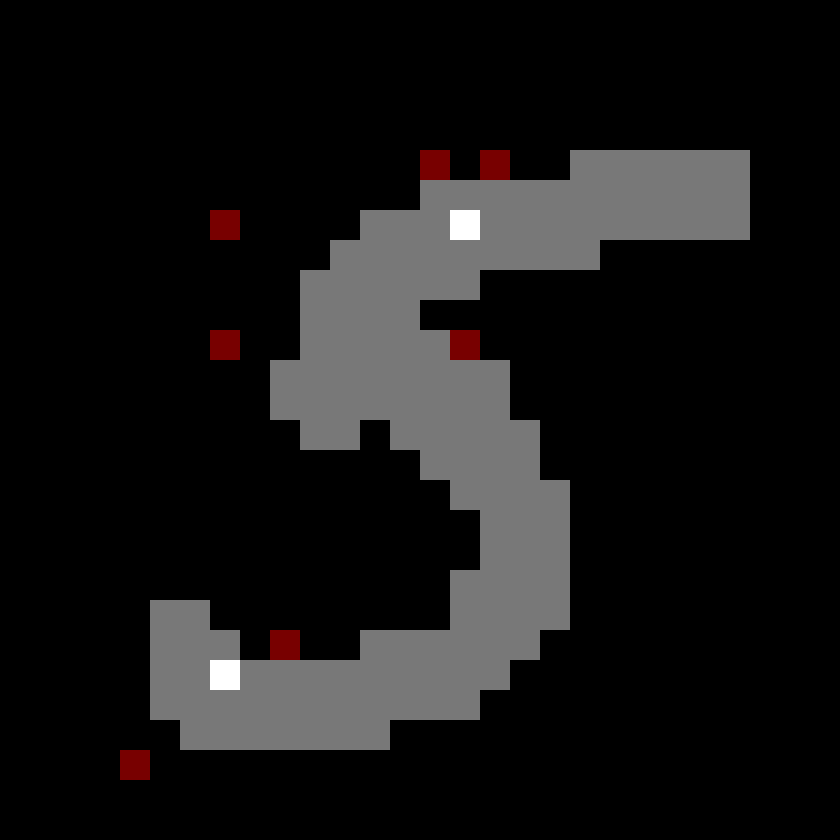}
			\caption{$k^\star = 9$}
		\end{subfigure}
		\begin{subfigure}{.3\textwidth}
			\centering
			\includegraphics[scale=0.1]{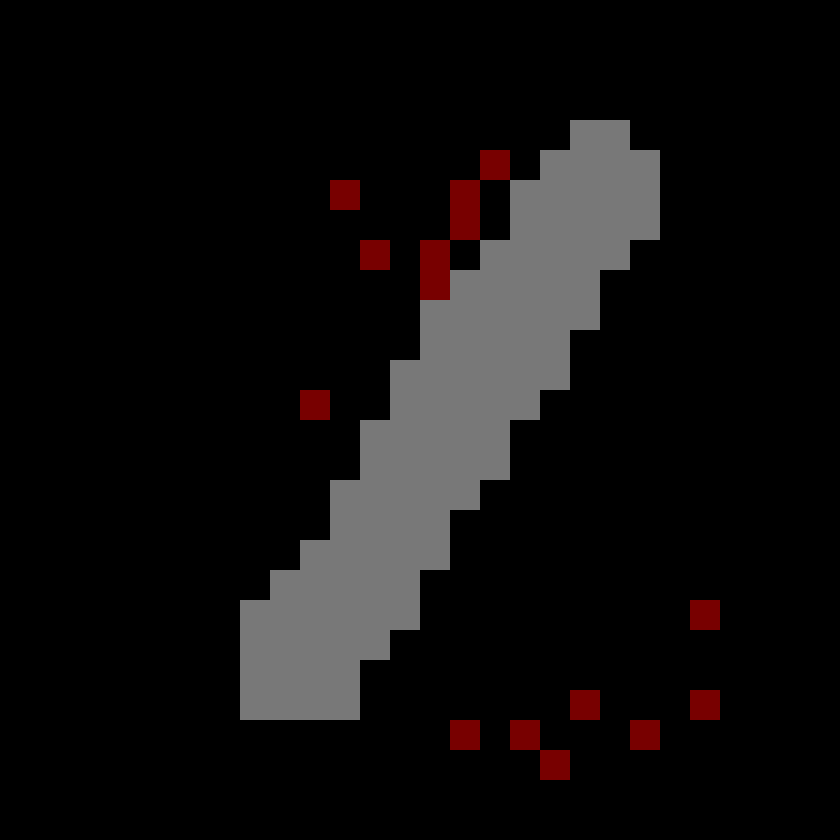}
			\caption{$k^\star = 15$}
		\end{subfigure}

		\caption{Examples of \emph{Minimum Sufficient Reasons} over the MNIST dataset. All images correspond to (correctly classified) negative instances for a decision tree of $1652$ leaves that detects the digit $9$.  Light pixels of the original image are depicted in grey, and the light pixels of the original image that are part of the minimum sufficient reason are colored white. Dark pixels that are part of the minimum sufficient reason are colored with red. Individual captions denote the size of the minimum sufficient reasons with $k^\star$.}
		\label{fig:experiments-negative-9}
\end{figure}

\end{document}